\documentclass{article}

\usepackage{algorithm}
\usepackage{algorithmicx}
\usepackage[noend]{algpseudocode}
\usepackage{bbm}
\usepackage{graphics}

     \usepackage[final]{neurips_2025}

\usepackage[utf8]{inputenc} 
\usepackage[T1]{fontenc}    
\usepackage{hyperref}       
\usepackage{url}            
\usepackage{booktabs}       
\usepackage{amsfonts}       
\usepackage{nicefrac}       
\usepackage{microtype}      
\usepackage{xcolor}         

\usepackage{amsmath}
\usepackage{amsthm}
\usepackage{amssymb}
\newtheorem{theorem}{Theorem}
\newtheorem{proposition}{Proposition}
\newtheorem{corollary}{Corollary}
\newtheorem{lemma}{Lemma}
\theoremstyle{definition}

\newtheorem{definition}{Definition}
\newtheorem{remark}{Remark}

\newcommand{\conv}{\mathrm{conv}}

\newcommand{\cP}{\mathcal{P}}
\newcommand{\cQ}{\mathcal{Q}}

\newcommand{\KL}{\mathop{D_{\mathrm{KL}}}\limits}

\newcommand{\TV}{\mathop{\mathrm{TV}}\limits}
\DeclareMathOperator*{\argmin}{arg\,min}
\DeclareMathOperator*{\argmax}{arg\,max}

\newcommand{\re}{\mathbb{R}}

\newcommand{\E}{\mathbf{E}}
\newcommand{\linner}{\left\langle}
\newcommand{\rinner}{\right\rangle}
\newcommand{\ies}{\partial_-}
\newcommand{\oes}{\partial_+}
\newcommand{\Reg}{\mathrm{Reg}}
\newcommand{\dist}{\mathrm{dist}}
\newcommand{\stab}{\mathrm{stab}}
\newcommand{\pena}{\mathrm{pena}}
\newcommand{\ind}{\mathbb{I}}
\renewcommand{\epsilon}{\varepsilon}

\newcommand{\rbr}[1]{\left(#1\right)}
\newcommand{\Bigrbr}[1]{\Big(#1\Big)}

\newcommand{\sbr}[1]{\left[#1\right]}

\newcommand{\Biggsbr}[1]{\Bigg[#1\Bigg]}
\newcommand{\cbr}[1]{\left\{#1\right\}}

\newcommand{\abr}[1]{\left|#1\right|}
\newcommand{\whatq}{\widehat{q}}
\newcommand{\order}{\mathcal{O}}
\newcommand{\hatl}{\widehat{\ell}}
\newcommand{\inner}[1]{ \left\langle {#1} \right\rangle }
\newcommand{\calA}{{\mathcal{A}}}

\newcommand{\calE}{{\mathcal{E}}}

\newcommand{\gapmin}{\Delta_{\textsc{min}}}
\newcommand{\gap}{\Delta}

\newcommand{\kevin}[1]{}
\newcommand{\arnab}[1]{}

\title{Adapting to Stochastic and Adversarial Losses in\\Episodic MDPs with Aggregate Bandit Feedback}

\author{
  Shinji Ito \\
  The University of Tokyo and RIKEN \\
  \texttt{shinji@mist.i.u-tokyo.ac.jp} \\
  \AND
  Kevin Jamieson \\
  University of Washington\\
  \texttt{jamieson@cs.washington.edu} \\
  \And
  Haipeng Luo \\
  University of Southern California \\
  \texttt{haipengl@usc.edu} \\
  \And
  Arnab Maiti \\
  University of Washington \\
  \texttt{arnabm2@uw.edu} \\
  \And
  Taira Tsuchiya \\
  The University of Tokyo and RIKEN \\
  \texttt{tsuchiya@mist.i.u-tokyo.ac.jp} \\
}
\begin{document}

\maketitle

\begin{abstract}
We study online learning in finite-horizon episodic Markov decision processes (MDPs) under the challenging \textit{aggregate bandit feedback} model,
where the learner observes only the cumulative loss incurred in each episode,
rather than individual losses at each state-action pair.
While prior work in this setting has focused exclusively on worst-case analysis,
we initiate the study of \textit{best-of-both-worlds} (BOBW) algorithms that achieve low regret in both stochastic and adversarial environments.
We propose the first BOBW algorithms for episodic tabular MDPs with aggregate bandit feedback.
In the case of known transitions,
our algorithms achieve $O(\log T)$ regret in stochastic settings and ${O}(\sqrt{T})$ regret in adversarial ones.
Importantly, we also establish matching lower bounds, showing the optimality of our algorithms in this setting.
We further extend our approach to unknown-transition settings by incorporating confidence-based techniques.
Our results rely on a combination of FTRL over occupancy measures,
self-bounding techniques,
and new loss estimators inspired by recent advances in online shortest path problems.
Along the way,
we also provide the first individual-gap-dependent lower bounds and demonstrate near-optimal BOBW algorithms for shortest path problems with bandit feedback.
\end{abstract}

\section{Introduction}
This paper considers online learning problems for finite-horizon episodic Markov decision processes (MDPs) with \textit{aggregate bandit feedback} \citep{efroni2021reinforcement,cohen2021online}.
In this feedback model,
the learner receives feedback on the \textit{aggregate loss} in each episode,
which is the sum of losses for all state-action pairs in the learner's trajectory of that episode,
rather than individual losses for each state-action pair.
The aggregate bandit feedback model naturally arises in various real-world applications where only trajectory-level outcomes are observable.
For example,
in personalized healthcare,
a sequence of medical treatments is administered,
but only the final patient outcome (e.g., recovery rate) is observed, without attributing effects to individual actions.
Similarly,
in application to the design of educational programs,
students experience a curriculum composed of multiple learning activities, while feedback is typically available only in the form of an overall test score.

In the study of online learning for episodic MDPs,
two different models of the loss (or reward) function are commonly considered:
the \textit{stochastic model} and the \textit{adversarial model}.
In the stochastic setting,
it is typically assumed that the loss function $\ell_t$ at each episode $t$ is independently drawn from an unknown fixed distribution.
In contrast,
the adversarial model makes no such probabilistic assumptions and allows the loss function $\ell_t$ to vary arbitrarily over time.
In various online learning/bandit problems,
including those with individual loss feedback in episodic MDPs,
it is well known that one can achieve instance-dependent $O(\mathrm{polylog}~T)$-regret in the stochastic setting,
and $\tilde{O}(\sqrt{T})$-regret in the adversarial setting,
where $T$ is the number of rounds/episodes and the $\tilde{O}(\cdot)$ notation hides logarithmic factors in parameters.
However,
to the best of our knowledge,
prior work on episodic MDPs with aggregate bandit feedback has focused exclusively on the worst-case analysis
(i.e.,~$O(\sqrt{T})$-bounds at best),
and no algorithm is known to achieve instance-dependent $O(\mathrm{polylog}~T)$-regret 
under the stochastic loss model with aggregate bandit feedback and unknown transitions.

This paper focuses on the design of algorithms that can effectively handle both stochastic and adversarial loss models.
More specifically,
we aim to develop a single algorithm that,
without any prior knowledge about the nature of the environment,
achieves $O(\mathrm{polylog}~T)$-regret in stochastic settings and $\tilde{O}(\sqrt{T})$-regret in adversarial settings.
Such algorithms are referred to as \emph{best-of-both-worlds} (BOBW) algorithms~\citep{bubeck2012best}.
While many prior works design separate algorithms tailored to either the stochastic or adversarial model,
real-world applications often involve uncertainty about the true nature of the environment,
making BOBW algorithms especially valuable in practice.
Although BOBW algorithms have been developed for various settings—including episodic MDPs with individual loss feedback \citep{jin2020simultaneously,jin2021best, jin2023no}—no such algorithm was known for episodic MDPs with aggregate bandit feedback with unknown transition, prior to this work.

\subsection{Contribution}
This paper presents the first BOBW algorithms for episodic tabular MDPs with aggregate bandit feedback and unknown transitions.
More specifically,
we consider layered MDPs with $L$-layers,
and begin by considering the setting where the transition probability matrix $P$ is known,
designing an algorithm that achieves $\tilde{O}(\sqrt{T})$-regret in adversarial environments and $O(\log T)$-regret in stochastic environments.
We then extend this approach to the more realistic and challenging setting where the transition matrix $P$ is unknown,
with the help of the techniques by \citet{jin2021best} for handling unknown transitions.
\renewcommand{\arraystretch}{1.5}
\begin{table}[t]
	\caption{Regret bounds for episodic MDPs with known transitions. Here $\pi^*$ is an optimal policy and $S^*$ is the set of states that can be reached by $\pi^*$. ``TC'' stands for computational time complexity and a checkmark ($\checkmark$) indicates that an efficient implementation is possible. }
	\label{tab:bounds-KT}
	\centering
        \resizebox{\columnwidth}{!}{
	\begin{tabular}{lccc}\toprule
	Algorithm & Stochastic  & Adversarial & TC
	\\\midrule
	\citet{bubeck2012towards}
	&
	$\sqrt{|S|^2 |A| T \log |A|}$
	&
	$\sqrt{|S|^2 |A| T \log |A|}$
	&
	\\
	\citet{lancewicki2025near}
	&
	$\sqrt{|S||A|LT \log \iota}$
	&
	$\sqrt{|S||A|LT \log \iota}$
	&
	$\checkmark$
	\\
	\citet{dann2023blackbox},
	\citet{ito2024adaptive}
	&
	$\frac{|S|^2|A| \log|A| \log T}{\Delta_{\min}}$
	&
	$\sqrt{|S|^2 |A| T \log |A|}$
	\\
	\textbf{This study} (Tsallis entropy)
	&
	$
	\sum_{s \neq s_L, a \neq \pi^*(s)}
		\frac{\log T}{\Delta(s, a)}
		+
		\frac{L|S| \log T}{\Delta_{\min}}
	$
	&
	$ \sqrt{|S||A|L T }$
	&
	$\checkmark$
	\\
	\textbf{This study} (Log-barrier)
	&
	$
	\sum_{s \neq s_L, a \neq \pi^*(s)}
		\frac{\log T}{\Delta(s, a)}
	$
	&
	$ \sqrt{|S||A|L T \log T}$
	&
	$\checkmark$
	\\
	Lower bound
	&
	$
	\sum_{s \in S^*, a: \Delta(s,a)>0}
		\frac{\log T}{\Delta(s, a)}
	$
	&
	$ \sqrt{|S||A|L T }$
	\\
	\bottomrule
	\end{tabular}
        }
\end{table}

\begin{table}[t]
	\caption{Regret bounds for episodic MDPs with unknown transition.
	$\log \iota = O(\log (|S||A|T))$.
	}
	\label{tab:bounds-UT}
	\centering
	\begin{tabular}{lccc}\toprule
	Algorithm &  Stochastic  & Adversarial & TC
	\\\midrule
	\citet{efroni2021reinforcement}
	&
	$\sqrt{|S|^4|A|^3 L T \log \iota }$
	&
	--
	\\
	\citet{cohen2021online}
	&
	$\sqrt{(|S||A|)^{O(1)} T}$
	&
	$\sqrt{(|S||A|)^{O(1)} T}$
	&
	$\checkmark$
	\\
	\citet{lancewicki2025near}
	&
	$\sqrt{|S|^2 |A| L T \log \iota}$
	&
	$\sqrt{|S|^2 |A| L T \log \iota}$
	&
	$\checkmark$
	\\
	\textbf{This study}
	&
	$
	\frac{(|S||A|)^{O(1)}\log^2 \iota}{\Delta_{\min}}
	$
	&
	$ \sqrt{(|A|+L) |S|^2 |A| L T \log^2 \iota }$
	&
	$\checkmark$
	\\
	\bottomrule
	\end{tabular}
\end{table}

Our results are accomplished by combining some algorithmic frameworks including follow-the-regularized-leader (FTRL) over occupancy measures and
self-bounding techniques \citep{wei2018more,zimmert2021tsallis}
with key ideas from the recent study by \citet{maiti2025efficient} on the online shortest path problem with bandit feedback.
More precisely,
our algorithm is inspired by their loss estimation method for the shortest path problem,
which plays a central role in our design.
By adopting their loss estimation approach,
not only can we construct an estimator using only bandit feedback,
but we also find that its second moment is well-controlled 
(see Lemma~\ref{lem:second-moment} and the subsequent discussion).
This property is highly beneficial for designing BOBW algorithms.
Building on this insight,
we propose an efficient and nearly optimal BOBW algorithm for the online shortest path problem with bandit feedback,
which naturally extends to the case of MDPs with known transitions.

However,
extending this estimation idea to the unknown-transition setting requires substantial care.
The new estimator contains negative terms,
and thus,
naively replacing the occupancy measure with its upper confidence bound,
as done in prior work~\citep{jin2020simultaneously,jin2021best},
does not necessarily yield an optimistic estimator.
To address this,
we carefully design the estimator so that it is optimistic in expectation and its second moment remains well-controlled.
This allows us not only to effectively handle aggregate bandit feedback,
but also,
perhaps surprisingly,
to avoid the technically involved loss-shifting technique used in prior analyses~\citep{jin2021best,jin2023no},
thereby simplifying the overall regret analysis.
We refer the reader to Section~\ref{sec:estimator} for a detailed discussion of these estimator constructions.

The regret upper bounds established in this work and in prior studies are summarized in Tables~\ref{tab:bounds-KT},
\ref{tab:bounds-UT}, and \ref{tab:bounds-SP}.
For detailed definitions of the symbols used in the tables,
we refer the reader to Section~\ref{sec:setup}.
``TC'' stands for computational time complexity and a checkmark ($\checkmark$) indicates that an efficient implementation is possible.
The symbol $\iota > 0$ denotes a polynomial factor in other parameters such as $T$, $|S|$, and $|A|$.
In Table~\ref{tab:bounds-SP} for the shortest path problem,
$\cP$ denotes the set of all directed paths,
and it holds that $\log |\mathcal{P}| \lesssim \min\{ |V|, L \log |E| \}$,
where $L$ is the maximum number of edges in a path of the given graph.
Here,
$X \lesssim Y$ means $X = O(Y)$ in this paper.
Similarly,
$X \gtrsim Y$ means $X = \Omega(Y)$.

\begin{table}[t]
	\caption{Comparison of regret bounds for online shortest path problems with bandit feedback.
The quantity $c^*>0$ represents the instance-dependent constant characterizing the asymptotic lower bound for linear bandits~\citep{lattimore2017end},
and it is known that $c^* \lesssim \frac{|E|}{\Delta_{\min}}$
as shown in \citet{lee2021achieving}. Here $E' = \bigcup_{v \in V \cup \{ s \}} \oes v \setminus \{ \pi^*(v) \}$ and $\tilde{E} = \bigcup_{v \in V^* \cup \{ s \}} \oes v \setminus \{ \pi^*(v) \}$, where $\oes v \subset E$ is the set of outgoing edges from $v$, $\pi^*$ is an optimal policy and $V^*$ is the set of nodes reached by $\pi^*$. Please refer to Definition~\ref{def:consistent} for further details.
	}
	\label{tab:bounds-SP}
	\centering
	\begin{tabular}{lccc}\toprule
	Algorithm & Stochastic & Adversarial & TC
	\\\midrule
	\citet{bubeck2012towards}
	& $\sqrt{|E| T \log |\cP|}$ & $\sqrt{|E| T \log |\cP|}$ 
	\\
	\citet{lattimore2017end}
	&   $c^* \log T$  & --
	\\
	\citet{lee2021achieving}
	&
	$c^* \log T \log (|\cP| T)$
	&
	$ \sqrt{|E| T} \log(|\cP| T)$
	\\
	\citet{dann2023blackbox},
	\citet{ito2024adaptive}
	&
	$\frac{|E| \log |\cP| \log T}{\Delta_{\min}}$
	&
	$ \sqrt{|E| T \log |\cP|}$
	\\
	\textbf{This study} (Tsallis entropy)
	&
	$
	\sum_{e \in E'}
		\frac{\log T}{\Delta(e)}
		+
		\frac{|V|^2 \log T}{\Delta_{\min}}
	$
	&
	$ \sqrt{|E|LT}$
	&
	$\checkmark$
	\\
	\textbf{This study} (Log-barrier)
	&
	$
	\sum_{e \in E'}
		\frac{\log T}{\Delta(e)}
	$
	&
	$ \sqrt{|E|LT \log T}$
	&
	$\checkmark$
	\\
	Lower bound
	&
	$
	c^* \log T
	\gtrsim
	\sum_{e \in \tilde{E}} \frac{\log T}{\Delta(e)} 
	$
	&
	$
	\sqrt{|E| T \log |\cP| / \log |E| }
	$
	\\
	\bottomrule
	\end{tabular}
\end{table}

As shown in Tables~\ref{tab:bounds-KT} and \ref{tab:bounds-SP},
we propose computationally efficient BOBW algorithms that achieve nearly tight regret bounds for known-transition MDPs and the online shortest path problem.
We note here that the corresponding lower bounds for stochastic environments are also new contributions of this paper.
The adversarial lower bounds for known-transition MDPs and online shortest paths are due to \citep{lancewicki2025near}
and \citep{maiti2025efficient},
respectively.
For unknown-transition MDPs (Table~\ref{tab:bounds-UT}),
we propose the first BOBW algorithm that achieves an 
$O\left(
	\sum_{s \neq s_L, a \neq \pi^*(s)}
	\frac{L^4|S| \log \iota  + |S||A| \log^2 \iota}{\Delta(s, a)}
	+
	\frac{(L^3 |S|^2)(|S|+|A|) \log \iota + L|S|^2|A| \log^2 \iota }{\Delta_{\min}}
\right)$-regret bound in the stochastic setting and simultaneously achieves an upper bound of 
$\tilde{O}\left(\sqrt{(|A|+L)|S|^2 |A| L T}\right)$ in the adversarial setting.
Moreover,
all of our BOBW algorithms exhibit robustness to corrupted stochastic environments,
achieving regret bounds of the form $O(U + \sqrt{UC})$, where $U$ is the stochastic regret and $C$ is the corruption level.
These results are established using an argument similar to the standard self-bounding technique~\citep{zimmert2021tsallis,jin2021best,jin2023no}.

Due to differences in the problem settings,
several caveats must be taken into account when making comparisons.
First,
while prior work on episodic MDPs assumes that the per-step loss or reward within each episode is $O(1)$,
our setting assumes that the \textit{aggregate} loss or reward over an entire episode is $O(1)$.
To align the scales,
we multiply the regret bounds from prior work by a factor of $1/L$.
As noted in Remark~\ref{rem:scale} in the appendix,
our setting is strictly more general.
In addition,
some prior works~\citep{cohen2021online,lancewicki2025near} consider non-layered settings,
and we reinterpret their bounds in terms of the layered setting by replacing $|S|L$ with $|S|$ where appropriate.
Furthermore,
while we evaluate the expected regret defined in Section~\ref{sec:setup},
some of the prior works \citep{efroni2021reinforcement,lancewicki2025near,lee2021achieving} establish high-probability regret bounds. 

Tables~\ref{tab:bounds-KT} and \ref{tab:bounds-SP} include the bounds achieved by applying algorithms developed for finite-armed linear bandits \citep{bubeck2012towards,lattimore2017end,lee2021achieving,dann2023blackbox,ito2024adaptive},
where the feature space dimension is $|S||A|$ or $|E|$,
and the number of arms is calculated as $|A|^{|S|}$ or $|\mathcal{P}|$,
respectively.
We note that due to the exponential number of arms in this approach,
it is generally unclear whether an efficient implementation is feasible.
In contrast,
the other algorithms listed,
including our proposed methods,
can be implemented efficiently using dynamic programming or convex optimization techniques. 
We include additional related work in Appendix~\ref{sec:additional_related_work}.
\arnab{The last two paragraphs seems like something extra that we can shift to appendix in order to save space.}






\section{Problem setup}
\label{sec:setup}
\subsection{Episodic Markov decision process}
In this paper,
we consider finite-horizon Markov decision processes (MDPs) with finite actions and finite states.
The model is defined by a tuple $({S}, {A}, {P})$,
where ${S}$ is the finite set of states,
${A}$ is the set of actions,
and $P: S \times A \times S \to [0,1]$
is the transition function that defines the probability of moving from one state to another given an action.
We assume that the state space $S$ consists of $(L+1)$ \textit{layers}:
$S$ can be expressed as a disjoint union as $S = \bigcup_{k=0}^{L} S_{k}$,
where 
$S_{0} = \{ s_{0} \}$ (initial state),
$S_{L} = \{ s_{L} \}$ (terminal state),
$S_{k} \neq \emptyset$ for $k \in [L-1]$,
and $S_{k} \cup S_{k'} = \emptyset$ for $k \neq k'$.
Transitions from the $k$-th layer are allowed only to the $(k+1)$-th layer,
i.e.,
for any $k \in \{0, 1, \ldots, L-1 \}$,
$\sum_{s' \in S_{k+1} } P(s' | s, a) = 1$ 
holds for all $s \in S_{k}$ and $a \in A$ and
$P(s'|s, a) = 0$ for all $s \in S_{k}$,
$a \in A$ and $s' \in S \setminus S_{k+1}$. Let $k(s)$ denote the index of the layer to which the state $s$ belongs.
In \textit{known-transition} setting,
we assume that the transition function $P$ is known to the player.
On the other hand,
in \textit{unknown-transition} setting,
the player does not know the transition function $P$
and it is learned through interactions with the environment.

In episodic MDPs,
the player interacts with the environment in a sequence of episodes.
Before each episode $t \in [T]$,
the player selects a policy 
$\pi_t \in \Pi := \{ \pi: (S \setminus \{ s_L \}) \times A \rightarrow [0,1] \mid \sum_{a \in A} \pi(a|s) = 1 \}$
and the environment selects a loss function $\ell_t: S \times A \rightarrow [0, 1]$.
Each episode $t \in [T]$ consists of a sequence of time steps.
The initial state $s_0^t$ is set to $s_0$ for all episodes $t \in [T]$.
At each time step $k \in \{0, 1, \ldots, L-1 \}$,
the player chooses an action $a_k^t \in A$ according to the policy $\pi_t$,
i.e.,
$a_k^t$ follows the distribution $\pi_t(\cdot | s_k^t)$,
and state $s_{k+1}^t$ is sampled from the transition function $P(\cdot|s_k^t, a_k^t)$ given the current state $s_k^t$ and action $a_k^t$.
At the end of the episode,
the player observes the \textit{aggregate loss} feedback $c_t \in [0, 1]$ such that 
$\E \left[ c_t \mid ((s_k^t, a_k^t))_{k \in [L-1]}\right] = \sum_{k=0}^{L-1} \ell_t(s_k^t,a_k^t)$,
which corresponds to the sum of the losses incurred at each time step in the episode.
We note that the player does not observe the individual losses $\ell_t(s_k^t,a_k^t)$ for $k \in [L-1]$.
All the information revealed to the player in each episode is the state-action trajectory 
$((s_k^t, a_k^t))_{k \in [L-1]}$ and the aggregate loss $c_t$.
We assume that the loss function $\ell_t$ is chosen so that $\sum_{k=0}^{L-1} \ell_t(s_k^t,a_k^t) \in [0, 1]$ for any possible trajectory $((s_k^t, a_k^t))_{k \in [L-1]}$.
For notational convenience,
we suppose that $\ell_{t}(s_L, a)$ is set to $0$ for all $a \in A$.

For a transition $P$,
a loss function $\ell: S \times A \to \re$,
and a policy $\pi \in \Pi$,
let $Q^{P, \pi}(s, a; \ell)$ and $V^{P, \pi}(s; \ell)$ express the values of the $V$- and $Q$- functions,
i.e.,
we set $V^{P,\pi}(s_L) = Q^{P,\pi}(s_L, a)= 0$
and recursively define
\begin{align*}
  Q^{P, \pi}(s, a; \ell)
  &
  = 
      \ell(s, a) + \sum_{s' \in S} P(s'|s, a) V^{P, \pi}(s'; \ell),
  \quad
  V^{P, \pi}(s; \ell) 
  &
  = 
  \sum_{a \in A} \pi(a|s) Q^{P,\pi}(s, a; \ell)
\end{align*}
for $s \neq s_L$ and $a \in A$.
The $Q$-function and $V$-function satisfy the equality stated in the following lemma,
serves as a fundamental property that underpins both the justification of our loss estimator and the derivation of the regret upper bound.
\begin{lemma}
  \label{lem:PDL}
  Suppose $\bar{\ell}$ is defined by 
  $\bar{\ell}(s,a) 
  = Q^{\pi}(s,a;\ell) - V^{\pi}(s;\ell)$ 
  for some $\pi \in \Pi$ and
  for all $s \in S$ and $a \in A$.
  We then have
  \begin{align}
    \label{eq:PDL}
    V^{\pi'}(s;\bar{\ell}) = 
    V^{\pi'}(s; \ell) - V^{\pi}(s; \ell),
    \quad
    Q^{\pi'}(s,a;\bar{\ell}) = Q^{\pi'}(s,a; \ell) - V^{\pi}(s; \ell),
  \end{align}
  for any $\pi' \in \Pi$, $s \in S$ and $a \in A$.
\end{lemma}
In addition,
we define the \textit{occupancy measure} $q^{P, \pi}: S \times A \to [0, 1]$ by
\begin{align}
  q^{P, \pi}(s, a) 
  = 
  \Pr \left[ (s_k, a_k) = (s, a) \mid ((s_k, a_k))_{k = 0}^L \text{ is sampled according to $\pi$ and $P$ } \right]
\end{align}
for $s \in S_k$ and $a \in A$.
We also denote $q^{P, \pi}(s) = \sum_{a \in A} q^{P, \pi}(s, a) = \Pr \left[ s_k = s \right]$
for $s \in S_k$,
for notational convenience.
Let $\cQ^{P} = \{ q^{P, \pi} \mid \pi \in \Pi \}$ be the set of occupation measures induced by the transition $P$.
We note that $\cQ^P$ is a closed convex set.
From the definitions of $V$, $Q$ and $q$,
we have $\E\left[c_t \mid \pi_t, \ell_t \right] = V^{P,\pi_t}(s_0, \ell_t) = \linner \ell_t, q^{P, \pi_t} \rinner := \sum_{k = 0}^{L-1}
	\sum_{s \in S_k}
	\sum_{a \in A}
	\ell_t(s, a) 
	q^{P, \pi_t}(s, a)$.
Using these concepts,
we define the \textit{regret}:
\begin{align}
  \Reg_T(\pi^*) 
  = 
  \E\left[\sum_{t=1}^{T} \left( V^{P, \pi_t}(s_0, \ell_t) - V^{P, \pi^*}(s_0, \ell_t) \right)\right]
  =
  \E\left[\sum_{t=1}^{T}  \linner \ell_t, q^{P, \pi_t} - q^{P, \pi^*} \rinner \right],
\end{align}
where $\E[\cdot]$ denotes the expectation with respect to all the randomness of the environment and the player.
We also denote $ \Reg_T= \max_{\pi^* \in \Pi}  \Reg_T(\pi^*) $.
Hereafter,
when it is clear from the context,
we may omit $P$ and $\pi$ for simplicity.
Additionally, for notational convenience,
we denote
$Q_t^{\pi}(\cdot)
=
Q^{P, \pi}(\cdot; \ell_t) $,
$V_t^{\pi}(\cdot) = V^{P, \pi}(\cdot; \ell_t) $,
$q_t = q^{P, \pi_t}$,
and
$q^* = q^{P, \pi^*}$.

\subsection{Regime of environments}
We consider several different regimes as models for the environment generating the loss function $\ell_t$.
In an \textit{adversarial regime},
$\ell_t$ can be chosen arbitrarily by an adversary,
depending on the history 
$((s_k^\tau, a_k^\tau))_{k\in \{0, \ldots, L-1\} , \tau \in [t-1] }$
so far and the policy $\pi_t$ chosen by the player.
In a \textit{stochastic regime},
the loss function $\ell_t$ is independently and identically drawn from an unknown distribution for each episode $t \in [T]$.
For a stochastic environment,
let $\ell^*: S \times A \to [0,1]$ denote the expected loss function,
which is defined as $\ell^*(s,a) = \E[\ell_t(s,a)]$ for all $s \in S$ and $a \in A$.
We then have
  $
  \Reg_T 
  =
  \E \left[ 
    \sum_{t=1}^T \linner \ell^*, q_t - q^* \rinner
  \right]
  =
  \E \left[ 
    \sum_{t=1}^T \sum_{s \in S \setminus \{ s_L \}} \sum_{a \in A} \Delta(s, a) q_t(s, a)
  \right]
  $,
where $q^* \in \argmin_{q \in \cQ} \linner \ell^*, q \rinner$
and $\Delta:S \times A \to [0, 1]$ is defined as
$
  \Delta(s, a)
  =
  Q^{\pi^*}(s, a; \ell^*) - V^{\pi^*}(s; \ell^*)
  =
  Q^{\pi^*}(s, a; \ell^*) - \min_{a' \in A} Q^{\pi^*}(s, a'; \ell^*)
$
for an optimal policy $\pi^*$ that minimizes $V^\pi(s_0;\ell^*)$.
We note that the optimal policy $\pi^*$ is unique if and only if
the set
$\argmin_{a \in A} \{ Q^{\pi^*}(s, a; \ell^*) \} = \{ a \in A \mid \Delta(s,a) = 0\}$
consists of a single action for all $s \in S$.
Here,
as a generalization of a stochastic regime admitting a unique optimal policy,
we can define the an adversarial regime with self-bounding constraints:
\begin{definition}[self-bounding regime for MDPs]
  Let $\pi^*: S \to A$ be a deterministic policy.
  Suppose that $\Delta: S \times A \rightarrow [0, 1]$ satisfies $\Delta(s, a) > 0$ for all $s \in S \setminus \{ s_L \}$ and
  $a \in A \setminus \{ \pi^* (s) \}$.
  Let $C \ge 0$.
  The environment is in an adversarial regime with a $(\pi^*, \Delta, C)$ self-bounding constraint
  (or, more concisely, a $(\pi^*, \Delta, C)$-self-bounding regime)
  if it holds for any algorithm that
  \begin{align}
    \label{eq:self-bounding-MDP}
    \Reg_T(\pi^*) 
    \ge
    \E\left[
      \sum_{t=1}^T
      \sum_{s \in S \setminus \{ s_L \}}
      \sum_{a \in A \setminus \{ \pi^*(s) \}} \Delta(s, a) q_t(s, a)
    \right]
    -
    C.
  \end{align}
\end{definition}
We also denote $\Delta_{\min}=\min_{s \neq s_L, a \neq \pi^*(s)} \Delta(s,a)$.
As discussed by \citet{zimmert2021tsallis},
this is a general regime that includes stochastic environments with adversarial corruption,
where the parameter $C$ corresponds to the total amount of corruption.
For more details, see, e.g., \citep{zimmert2021tsallis,jin2020simultaneously,jin2021best}.

\subsection{Online shortest path problem}
In online shortest path problem,
the player is given a directed acyclic graph (DAG) $G = (V \cup \{ s, g \}, E)$,
where $s$ and $g$ are the source and target vertices,
respectively,
$V$ is the set of other vertices,
and $E$ is the set of edges.
Denote $m = |E|$ and $n = |V|$.
Let $L$ denote the maximum number of edges in a $s$-$g$ path in $G$.
In each round $t \in [T]$,
the environment chooses a loss function $\ell_t: E \rightarrow [0,1]$
and the player chooses a path $p_t$ from $s$ to $g$.
At the end of the round,
the player can only observe the aggregate loss feedback $c_t \in [0, 1]$
such that $\E\left[c_t \mid \ell_t, p_t \right] = \sum_{e \in p_t} \ell_t(e)$,
while the player does not observe the individual losses $\ell_t(e)$ for $e \in E$.
The definition of regret, the regimes of the environment, and other related concepts are defined in a similar way to in the case of episodic MDPs.
More details on the model are given in the appendix.
We also note that the online shortest path problem can be interpreted as an ``almost'' special case of episodic MDPs with known transitions,
but it is not necessarily an exact special case.
For details,
see Remark~\ref{remark:shortest-path} in the appendix.

\section{Core idea: construction of loss estimator with aggregate feedback}
\label{sec:estimator}
A key aspect of the proposed algorithms lies in how to estimate the loss function in a setting where only the limited aggregate loss feedback is available.
In this paper,
inspired by the approach of \citet{maiti2025efficient} to the online shortest path problem with bandit feedback,
we extend the idea to the MDP setting.
We begin by briefly reviewing their approach.

\subsection{Review of the approach of \citet{maiti2025efficient} for the online shortest path problem}
\label{sec:estimator_SP}
The online shortest path algorithm by \citet{maiti2025efficient}
maintains an $s$-$g$ flow $q_t \in \cQ \subseteq [0,1]^E$ of capacity $1$.
Note that $\cQ$ can be interpreted as a convex hull of the set of all $s$-$g$ paths.
In the following,
we denote 
$q(v) = \sum_{e \in \oes v} q(e)$ for any $q \in [0, 1]^E$ and $v \in V \cup \{ s \}$,
where $\oes v \subset E$ is the set of outgoing edges from $v$.
From $q_t$,
it samples a path $p_t \in \{ 0, 1 \}^E$ in a \textit{Markovian} way,
i.e.,
we choose a path as follows:
(i) We first initialize $p_t \in \{0, 1 \}^E$ 
by $p_t(e) = 0$ for all $e \in E$ and set $v \leftarrow s$.
(ii) While $v \neq g$,
Pick $e \in \oes v$ with probability $\frac{q_t(e)}{q_t(v)}$,
set $p_t(e) \leftarrow 1$,
and transition to $e$'s terminal vertex $e_+$,
i.e.,
set
$v \leftarrow e_+$.
We then have $\E\left[ p_t \mid q_t \right] = q_t$,
i.e.,
each edge $e \in E$ is included in the path $p_t$ with probability $q_t(e)$.

After constructing the path $p_t$ as described above and obtaining the aggregate loss feedback $c_t$ such that $\E\left[ c_t \mid \ell_t, p_t \right] = \linner \ell_t, p_t \rinner$,
the loss estimator $\widehat{L}_t(p)$ for any $s$-$g$ path 
$p = (s = v_0, e_0, v_1, e_1, \ldots, v_{L-1}, e_{L-1}, v_{L} = g)$
is defined as follows:\footnote{
  The construction method by \citet{maiti2025efficient} does not exactly match the one described below,
  as they add a uniform shift and incorporate implicit exploration~\citet{neu2015explore}.
  These adjustments are designed to obtain high-probability regret bounds,
  but they are not essential in this study,
  which focuses on expected regret bounds.
}
\begin{align}
  \label{eq:loss_estimator_SP}
  \widehat{L}_t(p) = 
  \sum_{k=0}^{L-1} 
  \frac{p_t(e_k)}{q_t(e_k)} c_t
  -
  \sum_{k=1}^{L-1} 
  \frac{p_t(v_k)}{q_t(v_k)} c_t
  =
  c_t
  \cdot
  \left(
  \sum_{e \in E} \frac{p(e)p_t(e)}{q_t(e)}
  -
  \sum_{v \in V} \frac{p(v)p_t(v)}{q_t(v)}
  \right).
\end{align}
We then have $\E\left[ \widehat{L}_t(p) \mid q_t, \ell_t \right] = \linner \ell_t, p \rinner$ for any $s$-$g$ path $p$.
In fact,
the conditional expectation given $q_t, \ell_t$ satisfies
\begin{align}
  \label{eq:loss_estimator_SP_e}
  \E \left[
    \frac{p_t(e_k)}{q_t(e_k)} c_t
  \right]
  &
  =
  \E \left[
    \linner
    \ell_t,
    p_t
    \rinner
    \mid
    p_t(e_k) = 1
  \right]
  =
  \bar{L}_t(s \to v_k)
  +
  \ell_t(e_k)
  +
  \bar{L}_t(v_{k+1} \to g),
  \\
  \label{eq:loss_estimator_SP_v}
  \E \left[
    \frac{p_t(v_k)}{q_t(v_k)} c_t
  \right]
  &
  =
  \E \left[
    \linner
    \ell_t,
    p_t
    \rinner
    \mid
    p_t(v_k) = 1
  \right]
  =
  \bar{L}_t(s \to v_k)
  +
  \bar{L}_t(v_{k} \to g),
\end{align}
where $\bar{L}_t (v \to v')$ represents 
the conditional expectation of the cost of the subpath of $p_t$ from $v$ to $v'$,
given that $p_t$ goes through $v$ and $v'$.
By combining \eqref{eq:loss_estimator_SP},
\eqref{eq:loss_estimator_SP_e}
and
\eqref{eq:loss_estimator_SP_v},
we obtain
$
  \E \left[
    \widehat{L}_t(p) \mid q_t, \ell_t
  \right]
  =
  \sum_{k=1}^{L-1} \ell_t(e_k)
  =
  \linner \ell_t, p \rinner.
$
In this paper,
we define the loss estimator $\widehat{\ell}_t \in \re^E$ by
  $
  \widehat{\ell}_t(e) = c_t
  \cdot
  \left(
    \frac{p_t(e)}{q_t(e)}
    -
    \frac{p_t(e_-)}{q_t(e_-)}
  \right),
$
where $e_- \in V \cup \{ s \}$ is the initial vertex of the edge $e \in E$.
As we have $\inner{\widehat{\ell}_t, p} = \widehat{L}_t(p) - c_t$ for any $s$-$g$ path $p$,
we can use $\widehat{\ell}_t$ as an loss estimator in our FTRL framework.

\subsection{Loss estimator for MDPs with known transition}
Let $q_t \in \cQ$ be the occupancy measure for the policy $\pi_t$.
Suppose that the trajectory $( (s_k^t, a_k^t) )_{k=0}^{L-1}$ is generated according to the policy $\pi_t$
and $c_t$ is the observed aggregate loss feedback.
For any $k \in \{0, 1, \ldots, L-1 \}$
and for any $s \in S_k$ and $a \in A$,
denote
$\ind_t(s) = \mathbb{I}[ s_k^t = s ]$ and $\ind_t(s, a) = \mathbb{I}[(s_k^t, a_k^t) = (s, a)]$.
Note that we then have 
$\E\left[ \ind_t(s) \mid \pi_t \right] = q_t(s)$
and
$\E\left[ \ind_t(s, a) \mid \pi_t \right] = q_t(s, a)$.
Inspired by the approach of \citet{maiti2025efficient},
we define the loss estimator as in the following lemma:
\begin{lemma}
  \label{lem:loss_estimator_KT}
  The loss estimator $\widehat{\ell}_t: S\times A \to \re$
  defined as
  $\widehat{\ell}_t(s,a) =
  c_t \cdot
  \left(
    \frac{\ind_t(s, a)}{q_t(s, a)}
    -
    \frac{\ind_t(s)}{q_t(s)}
  \right)
  $
  satisfies
  \begin{align}
    \label{eq:defellbar}
    \E\left[ \widehat{\ell}_t(s, a) \mid \ell_t, \pi_t \right] = Q^{\pi_t}(s, a; \ell_t) - V^{\pi_t}(s; \ell_t) 
    =: \bar{\ell}_t(s, a)
    .
  \end{align}
\end{lemma}
From this and Lemma~\ref{lem:PDL},
we have
$
    \Reg_T(\pi^*)
    =
    \E \left[
      \sum_{t=1}^T \left(
        V^{\pi_t} (s_0; \widehat{\ell}_t) - V^{\pi^*} (s_0; \widehat{\ell}_t)
      \right)
    \right]
    =
    \E \left[
      \sum_{t=1}^T 
      \linner
        \widehat{\ell}_t,
        q_t - q^*
      \rinner
    \right]
$.
Thanks to this,
we can use $\widehat{\ell}_t$ instead of $\ell_t$ in our FTRL-based algorithms.
\subsection{Loss estimator for MDPs with unknown transition}
In the case of unknown transitions,
when attempting to construct a loss estimator in the same manner as in Lemma~\ref{lem:loss_estimator_KT},
a key difficulty arises from the fact that the true value of $q_t$ is not available.
To address this issue,
one may follow the approach of \citet{jin2020learning} and compute an upper confidence bound $u_t$ for $q_t$,
using it as a surrogate in the estimator.
However,
naively replacing $q_t$ with $u_t$ in the definition of $\widehat{\ell}_t$ in Lemma~\ref{lem:loss_estimator_KT}
introduces yet another issue.
Specifically,
as $\widehat{\ell}_t$ in Lemma~\ref{lem:loss_estimator_KT} contains a negative term 
(i.e., $- c_t \frac{\ind_t(s)}{q_t(s)}$), 
substituting $q_t$ with its upper bound $u_t$ may lead to an undesirable \textit{positive bias} in the estimator,
which creates an obstacle in the regret analysis.
To derive a valid regret upper bound,
it is essential that the estimator is \textit{optimistic},
i.e.,
its expectation must act as a lower confidence bound on
$\bar{\ell}_t$ in \eqref{eq:defellbar}.
To this end,
we define the following novel loss estimator:
\begin{equation}
    \ell_t^u(s,a) = \frac{c_t \cdot \ind_t(s,a) + \left(1 - \pi_t(a \mid s) - c_t\right) \cdot \ind_t(s)\pi_t(a \mid s)}{u_t(s,a)} - \left(1 - \pi_t(a \mid s)\right).
    \label{eq:loss-estimator-unknown-0}
\end{equation}
We can evaluate the expectation of $\ell_t^u$ 
in a manner similar to the proof of Lemma~\ref{lem:loss_estimator_KT},
as follows:
\begin{equation}
\E \left[ \ell_t^u(s,a) \mid \ell_t, \pi_t, u_t \right]
=
\frac{q_t(s)}{u_t(s)} \left( Q^{\pi_t}(s,a;\ell_t) - V^{\pi_t}(s;\ell_t) + 1 - \pi_t(a|s) \right)
-
(1 - \pi_t(a|s)).
\label{eq:Eellu}
\end{equation}
We here have 
$
    Q^{\pi_t}(s,a; \ell_t) - V^{\pi_t}(s; \ell_t)
=
Q^{\pi_t}(s,a ; \ell_t ) - \sum_{a'\in A} \pi_t(a' \mid s) Q^{\pi_t}(s,a' ; \ell_t)
=
(1 - \pi_t(a|s)) Q^{\pi_t}(s,a ; \ell_t) - \sum_{a' \ne a} \pi_t(a' \mid s) Q^{\pi_t}(s,a' ;\ell_t)
\ge - (1 - \pi_t(a|s))
$
as $Q^{\pi_t}(s,a; \ell_t )\in[0,1]$,
and thus:
$
Q^{\pi_t}(s,a; \ell_t) - V^{\pi_t}(s; \ell_t) + 1 - \pi_t(a|s) \ge 0$.
Hence,
under the condition of $u_t(s) \ge q_t(s)$,
the value of \eqref{eq:Eellu} is a lower bound on 
$
\bar{\ell}_t(s, a)
:=
Q^{\pi_t}(s,a; \ell_t) - V^{\pi_t}(s; \ell_t) 
$,
i.e.,
$
\ell_t^u(s, a)
$
is an optimistic estimator of $\bar{\ell}_t(s, a)$.
In addition,
the gap between them is at most $\frac{u_t(s) - q_t(s)}{u_t(s)}(1 - \pi_t (a|s))$:
\begin{lemma}
  Under the condition of $u_t(s) \ge q_t(a)$,
  we have
  \begin{align}
    \bar{\ell}_t(s,a) 
    -
    \frac{u_t(s) - q_t(s)}{u_t(s)}(1 - \pi_t (a|s))
    \le 
    \E \left[\ell_t^u(s, a) \mid \ell_t, \pi_t, u_t \right] 
    \le
    \bar{\ell}_t(s,a) .
  \end{align}
\end{lemma}

\subsection{Second moment of loss estimators}
Our proposed algorithm,
like those of \citet{jin2020learning,jin2020simultaneously,jin2021best},
is based on the Follow-the-Regularized-Leader (FTRL) framework over occupancy measures.
In this framework,
the second moment of the loss estimator plays a crucial role.
The second moment of the loss estimator introduced in this section can be bounded as follow:
\begin{lemma}
  \label{lem:second-moment}
  Loss estimators $\widehat{\ell}_t$ in Lemma~\ref{lem:loss_estimator_KT}
  and $\ell_t^u$ in \eqref{eq:loss-estimator-unknown-0} satisfy
  \begin{align*}
    \E \left[
      \widehat{\ell}_t(s,a)^2
      \mid
      \ell_t,
      \pi_t
    \right]
    \le
    \frac{1 - \pi(a|s)}{q_t(s,a)},
    \quad
    \E \left[
      \ell_t^u(s,a)^2
      \mid
      \ell_t,
      \pi_t,
      u_t
    \right]
    \lesssim
    \frac{1 - \pi(a|s)}{u_t(s,a)} 
    \cdot
    \left(
    \frac{q_t(s)}{u_t(s)}
    + 
    1
    \right) .
  \end{align*}
\end{lemma}
When applying the self-bounding technique to derive an $O(\mathrm{polylog}\, T)$ regret bound,
the $(1 - \pi(a|s))$ factor in this lemma plays a crucial role.
In prior work \citep{jin2021best,jin2023no},
the original loss estimator did not exhibit this factor in its second moment,
and hence the analysis relied on a carefully designed \textit{shifting function} to apply a \textit{loss-shifting trick} and extract the desired $(1 - \pi(a|s))$ factor.
However,
this significantly complicated the analysis.
In contrast,
our regret analysis does not require the loss-shifting trick,
as the self-bounding technique can be applied directly.
As a result, we avoid the technically involved analysis necessitated by the loss-shifting trick in previous work.




\section{Algorithm and regret bounds}

\subsection{Warmup: online shortest path problem}
\label{sec:OSP}
As a warm-up,
let us consider the algorithm for the online shortest path problem.
Following the approach of \citet{maiti2025efficient},
we update a point $q_t$ on the $s$-$g$ unit flow polytope $\cQ$
(i.e., the convex hull of all $s$-$g$ paths) using the following FTRL framework:
$
  q_t \in
  \argmin_{q \in \cQ}
  \left\{
    \linner
    \sum_{\tau=1}^{t-1} \widehat{\ell}_\tau,
    q
    \rinner
    +
    \psi_t(q)
  \right\}
$,
where $\widehat{\ell}_\tau$ is given as in Section~\ref{sec:estimator_SP}
and $\psi_t(q)$ is a regularizer function defined as:
\begin{align}
  \psi_t(q)
  &
  =
  -
  \frac{2}{\eta_t}
  \sum_{e \in E}
  \sqrt{q(e)}
  -
  \sum_{e \in E}
  \beta
  \ln q(e)
  \quad
  \mbox{with}
  \quad
  \eta_t = \frac{1}{\sqrt{t}},
  ~
  \beta = \Theta(1),
  \tag{Tsallis entropy}
  \label{eq:defpsi-SP-Tsallis}
  \\
  \psi_t(q)
  &
  =
  -
  \sum_{e \in E}
  \frac{1}{\eta_t(e)}
  \ln q(e)
  \quad
  \mbox{with}
  \quad
  \eta_t(e) 
  = \big(4 + \frac{1}{\ln T} \sum_{\tau = 1}^{t-1} 
  \rho_{\tau}(e)
  \big)^{-\frac{1}
  {2}},
  \tag{log-barrier}
  \label{eq:defpsi-SP-LB}
\end{align}
where we define 
$\rho_t(e) = c_{t}^2 p_t(e) \big( 1 - \frac{q_t(e)}{q_{t}(e_-)} \big)^2 $.
We then have the following regret upped bounds:
\begin{theorem}
  \label{thm:OSP}
  Let $p^* \in \{0 , 1 \}^E$ be an arbitrary $s$-$g$ path
  and $\pi^*: V \cup \{s\} \rightarrow E$ be such that
  $\pi^*(v) \in \oes v$ for all $v \in V \cup \{ s \}$ and $p^*(e) = 1 \Longrightarrow \pi^*(e_-) = e$.
  In the case of \ref{eq:defpsi-SP-Tsallis},
  \begin{align*}
    \Reg_T(p^*)
    \lesssim
    \sum_{t=1}^T\displaystyle
    \frac{1}{\sqrt{t}}
    \E
    \Big[
      \!\sum_{v \in V \cup \{ s \}}\!
      \sum_{e \in \oes v \setminus \{ \pi^*(v) \}}\!
      \sqrt{q_t(e)}
      +\!
      \sum_{v \in V \setminus V^*}\!
      \sqrt{ q_t(\pi^*(v)) }
    \Big]
    +
    \sqrt{mnL}
    +
    m \log T.
  \end{align*}
  If $\psi_t$ is given by \ref{eq:defpsi-SP-LB} regularizer,
  we have
  \begin{align*}
    \Reg_T
    \lesssim
    \E \Big[
      \sum_{v \in V \cup \{ s \}}
      \sum_{e \in \oes v}
      \sqrt{ \textstyle\sum_{t=1}^T\displaystyle c_t^2 p_t(e) ( 1 - \tfrac{q_t(e)}{q_t(v)} )^2 \log (T) }
    \Big] + m \log (T).
  \end{align*}
\end{theorem}
\begin{corollary}\label{cor:bobw-shortest-path}
  We have
  $\Reg_T \lesssim \sqrt{mL(n+T)} + n \log T$
  in the \ref{eq:defpsi-SP-Tsallis} case
  and
  $\Reg_T \lesssim \sqrt{mL(n+T) \log T}$
  in the \ref{eq:defpsi-SP-LB} case.
  Simultaneously,
  under the condition of 
  $
    \Reg_T(p^*)
    \ge
    \E\left[
      \sum_{t=1}^T \sum_{v \in V \cup \{ s \}} \sum_{e \in \oes v \setminus \{ \pi^*(v) \}} \Delta(e) p_t(e) \right] - C 
  $ 
  for some $\Delta \in [0,1]^E$ and $C \ge 0$, 
  we have 
  $\Reg_T(p^*) \lesssim U + \sqrt{UC}$, 
  where 
  $ 
    U 
    = 
    \sum_{v \in V \cup \{ s \}} \sum_{e \in \oes v \setminus \{ \pi^*(v) \}} \frac{\log T}{\Delta(e)} 
    +
    \frac{n^2 \log T}{\Delta_{\min}} 
  $ 
  in the \ref{eq:defpsi-SP-Tsallis} case and
  $ 
    U 
    =
    \sum_{v \in V \cup \{ s \}} \sum_{e \in \oes v \setminus \{ \pi^*(v) \}} \frac{\log T}{\Delta(e)}
  $
  for the \ref{eq:defpsi-SP-LB} case.
\end{corollary}
The tightness of the gap-dependent upper bound derived here is discussed in the appendix.

\subsection{MDPs with known transition}
\label{sec:known_transition}
The algorithm design for episodic MDPs with known transitions is almost identical to the case of the shortest path problem.
Specifically,
we apply the FTRL framework over the set of all occupancy measures as the feasible region,
replace each edge $e \in E$ in the regularization functions in \eqref{eq:defpsi-SP-Tsallis} and \eqref{eq:defpsi-SP-LB}
with a state-action pair $(s, a) \in S \subseteq \{s_L\} \times A$,
and redefine $\rho_t$ as
$
\rho_t(s, a) = c_t^2 \ind_t(s, a) \left(1 - \pi_t(a \mid s)\right)^2
$.
With this setup,
we obtain the following regret bound:
\begin{corollary}\label{cor:result-known-mdp}
  We have
  $\Reg_T \lesssim \sqrt{|S||A|LT} + |S||A| \log T$
  in the \ref{eq:defpsi-SP-Tsallis} case
  and
  $\Reg_T \lesssim \sqrt{|S||A|LT \log T}$
  in the \ref{eq:defpsi-SP-LB} case.
  Simultaneously,
  under the condition of \eqref{eq:self-bounding-MDP},
  we have
  $\Reg_T(\pi^*) \lesssim U + \sqrt{UC}$, 
  where 
  $ 
    U 
    = 
    \sum_{s \neq s_{L}} \sum_{a \neq \pi^*(s)} \frac{\log T}{\Delta(s,a)} 
    +
    \frac{L |S| \log T}{\Delta_{\min}} 
  $ 
  in the \ref{eq:defpsi-SP-Tsallis} case and
  $ 
    U 
    =
    \sum_{s \neq s_{L}} \sum_{a \neq \pi^*(s)} \frac{\log T}{\Delta(s,a)} 
  $
  for the \ref{eq:defpsi-SP-LB} case.
\end{corollary}
The gap-dependent upper bound achieved by the log-barrier regularization in this corollary is tight.
In fact,
the following lower bound holds:
\begin{theorem}
  Consider stochastic environment in which $c_t$ follows a Bernoulli distribution of parameter 
  $\sum_{k=1}^{L-1} \ell^*(s_t^k, a_t^k)$ 
  where we assume that this value is in $[3/8, 5/8]$ for any possible trajectories.
  Define $\Delta:S \times A \to [0, 1]$ by $\Delta(s,a) = Q^{\pi^*}(s, a; \ell^*) - V^{\pi^*}(s; \ell^*)$
  for an optimal policy $\pi^*$.
  Let $S^* $ be the set of all states
  $s \in S \setminus \{ s_L \}$
  such that $q^{\pi^*}(s)>0$ for some optimal policy $\pi^*$.
  Then,
  for any consistent algorithms,
  we have
  $
  \liminf_{T \to \infty} \frac{\Reg_T}{\log T} 
  \gtrsim \sum_{s \in S^*}\sum_{a \in A : \Delta(s,a) > 0} 
  \frac{1}{\Delta(s,a)}
  $.
\end{theorem}

\subsection{MDPs with unknown transition}
\label{sec:unknown_transition}
Our proposed algorithm for MDPs with unknown transitions adopts an epoch-based approach, similar to prior work~\citep{jin2021best,jin2020learning}. 
In each epoch $i$, it updates both the transition probability estimates and their corresponding confidence intervals. Based on these, we compute an upper confidence bound $u_t$ on the occupancy measure $q_t$, and use it to construct the loss estimator $\ell_t^u$ as defined in~\eqref{eq:loss-estimator-unknown-0}. Note that $u_t(s,a)$ can be efficiently computed using the \textsc{Comp-UOB} procedure proposed by~\citep{jin2020learning}.

We then define the adjusted loss estimator as $\widehat{\ell}_t := \ell_t^u - B_i$, where $B_i$ is a bonus term derived from the confidence width. Unlike prior work, our choice of the loss estimator $\ell_t^u$ allows us to avoid scaling $B_i$ by an additional factor of $L$. The policy for each episode is selected by applying FTRL over the estimated occupancy measure space using $\widehat{\ell}_t$. The regularization function used here matches the Tsallis entropy regularizer from Section~\ref{sec:known_transition}, except that the learning rate $\eta_t$ is reset at the beginning of each epoch, and a small log-barrier term is added to stabilize updates.

A notable improvement over prior work~\citep{jin2021best} is that the second-moment bound established in Lemma~\ref{lem:second-moment} allows us to bypass the loss-shifting technique. As a result, our regret bounds exhibit improved dependence on the horizon $L$. We refer the reader to Algorithm~\ref{alg:unknown-transitions} and the appendix \ref{sec:unknown_transition_appendix} for full details. Our algorithm, constructed in this way, achieves the following upper bounds:

\begin{theorem}\label{thm:main_bandit}
In the bandit feedback setting,  Algorithm \ref{alg:unknown-transitions} with $\delta = \frac{1}{T^3}$ and $\iota=\frac{|S||A|T}{\delta}$ guarantees $\Reg_T(\pi^\star)= \tilde{\mathcal{O}}\rbr{  L|S|\sqrt{|A|T} + |S||A| \sqrt{LT} + L^2|S|^3|A|^2 }$ and simultaneously $\Reg_T(\pi^\star)=\order\rbr{U + \sqrt{UC}+ V }$ under Condition~\eqref{eq:self-bounding-MDP},   
	where $V =  L^2|S|^3|A|^2 \ln^2 \iota$ and $U$ is defined as
	\[
	U =  \sum_{s \neq s_L} \sum_{ a\neq \pi^\star(s)} \sbr{ \frac{L^4|S|\ln \iota +  |S||A|\ln^2 \iota}{\gap(s,a)}} + \sbr{\frac{(L^4|S|^2+L^3|S|^2|A|) \ln \iota +  L|S|^2|A| \ln^2 \iota}{\gapmin}}.
	\]
    \label{prop:bobw_bandit_stoc_prop}
\end{theorem} 
We defer the proof of the above theorem to Appendix \ref{sec:analysis-bobw-unknown}.


\begin{algorithm}
\caption{BOBW algorithm for MDPs with unknown transitions and aggregate bandit feedback}\label{alg:unknown-transitions}
\begin{algorithmic}[1]
\State \textbf{Input:} confidence parameter $\delta \in (0,1)$
\State \textbf{Initialize:} epoch index $i = 1$ and epoch starting time $t_i = 1$; 
\State $\forall(s,a,s')$, set counters $m_1(s,a) = m_1(s,a,s') = m_0(s,a) = m_0(s,a,s') = 0$;
\State empirical transition $\bar{P}_1$ and confidence width $B_1$ based on Eq.~(2);
\For{$t = 1, \dots, T$}
    \State Let $\phi_t$ be the regularizer defined in Eq.~\eqref{eq:reg-unknown-appendix} and compute
    \[
    \widehat{q}_t = \argmin_{q \in \Omega(\bar{P}_i)} \textstyle\sum_{\tau = t_i}^{t-1} \langle q,  \widehat \ell_t \rangle + \phi_t(q),
    \]
    
    where $\widehat \ell_t={\ell}_\tau^u-B_i$ and $B_i(s,a)=\min\{2,\sum_{s'\in S_{k(s)+1}}B_i(s,a,s')\}$.
    \State Compute policy $\pi_t$ from $\widehat{q}_t$ such that $\pi_t(a|s) \propto \widehat{q}_t(s,a)$.
    \State Execute policy $\pi_t$ and obtain trajectory $(s_{t,k}, a_{t,k})$ for $k = 0, \dots, L-1$.
    \State Construct loss estimator ${\ell}_t^u$ as defined in Eq. \eqref{eq:loss-estimator-unknown-0}.
    \State Increment counters: for each $k < L$, 
    \[
        m_i(s_{t,k}, a_{t,k}, s_{t,k+1}) \gets m_i(s_{t,k}, a_{t,k}, s_{t,k+1})+1,\quad m_i(s_{t,k}, a_{t,k}) \gets m_i(s_{t,k}, a_{t,k})+1.
    \]
    \If{$\exists k, \; m_i(s_{t,k}, a_{t,k}) \ge \max\{1, 2m_{i-1}(s_{t,k}, a_{t,k})\}$} \Comment{entering a new epoch}
        \State Increment epoch index $i \gets i + 1$ and set new epoch starting time $t_i = t + 1$.
        \State Initialize new counters:
        \[
        \forall(s,a,s'), m_i(s,a,s') = m_{i-1}(s,a,s'), \quad m_i(s,a) = m_{i-1}(s,a).
        \]
        \State Update empirical transition $\bar{P}_i$ and confidence width $B_i$ based on Eq.~\eqref{eq:emp-transition-appendix} and \eqref{eq:conf-wid-appendix}.
    \EndIf
\EndFor
\end{algorithmic}
\end{algorithm}

\section{Conclusion}
This paper initiated the study of best-of-both-worlds (BOBW) algorithms for finite-horizon episodic MDPs with aggregate bandit feedback.
We proposed efficient algorithms that achieve low regret in both stochastic and adversarial settings,
and established nearly tight upper and lower bounds under both known- and unknown-transition settings.
Our approach is built upon FTRL over occupancy measures,
combined with carefully designed loss estimators that are optimistic in expectation and admit tight second-moment bounds.

Despite these contributions,
many open questions remain.
A central limitation of our approach is its reliance on occupancy measure updates via FTRL,
which---while grounded in convex optimization and thus computationally feasible to some extent---still requires solving a convex problem in each round.
Moreover,
this framework does not easily extend beyond tabular MDPs to more general representations such as linear models or function approximation.

A promising direction to address these limitations is to adopt policy optimization-based methods~\citep{shani2020optimistic,luo2021policy}.
In particular,
a recent paper by \citep{lancewicki2025near}
has shown that near-optimal and efficient adversarial regret bounds can be achieved through policy optimization.
Combining this line of work with the techniques in \citep{dann2023blackbox} may yield BOBW algorithms that are both computationally efficient and more broadly applicable.

Another important direction for future research is to extend the present results beyond the online shortest path problem to other combinatorial optimization settings, or to more challenging MDP formulations such as stochastic shortest path problems~\citep{chen2021minimax}. Addressing these challenges may lead to a more comprehensive understanding of online learning under aggregate feedback.

\section*{Acknowledgments}
Ito is supported by JSPS KAKENHI Grant Number JP25K03184.
Jamieson is funded in part by NSF Award CAREER 2141511 and Microsoft Grant for Customer Experience Innovation.
Luo is funded by NSF award IIS-1943607.
Tsuchiya is supported by JSPS KAKENHI Grant Number JP24K23852.

\bibliographystyle{plainnat}
\bibliography{reference}

\begin{thebibliography}{33}
\providecommand{\natexlab}[1]{#1}
\providecommand{\url}[1]{\texttt{#1}}
\expandafter\ifx\csname urlstyle\endcsname\relax
  \providecommand{\doi}[1]{doi: #1}\else
  \providecommand{\doi}{doi: \begingroup \urlstyle{rm}\Url}\fi

\bibitem[Bubeck and Slivkins(2012)]{bubeck2012best}
S{\'e}bastien Bubeck and Aleksandrs Slivkins.
\newblock The best of both worlds: Stochastic and adversarial bandits.
\newblock In \emph{Proceedings of the 25th Annual Conference on Learning Theory}, volume~23, pages 42.1--42.23, 2012.

\bibitem[Bubeck et~al.(2012)Bubeck, Cesa-Bianchi, and Kakade]{bubeck2012towards}
S{\'e}bastien Bubeck, Nicolo Cesa-Bianchi, and Sham Kakade.
\newblock Towards minimax policies for online linear optimization with bandit feedback.
\newblock In \emph{Conference on Learning Theory}, volume~23, pages 41.1--41.14, 2012.

\bibitem[Bubeck et~al.(2018)Bubeck, Cohen, and Li]{bubeck18sparsity}
S\'ebastien Bubeck, Michael Cohen, and Yuanzhi Li.
\newblock Sparsity, variance and curvature in multi-armed bandits.
\newblock In \emph{Proceedings of Algorithmic Learning Theory}, volume~83, pages 111--127. PMLR, 2018.

\bibitem[Canonne(2022)]{canonne2022short}
Cl{\'e}ment~L Canonne.
\newblock A short note on an inequality between {KL} and {TV}.
\newblock \emph{arXiv preprint arXiv:2202.07198}, 2022.

\bibitem[Cassel et~al.(2024)Cassel, Luo, Rosenberg, and Sotnikov]{cassel24near}
Asaf Cassel, Haipeng Luo, Aviv Rosenberg, and Dmitry Sotnikov.
\newblock Near-optimal regret in linear {MDP}s with aggregate bandit feedback.
\newblock In \emph{Proceedings of the 41st International Conference on Machine Learning}, volume 235, pages 5757--5791. PMLR, 2024.

\bibitem[Chatterji et~al.(2021)Chatterji, Pacchiano, Bartlett, and Jordan]{chatterji21theory}
Niladri Chatterji, Aldo Pacchiano, Peter Bartlett, and Michael Jordan.
\newblock On the theory of reinforcement learning with once-per-episode feedback.
\newblock In \emph{Advances in Neural Information Processing Systems}, volume~34, pages 3401--3412. Curran Associates, Inc., 2021.

\bibitem[Chen et~al.(2021)Chen, Luo, and Wei]{chen2021minimax}
Liyu Chen, Haipeng Luo, and Chen-Yu Wei.
\newblock Minimax regret for stochastic shortest path with adversarial costs and known transition.
\newblock In \emph{Conference on Learning Theory}, pages 1180--1215. PMLR, 2021.

\bibitem[Cohen et~al.(2021)Cohen, Kaplan, Koren, and Mansour]{cohen2021online}
Alon Cohen, Haim Kaplan, Tomer Koren, and Yishay Mansour.
\newblock Online {Markov} decision processes with aggregate bandit feedback.
\newblock In \emph{Conference on Learning Theory}, pages 1301--1329. PMLR, 2021.

\bibitem[Dann et~al.(2023{\natexlab{a}})Dann, Wei, and Zimmert]{dann2023blackbox}
Chris Dann, Chen-Yu Wei, and Julian Zimmert.
\newblock A blackbox approach to best of both worlds in bandits and beyond.
\newblock In \emph{The Thirty Sixth Annual Conference on Learning Theory}, pages 5503--5570. PMLR, 2023{\natexlab{a}}.

\bibitem[Dann et~al.(2023{\natexlab{b}})Dann, Wei, and Zimmert]{dann23best}
Christoph Dann, Chen-Yu Wei, and Julian Zimmert.
\newblock Best of both worlds policy optimization.
\newblock In \emph{Proceedings of the 40th International Conference on Machine Learning}, volume 202, pages 6968--7008. PMLR, 2023{\natexlab{b}}.

\bibitem[Efroni et~al.(2021)Efroni, Merlis, and Mannor]{efroni2021reinforcement}
Yonathan Efroni, Nadav Merlis, and Shie Mannor.
\newblock Reinforcement learning with trajectory feedback.
\newblock In \emph{Proceedings of the AAAI conference on artificial intelligence}, volume 35-8, pages 7288--7295, 2021.

\bibitem[Erez and Koren(2021)]{erez2021best}
Liad Erez and Tomer Koren.
\newblock Towards best-of-all-worlds online learning with feedback graphs.
\newblock In \emph{Advances in Neural Information Processing Systems}, volume~34, pages 28511--28521. Curran Associates, Inc., 2021.

\bibitem[Ito(2021)]{ito2021parameter}
Shinji Ito.
\newblock Parameter-free multi-armed bandit algorithms with hybrid data-dependent regret bounds.
\newblock In \emph{Conference on Learning Theory}, pages 2552--2583. PMLR, 2021.

\bibitem[Ito et~al.(2022)Ito, Tsuchiya, and Honda]{ito2022nearly}
Shinji Ito, Taira Tsuchiya, and Junya Honda.
\newblock Nearly optimal best-of-both-worlds algorithms for online learning with feedback graphs.
\newblock In \emph{Advances in Neural Information Processing Systems}, volume~35, pages 28631--28643. Curran Associates, Inc., 2022.

\bibitem[Ito et~al.(2024)Ito, Tsuchiya, and Honda]{ito2024adaptive}
Shinji Ito, Taira Tsuchiya, and Junya Honda.
\newblock Adaptive learning rate for follow-the-regularized-leader: Competitive analysis and best-of-both-worlds.
\newblock In \emph{The Thirty Seventh Annual Conference on Learning Theory}, pages 2522--2563. PMLR, 2024.

\bibitem[Jin et~al.(2020)Jin, Jin, Luo, Sra, and Yu]{jin2020learning}
Chi Jin, Tiancheng Jin, Haipeng Luo, Suvrit Sra, and Tiancheng Yu.
\newblock Learning adversarial {Markov} decision processes with bandit feedback and unknown transition.
\newblock In \emph{International Conference on Machine Learning}, pages 4860--4869. PMLR, 2020.

\bibitem[Jin and Luo(2020)]{jin2020simultaneously}
Tiancheng Jin and Haipeng Luo.
\newblock Simultaneously learning stochastic and adversarial episodic {MDPs} with known transition.
\newblock In \emph{Advances in Neural Information Processing Systems}, volume~33, pages 16557--16566. Curran Associates, Inc., 2020.

\bibitem[Jin et~al.(2021)Jin, Huang, and Luo]{jin2021best}
Tiancheng Jin, Longbo Huang, and Haipeng Luo.
\newblock The best of both worlds: stochastic and adversarial episodic {MDPs} with unknown transition.
\newblock In \emph{Advances in Neural Information Processing Systems}, volume~34, pages 20491--20502. Curran Associates, Inc., 2021.

\bibitem[Jin et~al.(2023)Jin, Liu, Rouyer, Chang, Wei, and Luo]{jin2023no}
Tiancheng Jin, Junyan Liu, Chlo\'{e} Rouyer, William Chang, Chen-Yu Wei, and Haipeng Luo.
\newblock No-regret online reinforcement learning with adversarial losses and transitions.
\newblock In \emph{Advances in Neural Information Processing Systems}, volume~36, pages 38520--38585. Curran Associates, Inc., 2023.

\bibitem[Lancewicki and Mansour(2025)]{lancewicki2025near}
Tal Lancewicki and Yishay Mansour.
\newblock Near-optimal regret using policy optimization in online {MDPs} with aggregate bandit feedback.
\newblock \emph{arXiv preprint arXiv:2502.04004}, 2025.

\bibitem[Lattimore and Szepesvari(2017)]{lattimore2017end}
Tor Lattimore and Csaba Szepesvari.
\newblock The end of optimism? an asymptotic analysis of finite-armed linear bandits.
\newblock In \emph{Artificial Intelligence and Statistics}, pages 728--737. PMLR, 2017.

\bibitem[Lattimore and Szepesv{\'a}ri(2020)]{lattimore2020bandit}
Tor Lattimore and Csaba Szepesv{\'a}ri.
\newblock \emph{Bandit algorithms}.
\newblock Cambridge University Press, 2020.

\bibitem[Lee et~al.(2021)Lee, Luo, Wei, Zhang, and Zhang]{lee2021achieving}
Chung-Wei Lee, Haipeng Luo, Chen-Yu Wei, Mengxiao Zhang, and Xiaojin Zhang.
\newblock Achieving near instance-optimality and minimax-optimality in stochastic and adversarial linear bandits simultaneously.
\newblock In \emph{International Conference on Machine Learning}, pages 6142--6151. PMLR, 2021.

\bibitem[Luo et~al.(2021)Luo, Wei, and Lee]{luo2021policy}
Haipeng Luo, Chen-Yu Wei, and Chung-Wei Lee.
\newblock Policy optimization in adversarial {MDPs}: Improved exploration via dilated bonuses.
\newblock \emph{Advances in Neural Information Processing Systems}, 34:\penalty0 22931--22942, 2021.

\bibitem[Maiti et~al.(2025)Maiti, Fan, Jamieson, Ratliff, and Farina]{maiti2025efficient}
Arnab Maiti, Zhiyuan Fan, Kevin Jamieson, Lillian~J Ratliff, and Gabriele Farina.
\newblock Efficient near-optimal algorithm for online shortest paths in directed acyclic graphs with bandit feedback against adaptive adversaries.
\newblock \emph{arXiv preprint arXiv:2504.00461}, 2025.

\bibitem[Masoudian et~al.(2024)Masoudian, Zimmert, and Seldin]{masoudian24best}
Saeed Masoudian, Julian Zimmert, and Yevgeny Seldin.
\newblock A best-of-both-worlds algorithm for bandits with delayed feedback with robustness to excessive delays.
\newblock In \emph{Advances in Neural Information Processing Systems}, volume~37, pages 141071--141102. Curran Associates, Inc., 2024.

\bibitem[Neu(2015)]{neu2015explore}
Gergely Neu.
\newblock Explore no more: Improved high-probability regret bounds for non-stochastic bandits.
\newblock In \emph{Advances in Neural Information Processing Systems}, volume~28, pages 3168--3176. Curran Associates, Inc., 2015.

\bibitem[Shani et~al.(2020)Shani, Efroni, Rosenberg, and Mannor]{shani2020optimistic}
Lior Shani, Yonathan Efroni, Aviv Rosenberg, and Shie Mannor.
\newblock Optimistic policy optimization with bandit feedback.
\newblock In \emph{International Conference on Machine Learning}, pages 8604--8613. PMLR, 2020.

\bibitem[Tsuchiya et~al.(2023{\natexlab{a}})Tsuchiya, Ito, and Honda]{tsuchiya23best}
Taira Tsuchiya, Shinji Ito, and Junya Honda.
\newblock Best-of-both-worlds algorithms for partial monitoring.
\newblock In \emph{Proceedings of The 34th International Conference on Algorithmic Learning Theory}. PMLR, 2023{\natexlab{a}}.

\bibitem[Tsuchiya et~al.(2023{\natexlab{b}})Tsuchiya, Ito, and Honda]{tsuchiya23stability}
Taira Tsuchiya, Shinji Ito, and Junya Honda.
\newblock Stability-penalty-adaptive follow-the-regularized-leader: Sparsity, game-dependency, and best-of-both-worlds.
\newblock In \emph{Advances in Neural Information Processing Systems}, volume~36, 2023{\natexlab{b}}.

\bibitem[Wei and Luo(2018)]{wei2018more}
Chen-Yu Wei and Haipeng Luo.
\newblock More adaptive algorithms for adversarial bandits.
\newblock In \emph{Conference On Learning Theory}, pages 1263--1291. PMLR, 2018.

\bibitem[Zimmert and Seldin(2020)]{zimmert20optimal}
Julian Zimmert and Yevgeny Seldin.
\newblock An optimal algorithm for adversarial bandits with arbitrary delays.
\newblock In \emph{Proceedings of the Twenty Third International Conference on Artificial Intelligence and Statistics}, volume 108, pages 3285--3294. PMLR, 2020.

\bibitem[Zimmert and Seldin(2021)]{zimmert2021tsallis}
Julian Zimmert and Yevgeny Seldin.
\newblock Tsallis-{INF}: An optimal algorithm for stochastic and adversarial bandits.
\newblock \emph{Journal of Machine Learning Research}, 22\penalty0 (28):\penalty0 1--49, 2021.

\end{thebibliography}
\newpage
\appendix

\addcontentsline{toc}{section}{Appendix}
\tableofcontents

\section{Additional related work}\label{sec:additional_related_work}


\paragraph{FTRL and best-of-both-worlds algorithms}

In episodic tabular MDPs with adversarial losses, \citet{jin2020simultaneously} is the first to propose a BOBW algorithm under known transitions.
\citet{jin2021best} extended this to the unknown-transition setting, which is further improved and extend to the setting where the transition can vary over episodes \citep{jin2023no}.
Subsequently, policy optimization  algorithms was shown to achieve BOBW guarantees with improved computational efficiency~\citep{dann23best}.
In spite of these developments, our work is the first to consider BOBW algorithms under aggregate feedback.
We build on the analysis of \citet{jin2021best}. While we may improve our bounds by using the optimistic transition technique from the recent work by \citet{jin2023no}, it remains unclear whether this technique can be effectively combined with our loss estimator, which can be negative.

A key challenge in achieving BOBW is the design and analysis of the regularizer in FTRL. In this work, we consider two types of regularizers (see Section~\ref{sec:OSP}). 
The first one is a hybrid regularizer~\citep{bubeck18sparsity} that combines the Tsallis entropy with a small-coefficient log-barrier. 
The first BOBW algorithm based on the Tsallis entropy was initially explored by \citet{zimmert2021tsallis},
and the hybrid regularizers to ensure the stability of FTRL have been shown to be powerful in obtaining BOBW guarantees for complex online learning problems or for obtaining adaptive bounds \citep{zimmert20optimal,erez2021best,ito2022nearly,ito2024adaptive,tsuchiya23best,tsuchiya23stability,masoudian24best}.
The second regularizer we consider is the log-barrier regularizer with adaptive learning rates, developed in \citet{wei2018more,ito2021parameter}. 
As we show in this paper, although the strong regularization of the log-barrier can introduce an additional $O(\sqrt{\log T})$ multiplicative factor in adversarial settings, it ensures a strong stability of FTRL.

\paragraph{Episodic MDPs with aggregate feedback.} Recently, MDPs with aggregate feedback have received growing attention. For example, in tabular MDPs, aggregate feedback has been studied in both the stochastic and adversarial settings (see \citep{efroni2021reinforcement, cohen2021online, chatterji21theory}), as well as in the context of policy optimization~\citep{lancewicki2025near}. Similar interest has emerged for linear MDPs as well (see \citep{cassel24near}). However, to the best of our knowledge, our work is the first to study best-of-both-worlds guarantees in the setting of aggregate feedback.



\section*{Remarks in comparing results}
\begin{remark}[On the scale of loss]
  \label{rem:scale}
  In existing studies on online learning for MDPs with adversarial losses,
  it's common to assume that $\ell_t(s,a) \in [0,1]$ for all $s$ and $a$.
  Such setting is reduced to our setting by scaling the losses by a factor of $1/L$,
  which therefore can be regarded as a special case of our setting.
  If the loss is given by this reduction, values of losses have the same scale of $O(1/L)$ for all layers.
  In contrast,
  in our setting,
  the losses may have different scales (possibly $>1/L$) for each layer,
  which can be interpreted to be more general.\kevin{This is already stated above, cut there or here.}
\end{remark}

\begin{remark}[On online shortest path problems and episodic MDPs]
  \label{remark:shortest-path}
  The online shortest path problem can be interpreted as an ``almost'' special case of episodic MDPs with known transitions,
  but it is not necessarily an exact special case.
  Intuitively,
  the vertices $v \in V$ in the shortest path problem correspond to the states $s \in S$ in an MDP,
  and selecting one outgoing edge $e \in \oes v$ from a vertex $v$ corresponds to choosing an action $a \in A$ in the MDP.
  The shortest path problem however differs in several aspects:
  the set of vertices $V$ does not necessarily have a hierarchical structure,
  the number of edges in a path from the source to the sink is not necessarily fixed,
  and the set $\oes v$ of outgoing edges available for selection can vary depending on the vertex $v$.
  Therefore,
  the shortest path problem cannot always be directly interpreted as an MDP.
  Consequently,
  regret bounds for MDPs do not immediately translate to results for the shortest path problem.
  We therefore provide a separate discussion on the online shortest path problem.
  However,
  the overall framework for algorithm design and analysis is similar to that for episodic MDPs with known transitions.
\end{remark}

\section{Auxiliary lemmas}
\subsection{Lemmas for FTRL}
\subsubsection{Stability terms for one-dimensional functions}
\begin{lemma}
	\label{lem:stabTsallis}
	Let $\phi: \re_{\ge 0} \rightarrow \re$ be defined as $\phi(x) = - 2\sqrt{x}$
	and $D_{\phi}(y, x)$ be the Bregman divergence associated with $\phi$,
	i.e.,
	\begin{align*}
		D_{\phi}(y, x)
		=
		- 2\sqrt{y} + 2 \sqrt{x}
		+ \frac{1}{\sqrt{x}} (y-x)
		=
		\frac{1}{\sqrt{x}}\left(
			\sqrt{x} - \sqrt{y}
		\right)^2.
	\end{align*}
	Then,
	for any $x \in (0,1)$, $\ell \in \re$ and $\eta > 0$ such that $\eta \sqrt{x} \ell  > -1$,
	we have
	\begin{align*}
		\sup_{y \in [0,1]}
		\left\{
		\ell \cdot (x - y) - \frac{1}{\eta} D_{\phi}(y, x)
		\right\}
		\le
		\frac{\eta x^{3/2} \ell^2}{1 + \eta \ell \sqrt{x}}.
	\end{align*}
\end{lemma}
\begin{proof}
	We have
	\begin{align*}
		&
		\ell \cdot (x - y) - \frac{1}{\eta} D_{\phi}(y, x)
		=
		\frac{1}{\eta}
		\left(
		2\sqrt{y} - 2 \sqrt{x}
		+ \left(\frac{1}{\sqrt{x}} + \eta \ell \right) (x-y)
		\right)
		\\
		&
		=
		\frac{1}{\eta}
		\left(
			-
			y
			\left(\frac{1}{\sqrt{x}} + \eta \ell \right)^{-1}
			\left(
				\left(\frac{1}{\sqrt{x}} + \eta \ell -  \frac{1}{\sqrt{y}} \right)^2
				-
				\frac{1}{y}
			\right)
			-
			\sqrt{x}
			+
			\eta \ell x
		\right)
		\\
		&
		=
		\frac{1}{\eta}
		\left(
			-
			y
			\left(\frac{1}{\sqrt{x}} + \eta \ell \right)^{-1}
			\left(\frac{1}{\sqrt{x}} + \eta \ell -  \frac{1}{\sqrt{y}} \right)^2
			+
			\left(\frac{1}{\sqrt{x}} + \eta \ell \right)^{-1}
			-
			\sqrt{x}
			+
			\eta \ell x
		\right)
		\\
		&
		\le
		\frac{1}{\eta}
		\left(
			\left(\frac{1}{\sqrt{x}} + \eta \ell \right)^{-1}
			-
			\sqrt{x}
			+
			\eta \ell x
		\right)
		\\
		&
		=
		\frac{\sqrt{x}}{\eta}
		\left(
			\frac{1}{ 1 + \eta \ell \sqrt{x} }
			-
			1
			+
			\eta \ell \sqrt{x}
		\right)
		=
		\frac{\sqrt{x}}{\eta (1 + \eta \ell \sqrt{x})}
		\left(
			\eta^2 \ell^2 x
		\right)
		=
		\frac{\eta x^{3/2} \ell^2}{1 + \eta \ell \sqrt{x}}.
	\end{align*}
\end{proof}
\begin{lemma}
	\label{lem:stabTsallisLB}
	Let $\eta > 0$ and $\beta > 0$.
	Suppose that $\phi: \re_{>0} \rightarrow \re$ is defined as
	$\phi(x) = -\frac{2}{\eta} \sqrt{x} - \beta \ln(x)$.
	Let $D_{\phi}(y, x)$ be the Bregman divergence associated with $\phi$.
	Then,
	for any $x \in (0, 1)$,
	$\ell \in \re$
	and $\eta$ such that $ x \ell \ge - \frac{\beta}{2}$,
	we have
	\begin{align}
		\sup_{y \in [0,1]}
		\left\{
		\ell \cdot (x - y) - D_{\phi}(y, x)
		\right\}
		\le
		6 {\eta x^{3/2} \ell^2}.
		\label{eq:lemstabTsallisLB}
	\end{align}
\end{lemma}
\begin{proof}
	If $\ell \ge 0$,
	it immediately follows from Lemma~\ref{lem:stabTsallis} that
	the left-hand side of \eqref{eq:lemstabTsallisLB} is bounded by
	$\eta x^{3/2} \ell^2$ from above.
	We next consider the case of $\ell < 0$.
	The derivative of 
	$ \ell \cdot (x - y) - D_{\phi}(y, x)$ in $y$ is given as
	\begin{align*}
		g(y)
		:=
		-
		\ell
		+
		\frac{1}{\eta \sqrt{y}}
		-
		\frac{1}{\eta \sqrt{x}}
		+
		\frac{\beta}{y}
		-
		\frac{\beta}{x}.
	\end{align*}
	This is a monotone decreasing function and hence the maximum of 
	$ \ell \cdot (x - y) - D_{\phi}(y, x)$
	is attained at $y^* \in \re_{> 0}$ such that $g(y^*) = 0$.
	As we have
	\begin{align*}
		g\left(
			\frac{\beta}{\beta + \ell x} x
		\right)
		\le
		- \ell
		+
		\beta
		\cdot
		\frac{\beta + \ell x}{\beta x} 
		-
		\frac{\beta}{x}
		=
		0
	\end{align*}
	and
	\begin{align*}
		g\left(
			\left(
				\frac{1}{\sqrt{x}}
				+
				\eta \ell
			\right)^{-2}
		\right)
		\le
		- \ell
		+
		\frac{1}{\eta}
		\left(
			\frac{1}{\sqrt{x}}
			+
			\eta \ell
		\right)
		-
		\frac{1}{\eta \sqrt{x}}
		=
		0,
	\end{align*}
	we have
	\begin{align*}
		y^*
		\le
		\min \left\{
			\frac{\beta}{\beta + \ell x} x,
			\left(
				\frac{1}{\sqrt{x}}
				+
				\eta \ell
			\right)^{-2}
		\right\}
		&
		\le
		\min \left\{
			2 x,
			\left(
				\frac{1}{\sqrt{x}}
				+
				\eta \ell
			\right)^{-2}
		\right\}
		\\
		&
		=
		x
		\min \left\{
			2,
			\left(
				1
				+
				\eta \ell \sqrt{x}
			\right)^{-2}
		\right\},
	\end{align*}
	where the last inequality follows from the assumption of $\ell x \ge - \frac{\beta}{2}$.
	We hence have
	\begin{align}
		\nonumber
		\sup_{ y \in (0, 1]}
		\left\{
		\ell \cdot (x - y) - D_{\phi}(y, x)
		\right\}
		&
		\le
		- \ell ( y^* - x)
		-
		D_{\phi}(y^*, x)
		\le
		- \ell ( y^* - x)
		\\
		&
		\le
		- \ell
		x
		\cdot
		\min \left\{
			1,
			\left(
				1
				+
				\eta \ell \sqrt{x}
			\right)^{-2}
			-
			1
		\right\}.
		\label{eq:stabTsallisLB0}
	\end{align}
	If $ 0 < - \eta \ell \sqrt{x} \le 1/2$,
	we then have
	$
	\left(
		1
		+
		\eta \ell \sqrt{x}
	\right)^{-2}
	-
	1
	\le
	- 6 \eta \ell \sqrt{x}
	$,
	which implies that
	the value of \eqref{eq:stabTsallisLB0}
	is at most
	$- \ell x \cdot (-6\eta \ell \sqrt{x})
	=
	6
	\eta x^{3/2} \ell^2
	$.
	If $-\eta \ell \sqrt{x} > 1/2$,
	we then have
	the value of \eqref{eq:stabTsallisLB0}
	is at most
	$
	- \ell x
	<
	- \ell x
	\cdot
	(-2 \eta \ell \sqrt{x})
	=
	2 \eta x^{3/2} \ell^2
	$.
	This completes the proof.
\end{proof}

\begin{lemma}
	\label{lem:stabLB}
	Let $\phi: \re_{>0} \rightarrow \re$ be defined as $\phi (x) = - \log (x)$
	and $D_{\phi}(y, x)$ be the Bregman divengence associated with $\phi$,
	i.e.,
	\begin{align*}
		D_{\phi}(y, x)
		=
		- \log y + \log x
		+ \frac{1}{x} (y-x)
		=
		-
		\log \frac{y}{x}
		+
		\frac{y}{x}
		-
		1.
	\end{align*}
	Then,
	for any $x \in (0,1)$, $\ell \in \re$ and $\eta > 0$ such that $\eta x \ell  > -1$,
	we have
	\begin{align*}
		\sup_{y \in [0,1]}
		\left\{
		\ell \cdot (x - y) - \frac{1}{\eta} D_{\phi}(y, x)
		\right\}
		\le
		\frac{1}{\eta}
		\left(
			- \log \left(
				1 + \eta x \ell
			\right)
			+
			\eta x \ell
		\right).
	\end{align*}
	Consequently,
	if $\eta \ell x \ge - \frac{1}{2}$,
	we have
	\begin{align*}
		\sup_{y \in [0,1]}
		\left\{
		\ell \cdot (x - y) - \frac{1}{\eta} D_{\phi}(y, x)
		\right\}
		\le
		\eta x^2 \ell^2
		\left(
			\frac{1}{2}
			+
			\mathbb{I}[\ell < 0]
			\cdot
			\eta x |\ell|
		\right)
		\le
		\eta x^2 \ell^2.
	\end{align*}
\end{lemma}
\begin{proof}
	We have
	\begin{align*}
		&
		\ell \cdot (x - y) - \frac{1}{\eta} D_{\phi}(y, x)
		=
		\frac{1}{\eta}
		\left(
		\log \frac{y}{x}
		+
		\left(
		\frac{1}{x} 
		+
		\eta \ell
		\right)
		(x-y)
		\right).
	\end{align*}
	For fixed $x$, $\ell$ and $\eta$,
	this value is maximized when $\frac{1}{y} = \frac{1}{x} + \eta \ell$.
	We then have
	\begin{align*}
		\ell \cdot (x - y) - \frac{1}{\eta} D_{\phi}(y, x)
		=
		\frac{1}{\eta}
		\left(
			- \log \frac{x}{y}
			+ \frac{x}{y}
			- 1
		\right)
		=
		\frac{1}{\eta}
		\left(
			- \log \left(
				1 + \eta x \ell
			\right)
			+
			\eta x \ell
		\right).
	\end{align*}
	As we have
	$- \log (1+a) + a \le \frac{1}{2}a^2 + \mathbb{I}[a<0] \cdot |a|^3 $
	for $a \ge -1/2$,
	we have
	\begin{align*}
		\ell \cdot (x - y) - \frac{1}{\eta} D_{\phi}(y, x)
		&
		\le
		\frac{1}{\eta}
		\left(
			- \log \left(
				1 + \eta x \ell
			\right)
			+
			\eta x \ell
		\right)
		\\
		&
		\le
		\frac{1}{\eta}
		\left(
			\frac{1}{2}
			(\eta x \ell)^2
			+
			\mathbb{I}[\ell < 0]
			\cdot
			|\eta x \ell|^3
		\right)
		\\
		&
		=
		\eta x^2
		\left(
			\frac{1}{2}
			\ell^2
			+
			\mathbb{I}[\ell < 0]
			\cdot
			\eta x |\ell|^3
		\right)
		\\
		&
		\le
		\eta x^2 \ell^2.
	\end{align*}
\end{proof}

\begin{proof}[Proof of Lemma \ref{lem:PDL}]
	We can show this by backward induction in layers.
	For $s = s_{L}$,
	\eqref{eq:PDL} is clear as both sides are equal to $0$.
	For $s \in S_{k}$ with $k < L$,
	\begin{align*}
		Q^{\pi'}(s,a;\bar{\ell})
		&
		=
		\bar{\ell}(s,a)
		+
		\sum_{s' \in S_{k+1}} P(s'|s,a) V^{\pi'}(s';\bar{\ell})
		\\
		&
		= Q^{\pi}(s,a;\ell) - V^{\pi}(s;\ell)
		+
		\sum_{s' \in S_{k+1}} P(s'|s,a) 
		\left(
		V^{\pi'}(s'; \ell) - V^{\pi}(s'; \ell)
		\right)
		\\
		&
		= 
		\ell(s,a)
		+
		\sum_{s' \in S_{k+1}} P(s'|s,a) V^{\pi}(s'; \ell)
		- 
		V^{\pi}(s;\ell)
		+
		\sum_{s' \in S_{k+1}} P(s'|s,a) 
		\left(
		V^{\pi'}(s'; \ell) - V^{\pi}(s'; \ell)
		\right)
		\\
		&
		= 
		\ell(s,a)
		- 
		V^{\pi}(s;\ell)
		+
		\sum_{s' \in S_{k+1}} P(s'|s,a) 
		V^{\pi'}(s'; \ell) 
		\\
		&
		=
		Q^{\pi'}(s,a;\ell)
		- 
		V^{\pi}(s;\ell),
	\end{align*}
	where the second equality follows from the induction hypothesis and the definition of $\bar{\ell}$.
	We hence have
	\begin{align*}
		V^{\pi'}(s ; \bar{\ell})
		=
		\sum_{a \in A} \pi'(a|s) Q^{\pi'}(s,a ; \bar{\ell} )
		=
		\sum_{a \in A} \pi'(a|s) 
		\left(
		Q^{\pi'}(s,a;\ell)
		- 
		V^{\pi}(s;\ell)
		\right)
		=
		V^{\pi'}(s;\ell)
		- 
		V^{\pi}(s;\ell).
	\end{align*}
\end{proof}
\section{Online shortest path problem with bandit feedback}
In this section, we analyze our algorithm for online shortest path problem. Specifically, we prove Theorems~\ref{thm:Tsallis-SP} and~\ref{thm:log-barrier-SP-appendix}, which together directly imply Corollary~\ref{cor:bobw-shortest-path} in the main body. In addition, we prove our lower bound result in Theorem~\ref{thm:LBSP}.
\subsection{Notation and problem setup}
\begin{itemize}
	\item $G = (V \cup \{ s, g \}, E )$: a directed acyclic graph.
		\item $s$: Source node.
		\item $g$: Sink node.
		\item $V$: Set of vertices that are neither sources nor sinks.
		\item $E \subseteq (V \cup \{ s \}) \times (V \cup \{ g \})$: set of directed edges.
		\item $e_-, e_+ \in V \cup \{ s, g \}$: initial and terminal vertices of an edge $e \in E$,
		i.e.,
		$e = (e_-, e_+)$.
		\item $\ies v, \oes v \subseteq E$: sets of incoming and outgoing edges of a vertex $v \in V \cup \{ s, g \}$,
		i.e.,
		$\ies v = \{ e \in E \mid e_+ = v \}$,
		$\oes v = \{ e \in E \mid e_- = v \}$.
		\item $n = |V|$.
		\item $m = |E|$.
		\item $\cP \subseteq \{ 0, 1 \}^E$: set of (vector representations of) $s$-$g$ paths.
		\item $\cQ = \conv(\cP) \subseteq [0,1]^E$: set of $s$-$g$ flows of value $1$,
		equivalently convex hull of $\cP$.
		Note $\cQ = \{ q \in [0,1]^E \mid 
		\sum_{e \in \oes s} q(e) = \sum_{e \in \ies g} q(e) = 1 ,
		\sum_{e \in \oes v} q(e) = \sum_{e \in \ies v} q(e) (\forall v \in V) \}$.
		\item $L = \max_{p \in \cP} \{ \| p \|_1 \}  \le |V|+1$: maximum length of $s$-$g$ paths.
		\item Without loss of generality,
		we assume that every vertex $v \in V$ admits a path from $s$ to $g$ passing through $v$.
\end{itemize}
In each round $t \in [T]$,
an environment chooses $\ell_t \in \re_{\ge 0}^E$ and then the player chooses
$p_t \in P$,
after which the player observes a feedback $c_t \in [0, 1]$ such that
$\E [c_t | \ell_t, p_t] = \linner \ell_t, p_t \rinner$.
We here assume that $\ell_t$ satisfies $ \linner \ell_t, p \rinner \le 1 $ for all $p \in P$. 
The performance of the player is evaluated in terms of regret defined as:
\begin{align*}
	\Reg_T ( p^* ) = \E \left[ \sum_{t=1}^T \linner \ell_t, p_t - p^* \rinner \right],
	\quad
	\Reg_T = \max_{p^* \in P} \Reg_T(p^*),
\end{align*}
where the expectation is taken over all randomness arising from the environment and the algorithm.


\subsection{Algorithm}
The algorithm updates $q_t \in \cQ$ by an FTRL approach and picks $p_t \in \cP$ so that
$\E[ p_t | q_t ] = q_t$,
following the technique of \citep{maiti2025efficient}.

In the following,
for $q \in \cQ$ and $v \in V$,
we denote 
\begin{align}
	\label{eq:defqv}
	q(v) 
	= \sum_{e \in \ies v} q(e)
	= \sum_{e \in \oes v} q(e)
\end{align}
for the notational simplicity.

Let $p_v \in P$ be an $s$-$g$ path that passes through $v$,
i.e.,
$p_v(v) = 1$.
Define $q_0 \in Q$ by
\begin{align}
	\label{eq:defq0}
	q_0 = \frac{1}{|V|} \sum_{v \in V} p_v.
\end{align}
We then have $q_0(v) \ge 1/|V| = 1/n$ for any $v \in V$.

Define ${q}_t$ by
\begin{align}
	\label{eq:defFTRL}
	{q}_t \in
	\argmin_{q \in Q}
	\left\{
		\linner
		\sum_{\tau=1}^{t-1} \widehat{\ell}_\tau,
		q
		\rinner
		+
		\psi_t(q)
	\right\},
\end{align}
where $\widehat{\ell}_\tau$ is an estimator for $\ell_\tau$ defined later.
The regularizer $\psi_t$ is given by
\begin{align}
  \psi_t(q)
  &
  =
  -
  \frac{2}{\eta_t}
  \sum_{e \in E}
  \sqrt{q(e)}
  -
  \sum_{e \in E}
  \beta
  \ln q(e)
  \quad
  \mbox{with}
  \quad
  \eta_t = \frac{1}{\sqrt{t}},
  ~
  \beta = 2,
  \quad
  \mbox{or}
  \label{eq:defpsi-SP-Tsallis-app}
  \\
  \psi_t(q)
  &
  =
  -
  \sum_{e \in E}
  \frac{1}{\eta_t(e)}
  \ln q(e)
  \quad
  \mbox{with}
  \quad
  \eta_t(e) 
  = \big(4 + \frac{1}{\ln T} \sum_{\tau = 1}^{t-1} 
  \rho_{\tau}(e)
  \big)^{-\frac{1}
  {2}},
  \label{eq:defpsi-SP-LB-app}
\end{align}
where $\rho_{\tau} \in [0,1]$ will be defined later.

Based on $q_t \in \cQ$,
we pick $p_t \in P$ in the same way as in \citep{maiti2025efficient}:
\begin{itemize}
	\item Initialize $p \in \{ 0, 1 \}^E$ by $p(e) = 0$ for all $e \in E$ and set $v \leftarrow s$.
	\item While $v \neq g$:
	\begin{itemize}
		\item Pick $e \in \oes v$ with probability $q_t(e)/q_t(v)$.
		\item Set $p(e) \leftarrow 1$ and transition to the next node $e_+$,
		i.e.,
		$v \leftarrow e_+$.
	\end{itemize}
\end{itemize}
We then have $\E[p_t | q_t]= q_t$.

After outputting $p_t$,
we get feedback $c_t \in [0,1]$ such that
$\E[c_t | p_t, \ell_t]= \linner \ell_t, p_t \rinner$.
Based on this,
we define $\widehat{\ell}_t \in \re^E$ by
\begin{align}
	\label{eq:defellhat}
	\widehat{\ell}_t(e) = 
	c_t \cdot \left( \frac{p_t(e)}{q_t(e)} - \frac{p_t(e_-)}{q_t{(e_-)}} \right).
\end{align}
Note that the notation of \eqref{eq:defqv} applies to $p \in P \subseteq Q$ as well,
and that $p_t(v) = 1$ if and only if the path passes through the node $v$.
Then,
it holds for any $q \in Q$ that
\begin{align}
	\nonumber
	\linner \widehat{\ell}_t, q \rinner
	&
	=
	c_t \cdot \sum_{e \in E} \left( \frac{p_t(e)}{q_t(e)} - \frac{p_t(e_-)}{q_t{(e_-)}} \right) q(e) 
	\\
	\nonumber
	&
	=
	c_t \cdot \left( \sum_{e \in E} \frac{p_t(e)}{q_t(e)}q(e) - 
	\sum_{v \in V \cup \{ s \}}
	\sum_{e \in \oes v}
	\frac{p_t(e_-)}{q_t{(e_-)}}  q(e) \right) 
	&
	(E = \bigcup_{v \in V \cup \{ s \}} \oes v)
	\\
	\nonumber
	&
	=
	c_t \cdot \left( \sum_{e \in E} \frac{p_t(e)}{q_t(e)}q(e) - 
	\sum_{v \in V \cup \{ s \}}
	\frac{p_t(v)}{q_t{(v)}}  q(v) \right) 
	&
	(e \in \oes v \iff e_- = v, \mbox{\eqref{eq:defqv}} )
	\\
	&
	=
	c_t \cdot \left( \sum_{e \in E} \frac{p_t(e)}{q_t(e)}q(e) - 
	\sum_{v \in V }
	\frac{p_t(v)}{q_t{(v)}}  q(v) 
	-1
	\right) .
	&
	( p_t(s)=q_t(s)=q(s)=1 )
	\label{eq:ellhatev}
\end{align}
We note that,
an alternative definition of $\widehat{\ell}_t$ given as
\begin{align}
	\label{eq:defellhatp}
	\widehat{\ell}'_t(e) = 
	c_t \cdot \left( \frac{p_t(e)}{q_t(e)} - \frac{p_t(e_+)}{q_t{(e_+)}} \right)
\end{align}
also satisfies \eqref{eq:ellhatev} similarly,
and hence we have
\begin{align*}
	\linner
	\widehat{\ell}_t, q
	\rinner
	=
	\linner
	\widehat{\ell}'_t, q
	\rinner
\end{align*}
for any $q \in Q$.
Therefore, using $\widehat{\ell}'_t$ in \eqref{eq:defellhatp} instead of $\widehat{\ell}_t$ in \eqref{eq:defellhat} does not change the behavior of the algorithm.

We now state the following that can be proved in a similar way as in \citep{maiti2025efficient}:
\begin{lemma}
	\label{lem:unbiased}
	For any $q, q' \in Q$,
	we have
	\begin{align*}
		\E \left[ \linner \widehat{\ell}_t, q - q' \rinner | q_t, \ell_t \right]
		=
		\linner \ell_t, q - q' \rinner,
	\end{align*}
	where the expectation is taken w.r.t.~$p_t$.
\end{lemma}
Note that the effect of the path-length does not appear here.
That is, we do not need to assume that the path lengths are the same.

\subsection{Regret analysis}
\begin{definition}[consistent policy]
	\label{def:consistent}
	Define $\Pi = \{ \pi: V \cup \{ s \} \rightarrow E \mid \pi(v) \in \oes v ~(\forall v \in V \cup \{ s \} )  \}$.
	Let $p^* \in P$ be an arbitrary $s$-$g$ path.
	Let $E^* \subseteq E$ and $V^* \subseteq V$ denote the sets of edges and nodes included in $p^*$,
	i.e.,
	$E^* = \{ e \in E \mid p^*(e)=1 \} $ and
	$V^* = \{ v \in V \mid p^*(v)=1 \} $.
	We say $\pi^* \in \Pi$ is \textit{consistent} with $p^* \in P$ if and only if
	$\pi^*(v) \in E^*$ for all $v \in V^* \cup \{ s \}$.
	We denotes $E' = E \setminus \mathrm{Im}(\pi^*)$.
\end{definition}
\begin{definition}[self-bounding regime for online shortest path]
	\label{def:SBR-SP}
	Let $p^* \in P$ be an arbitrary $s$-$g$ path and suppose that $\pi^* \in \Pi$ is consistent with $p^*$.
	Suppose that $\Delta \in [0,1]^E$ satisfies $\Delta(e) > 0$ for all $e \in E' = E \setminus \mathrm{Im}(\pi^*)$.
	The environment is in a $(p^*, \pi^*, \Delta, C)$-self-bounding regime if it holds that
	\begin{align}
		\label{eq:self-bounding}
		\Reg_T(p^*) 
		\ge
		\E\left[
			\sum_{t=1}^T
			\sum_{v \in V \cup \{ s \}}
			\sum_{e \in \oes v \setminus \{ \pi^*(v) \}} \Delta(e) p_t(e)
		\right]
		-C.
	\end{align}
\end{definition}
\begin{remark}
	An example of $\Delta \in [0,1]^E$ can be constructed as follows:
	Assume that $\ell_t$ follows an identical distribution independently for all $t \in [T]$
	and denote $\ell^* = \E[\ell_t]$.
	For each $u, v \in V \cup \{ s, g \}$,
	let $\dist (u, v)$ denote the length of $u$-$v$ shortest path w.r.t.~the weight $\ell^*$.
	(Set $\dist (u,v)= + \infty$ if there is no $u$-$v$ path.)
	For each $e \in E$,
	define $\Delta \in [0,1]^E $ by
	\begin{align}
		\label{eq:defDelta1}
		\Delta(e) = \ell^*(e) + \dist(e_{+}, g) - \dist(e_{-}, g).
	\end{align}
	Then,
	if $p^* \in P$ is a shortest $s$-$g$ path,
	then \eqref{eq:self-bounding} holds.
	In fact,
	for any $s$-$g$ path $p \in P$ expressed as a sequence of $(s=v_0,e_1,v_1,e_2, v_2, \ldots, v_{h-1}, e_{h}, g=v_h)$,
	we have
	\begin{align*}
		\linner \Delta, p \rinner
		=
		\sum_{j=1}^h
		\Delta(e_j)
		=
		\sum_{j=1}^h
		\left(
		\ell^*(e_j) + 
		\dist(v_{j}, g)
		-
		\dist(v_{j-1}, g)
		\right)
		\\
		=
		\sum_{j=1}^h
		\ell^*(e_j) 
		+
		\dist(v_h, g)
		-
		\dist(v_0, g)
		=
		\linner \ell^*, p \rinner
		-
		\linner \ell^*, p^* \rinner,
	\end{align*}
	which implies
	$\Reg_T(p^*) = \E \left[ \sum_{t=1}^T \linner \Delta, p_t \rinner \right]$.
	Suppose $\pi^*$ is chosen so that
	$\pi^* (v) \in \argmin_{e \in \oes v} \{ \Delta(e) \}$.
	Then,
	$\Delta(e)$ for all $e \in E' = E \setminus \mathrm{Im}(\pi^*)$ if and only if
	the shortest $v$-$g$ path is unique for all $v \in V \cup \{ s \}$.
\end{remark}
\begin{remark}
	An issue of the definition of $\Delta$ in \eqref{eq:defDelta1} is the requirement for a strong assumption that the $v$-$g$ shortest path is unique for all $v \in V \cup \{ s \}$ to ensure that $\Delta(e) > 0$ for all $e \in E'$.
	We can relax this assumption by some alternative definitions of $\Delta$.
	When using the following definition, it suffices to assume that the $s$-$g$ shortest path is unique:
	Define $L' = \max_{ p \in \cP } \{ \sum_{e \in E'} p(e) \} \le L$.
	For $e \in E'$ and $k \in [L']$,
	define $\tilde{\cP}(e, k) \subseteq P$ 
	and
	$\tilde{\Delta}(e)$
	by
	\begin{align}
		\nonumber
		\tilde{\cP}(e, k) 
		&
		= \left\{ 
			p \in \cP \mid
			p(e) = 1,
			\sum_{e' \in E'} p(e') = k
		\right\},
		\\
		\tilde{\Delta}(e)
		&
		=
		\min_{k \in [L']} 
		\left\{
		\frac{1}{k}
		\inf_{p \in \tilde{\cP}(e,k)}
			\left\{
				\linner \ell^*, p - p^*  \rinner
			\right\}
		\right\}
		=
		\min_{k \in [L']} 
		\left\{
		\frac{1}{k}
		\inf_{p \in \tilde{P}(e,k)}
			\left\{
				\linner \Delta, p  \rinner
			\right\}
		\right\} .
		\label{eq:defDelta2}
	\end{align}
	We then have $\sum_{e \in E'} \tilde{\Delta}(e) p(e) \le \sum_{e \in E'} \Delta(e) p(e)$ for all $p \in P$,
	which implies that the environment is in $(p^*, \pi^*, \tilde{\Delta})$ as well.
	Further,
	as we have $p^* \notin \tilde{P}(e, k)$ for any $e \in E'$ and $k$,
	we have
	\begin{align*}
		\tilde{\Delta}(e)
		\ge
		\frac{1}{L'}
		\min_{p \in P \setminus \{ p^* \}}
		\left\{
			\linner \ell^*, p - p^* \rinner
		\right\}
		=:
		\frac{1}{L'}
		\Delta_{\min}
	\end{align*}
	for any $e \in E'$.
	Hence,
	this value is positive as long as the $s$-$g$ shortest path is unique.
	We also have
	$
		\tilde{\Delta}(e) \ge
		\min_{e' \in E'} \Delta(e)
	$ for any $e \in E'$.
\end{remark}

Let $q_0 \in \cQ$ be such that $q_0(e) \ge 1/m$.
For any $p^* \in \cQ$ and $\epsilon \in [0,1]$,
set $q^*$ by
\begin{align}
	\label{eq:defqstar-SP}
	q^* = \left(
		1 - \epsilon
	\right)p^* 
	+
	\epsilon
	q_0.
\end{align}
Using Lemma~\ref{lem:unbiased} and standard analysis for FTRL (see, e.g., Exercise 28.12 of \citep{lattimore2020bandit}),
we obtain:
\begin{align}
	\nonumber
	&
	\Reg_T(p^*)
	=
	\E \left[
		\sum_{t=1}^T \linner \ell_t, p_t - p^* \rinner
	\right]
	=
	\E \left[
		\sum_{t=1}^T \linner \widehat{\ell}_t, p_t - p^* \rinner
	\right]
	=
	\E \left[
		\sum_{t=1}^T \linner \widehat{\ell}_t, q_t - p^* \rinner
	\right]
	\\
	\nonumber
	&
	=
	\E \left[
		\sum_{t=1}^T \linner \widehat{\ell}_t, {q}_t - q^* \rinner
	\right]
	+
	\epsilon
	\E \left[
		\sum_{t=1}^T \linner \widehat{\ell}_t, q_0 - p^*  \rinner
	\right]
	\le
	\E \left[
		\sum_{t=1}^T \linner \widehat{\ell}_t, {q}_t - q^* \rinner
	\right]
	+
	\epsilon T
	\\
	&
	\le
	\epsilon
	T 
	+
	\sum_{t=1}^T 
		\E \left[
				\underset{=: \stab_t}{
				\underline{
		\linner \widehat{\ell}_t, {q}_t - {q}_{t+1} \rinner
		-
		D_t({q}_{t+1}, {q}_t )
				}
			}
		+
			\underset{=: \pena_t}{
				\underline{
		\left(
			\psi_t(q^*)
			-
			\psi_{t-1}(q^*)
			-
			\psi_t({q}_{t})
			+
			\psi_{t-1}({q}_{t})
		\right)
				}
			}
	\right],
	\label{eq:spdecomposition}
\end{align}
where $D_t(\cdot, \cdot)$ represents the Bregman divergence associated with $\psi_t$ defined by \eqref{eq:defpsi-SP-Tsallis-app} or \eqref{eq:defpsi-SP-LB-app},
and we set $\psi_t(q) = 0$ for $t=0$ as an exception.
\subsubsection{Analysis for Tsallis-entropy case}
\begin{theorem}[First part of Theorem~\ref{thm:OSP}]
	\label{thm:Tsallis-SP}
	Let $p^* \in P$ be an arbitrary $s$-$g$ path and suppose that $\pi^* \in \Pi$ is consistent with $p^*$.
	Then the proposed algorithm with the Tsallis-entropy regularizer \eqref{eq:defpsi-SP-Tsallis-app} achieves:
	\begin{align}
		\label{eq:thmTsallis1}
		&
		\Reg_T(p^*)
		\lesssim
		\sum_{t=1}^T
		\frac{1}{\sqrt{t}}
		\E
		\left[
		\sum_{v \in V \cup \{ s \}}
		\sum_{e \in \oes v \setminus \{ \pi^*(v) \}}
		\sqrt{q_t(e)}
		+
		\sum_{v \in V \setminus V^*}
		\sqrt{ q_t(\pi^*(v)) }
		\right]
		+
		m\log T
		\\
		&
		\le
		\sum_{t=1}^T
		\frac{1}{\sqrt{t}}
		\E
		\left[
		\sum_{v \in V \cup \{ s \}}
		\sum_{e \in \oes v \setminus \{ \pi^*(v) \}}
		\sqrt{q_t(e)}
		+
		n
		\sqrt{
		\sum_{v \in V^* \cup \{ s \}}
		\sum_{e \in \oes v \setminus \{ \pi^*(v) \}} q_t(e)
		}
		\right]
		+
		m\log T
		\label{eq:thmTsallis2}
	\end{align}
\end{theorem}
\begin{corollary}[First part of Corollary~\ref{cor:bobw-shortest-path}]
	\label{cor:Tsallis-SP}
	In the adversarial regime,
	the proposed algorithm with the Tsallis-entropy regularizer \eqref{eq:defpsi-SP-Tsallis-app} achieves
	$\Reg_T = O( \sqrt{mLT} + m \log T )$.
	Further,
	if the environment is in a $(p^*, \pi^*, \Delta, C)$-self-bounding regime given in Definition~\ref{def:SBR-SP},
	we then have
	$\Reg_T(p^*) \lesssim U + \sqrt{UC} + m \log T$,
  where we define
  \begin{align}
    U 
    &= 
    \sum_{v \in V \cup \{ s \}} \sum_{e \in \oes v \setminus \{ \pi^*(v) \}} \frac{\log T}{\Delta(e)} 
    +
    \frac{n^2 \log T}{\Delta^*} ,
    \\
	\Delta^* &=
	\min \left\{ \Delta(e) \mid \exists v \in V^* \cup \{ s \}, e \in \oes v \setminus \{ \pi^*(v) \} \right\}.
  \end{align}
\end{corollary}

In the following,
we provide a proof of Theorem~\ref{thm:Tsallis-SP}.
\begin{lemma}
	\label{lem:stabTsallis-SP}
	When we use the Tsallis-entropy regularizer \eqref{eq:defpsi-SP-Tsallis-app},
	stability terms are bounded as
	\begin{align}
		\label{eq:stabTsallis-SP}
		\E [ \stab_t ]
		&
		\lesssim
		\E\left[
		\eta_t
		\sum_{v \in V \cup \{ s \}}
		\left(
		\sum_{e \in \oes v}
		\sqrt{q_t(e)}
		-
		\sqrt{q_t(v)}
		\right)
		\right]
		\lesssim
		\E\left[
		\eta_t
		\sum_{v \in V \cup \{ s \}}
		\sum_{e \in \oes v \setminus \{ \pi^*(e) \}}
		\sqrt{q_t(e)}
		\right].
	\end{align}
\end{lemma}
\begin{proof}
To bound the stability term,
we can apply Lemma~\ref{lem:stabTsallisLB}
with $\ell = \widehat{\ell}_t(e)$,
$x = {q}_t(e)$,
$\eta = \eta_t$,
and $\beta = 2$.
In fact,
we can verify that $- \widehat{\ell}_t(e) q_t(e) \le \frac{q_t(e)}{q_t(e_-)} \le 1 \le \frac{\beta}{2}$.
Hence,
from Lemma~\ref{lem:stabTsallisLB},
we have
\begin{align*}
	&
	\linner \widehat{\ell}_t, {q}_t - {q}_{t+1} \rinner
	-
	\frac{1}{\eta_t}D({q}_{t+1}, {q}_t )
	\lesssim
	\eta_t
	\sum_{e \in E}
	\left(
		{q}_t(e)
	\right)^{3/2}
	\left(
	\widehat{\ell}_t(e)
	\right)^2
	\\
	&
	\lesssim
	\eta_t
	\sum_{e \in E}
	\left(
		{q}_t(e)
	\right)^{3/2}
	\left(
		\frac{p_t(e)}{q_t(e)}
		-
		\frac{p_t(e_{-})}{q_t(e_{-})}
	\right)^2
	\\
	&
	=
	\eta_t
	\sum_{v \in V \cup \{ s \}}
	\sum_{e \in \oes v}
	\left(
		{q}_t(e)
	\right)^{3/2}
	\left(
		\frac{p_t(e)}{q_t(e)}
		-
		\frac{p_t(e_{-})}{q_t(e_{-})}
	\right)^2
	&
	\mbox{($E = \bigcup_{v \in V \cup \{ s\}} \oes v$)}
	\\
	&
	=
	\eta_t
	\sum_{v \in V \cup \{ s \}}
	\sum_{e \in \oes v}
	\left(
		{q}_t(e)
	\right)^{3/2}
	\left(
		\frac{p_t(e)}{q_t(e)}
		-
		\frac{p_t(v)}{q_t(v)}
	\right)^2
	&
	\mbox{($e \in \oes v \iff e_{-} = v$)}
	\\
	&
	=
	\eta_t
	\sum_{v \in V \cup \{ s \}}
	\frac{p_t(v)}{\sqrt{q_t(v)}}
	\sum_{e \in \oes v}
	\left(
		\frac{{q}_t(e)}{q_t(v)}
	\right)^{3/2}
	\left(
		\frac{p_t(e)q_t(v)}{q_t(e)}
		-
		1
	\right)^2,
	&
\end{align*}
where we used the fact that $p_t(v) = 0$ implies $p_t(e)=0$ for $e \in \oes v$.
Taking the conditional expectation w.r.t.~$p_t$ given $q_t$,
we obtain:
\begin{align*}
	&
	\E \left[
	\frac{p_t(v)}{\sqrt{q_t(v)}}
	\sum_{e \in \oes v}
	\left(
		\frac{{q}_t(e)}{q_t(v)}
	\right)^{3/2}
	\left(
		\frac{p_t(e)q_t(v)}{q_t(e)}
		-
		1
	\right)^2
	\right]
	\\
	&
	=
	\sqrt{q_t(v)}
	\cdot
	\E \left[
	\sum_{e \in \oes v}
	\left(
		\frac{{q}_t(e)}{q_t(v)}
	\right)^{3/2}
	\left(
		\frac{p_t(e)q_t(v)}{q_t(e)}
		-
		1
	\right)^2
	|
	p_t(v)=1
	\right]
	\\
	&
	=
	\sqrt{q_t(v)}
	\sum_{e \in \oes v}
	\left(
		\frac{{q}_t(e)}{q_t(v)}
	\right)^{3/2}
	\left(
		\frac{q_t(e)}{q_t(v)}
		\cdot
		\left(
			\frac{q_t(v)}{q_t(e)}
			-
			1
		\right)^2
		+
		\left(
			1
			-
			\frac{q_t(e)}{q_t(v)}
		\right)
		\cdot
		1
	\right)
	\\
	&
	=
	\sqrt{q_t(v)}
	\sum_{e \in \oes v}
	\left(
		\frac{{q}_t(e)}{q_t(v)}
	\right)^{1/2}
	\left(
		1
		-
		\frac{q_t(e)}{q_t(v)}
	\right)
	=
	\sum_{e \in \oes v}
	\sqrt{q_t(e)}
	\left(
		1
		-
		\frac{q_t(e)}{q_t(v)}
	\right).
\end{align*}
Further,
as we have $(1-x) \le 2(1 - \sqrt{x})$ for $x \in [0, 1]$,
we have
\begin{align*}
	&
	\sum_{e \in \oes v}
	\sqrt{q_t(e)}
	\left(
		1
		-
		\frac{q_t(e)}{q_t(v)}
	\right)
	\le
	2
	\sum_{e \in \oes v}
	\sqrt{q_t(e)}
	\left(
		1
		-
		\sqrt{
		\frac{q_t(e)}{q_t(v)}
		}
	\right)
	=
	2
	\left(
	\sum_{e \in \oes v}
	\sqrt{q_t(e)}
	-
	\sum_{e \in \oes v} \frac{q_t(e)}{\sqrt{q_t(v)}}
	\right)
	\\
	&
	=
	2
	\left(
	\sum_{e \in \oes v}
	\sqrt{q_t(e)}
	-
	\frac{q_t(v)}{\sqrt{q_t(v)}}
	\right)
	=
	2
	\left(
	\sum_{e \in \oes v}
	\sqrt{q_t(e)}
	-
	\sqrt{q_t(v)}
	\right).
\end{align*}
By combining the above inequalities,
\begin{align*}
	\E \left[
		\linner \widehat{\ell}_t, \tilde{q}_t - \tilde{q}_{t+1} \rinner
		-
		\frac{1}{\eta_t}D(\tilde{q}_{t+1}, \tilde{q}_t )
		|
		q_t
	\right]
	\lesssim
	\eta_t
	\sum_{v \in V \cup \{ s \}}
	\left(
	\sum_{e \in \oes v}
	\sqrt{q_t(e)}
	-
	\sqrt{q_t(v)}
	\right).
\end{align*}
The second inequality in \eqref{eq:stabTsallis-SP} follows from
the fact that
$q_t(e) \le q_t(v)$ for any $e \in \oes v$.
\end{proof}

\begin{lemma}
	\label{lem:penaTsallis-SP}
	Let $p^*$ be a path consisting of $V^* \in V$ and $E^* \subseteq E$.
	Suppose that $q^*$ is given by \eqref{eq:defqstar-SP}.
	For $t \ge 2$,
	if $q^*$ is a path consisting of 
	$V^* \in V$ and $E^* \subseteq E$,
	penalty terms are bounded as
	\begin{align}
		\nonumber
		\pena_t
		&
		\lesssim
		\eta_t
		\left(
			\sum_{e \in E} \sqrt{q(e)}
			-
			|E|
			+
			m \epsilon
		\right)
		\le
		\eta_t
		\sum_{e \in E \setminus E^*}
		\sqrt{q_t(e)}
		+
		m \epsilon
		\\
		&
		=
		\eta_t
		\left(
		\sum_{v \in V \cup \{ s \}}
		\sum_{e \in \oes v \setminus \{ \pi^*(v) \}}
		\sqrt{q_t(e)}
		+
		\sum_{v \in V \setminus V^*}
		\sqrt{q_t(\pi^*(v))}
		\right)
		+
		m \epsilon
		\label{eq:penaTsallis-SP}
	\end{align}
	In the case of $t=1$,
	the bound includes an $O(m \log(m/\epsilon))$ term in addition to the above.
\end{lemma}
\begin{proof}
Suppose that $t \ge 2$.
We note it follows from 
$\eta_t = 1/\sqrt{t}$
that
$\frac{1}{\eta_{t-1}} - \frac{1}{\eta_{t}} = \Theta(\eta_t)$.
We hence have
\begin{align*}
	\pena_t
	&
	=
	2
	\left(
		\frac{1}{\eta_{t}}
		-
		\frac{1}{\eta_{t-1}}
	\right)
	\sum_{e \in E}
	\left(
	\sqrt{q_t(e)}
	-
	\sqrt{q^*(e)}
	\right)
	\lesssim
	\eta_t
	\sum_{e \in E}
	\left(
	\sqrt{q_t(e)}
	-
	\sqrt{q^*(e)}
	\right)
	\\
	&
	\lesssim
	\eta_t
	\sum_{e \in E}
	\left(
	\sqrt{q_t(e)}
	-
	\sqrt{p^*(e)}
	+
	\epsilon
	\right)
	=
	\eta_t
	\left(
	\sum_{e \in E}
	\sqrt{q_t(e)}
	-
	|E^*|
	+
	m
	\epsilon
	\right)
	\\
	&
	\le
	\eta_t
	\sum_{e \in E \setminus E^*}
	\sqrt{q_t(e)}
	+
	m \epsilon.
\end{align*}
The equality in \eqref{eq:penaTsallis-SP} follows from
$E \setminus E^* = (\bigcup_{v \in V\cup \{ s \}} (\oes v \setminus \pi^*(v))  ) \cup (V \setminus V^*)$.
In the case of $t=1$,
the penalty term includes an additional term of
\begin{align*}
	\beta 
	\sum_{e \in E}
	\left(
		\ln(q_t(e))
		-
		\ln (q^*(e))
	\right)
	\lesssim
	\sum_{e \in E}
	\ln \frac{1}{q^*(e)}
	\lesssim
	m \log \left( \frac{m}{\epsilon} \right),
\end{align*}
which completes the proof.
\end{proof}
Lemmas~\ref{lem:stabTsallis-SP} and \ref{lem:penaTsallis-SP} combined with \eqref{eq:spdecomposition} immediately lead to \eqref{eq:thmTsallis1}.
Given this,
the next step to \eqref{eq:thmTsallis2} follows from the following lemma:
\begin{lemma}
	\label{lem:qve}
	For any $v \in V \setminus V^*$ and any $q \in \cQ$,
	we have
	\begin{align}
		\label{eq:qve}
		q(v) \le \sum_{v' \in V^* \cup \{ s \}} \sum_{e \in \oes v' \setminus \pi^*(v')} q(e).
	\end{align}
\end{lemma}
\begin{proof}
	As $\cQ$ is a convex hull of $\cP$ and 
	both sides of \eqref{eq:qve} are linear in $q$,
	it suffices to show \eqref{eq:qve} for $q \in \cP $.
	Then,
	for any $q \in \cP$,
	if RHS of \eqref{eq:qve} is positive,
	then it is at least $1$,
	and hence \eqref{eq:qve} holds.
	Therefore,
	it suffices to show that RHS $= 0 \Longrightarrow $ LHS $=0$ for $q \in \cP$.
	Suppose $q$ corresponds to the sequence of $( s=v_0, e_1, v_1, \ldots, e_h, v_h = g )$.
	If RHS $=0$,
	then $q$ consists of $V^* = \{ v^*_1, \ldots, v^*_{h^*} = g \}$ and $E^* = \{ \pi^*(v^*_j ) \}_{j=0,1, \ldots, h^*}$.
	In fact,
	we can show this in induction in $j$,
	under the condition that $\sum_{e \in \oes v' \setminus \{\pi^*(v')\}} q(e) = 0$ for all $v' \in V^* \cup \{ s \}$:
	$q(v^*_j) = 1 \Longrightarrow \sum_{e \in \oes {v^*_j}}  q(e) = 1 
	\Longrightarrow q(\pi^*(v^*_j)) = 1
	\Longrightarrow q(v^*_{j+1}) = 1
	$.
\end{proof}
\begin{proof}[Proof of Theorem~\ref{thm:Tsallis-SP}]
	We choose $\epsilon = \frac{1}{T} \in [0, 1]$ in the definition of $q^*$ in \eqref{eq:defqstar-SP}.
	Then,
	Lemmas~\ref{lem:stabTsallis-SP} and \ref{lem:penaTsallis-SP} combined with \eqref{eq:spdecomposition} lead to \eqref{eq:thmTsallis1}.
	The other inequality \eqref{eq:thmTsallis2} follows from
	\begin{align*}
		\sum_{v \in V \setminus V^*}
		\sqrt{{q}_t(\pi^*(v))}
		&
		\le
		\sum_{v \in V \setminus V^*}
		\sqrt{{q}_t(v)}
		\le
		n
		\sqrt{
			\sum_{v \in V^* \cup \{ s \}}
			\sum_{e \in \oes v \setminus \{ \pi^*(v) \}}
			{q}_t(e)
		},
	\end{align*}
	where we used Lemma~\ref{lem:qve} in the second inequality.
\end{proof}
\begin{proof}[Proof of Corollary~\ref{cor:Tsallis-SP}]
	As we have 
	$\sum_{e \in E} q(e) \le L$
	and
	$\sum_{v \in V} q(v) \le L$
	for any $q \in \cQ$,
	from the Cauchy-Schwarz inequality,
	we have
	\begin{align*}
		\sum_{v \in V \cup \{ s \}}
		\sum_{e \in \oes v \setminus \{ \pi^*(v) \}}
		\sqrt{q_t(e)}
		\le
		\sum_{e \in E} \sqrt{q_t(e)}
		\le
		\sqrt{|E| \sum_{e \in E} q_t(e)}
		\le
		\sqrt{m L}
	\end{align*}
	and
	\begin{align*}
		\sum_{v \in V \setminus V^*}
		\sqrt{q_t(v)}
		\le
		\sum_{v \in V}
		\sqrt{q_t(v)}
		\le
		\sqrt{|V| \sum_{v \in V}{q_t(v)}}
		\le
		\sqrt{n L}.
	\end{align*}
	Combining Theorem~\ref{thm:Tsallis-SP} with this and $\sum_{t=1}^T \frac{1}{\sqrt{t}} \lesssim \sqrt{T}$,
	we obtain $\Reg_T \lesssim \sqrt{mLT} + m\log T$.
	By using the Cauchy-Schwarz inequality,
	we obtain the following:
	\begin{align*}
		&
		\sum_{t=1}^T
		\frac{1}{\sqrt{t}}
		\sum_{v \in V \cup \{ s \}}
		\sum_{e \in \oes v \setminus \{ \pi^*(v) \}}
		\sqrt{q_t(e)}
		\\
		&
		=
		\sum_{t=1}^T
		\frac{1}{\sqrt{t}}
		\sum_{v \in V \cup \{ s \}}
		\sum_{e \in \oes v \setminus \{ \pi^*(v) \}}
		\frac{1}{\sqrt{\Delta(e)}}
		\sqrt{\Delta(e)q_t(e)}
		\\
		&
		\le
		\sum_{t=1}^T
		\frac{1}{\sqrt{t}}
		\sqrt{
		\sum_{v \in V \cup \{ s \}}
		\sum_{e \in \oes v \setminus \{ \pi^*(v) \}}
		\frac{1}{\Delta(e)}
		}
		\sqrt{
		\sum_{v \in V \cup \{ s \}}
		\sum_{e \in \oes v \setminus \{ \pi^*(v) \}}
		\Delta(e)q_t(e)
		}
		\\
		&
		\lesssim
		\sqrt{
		\sum_{v \in V \cup \{ s \}}
		\sum_{e \in \oes v \setminus \{ \pi^*(v) \}}
		\frac{1}{\Delta(e)}
		}
		\sqrt{
			\sum_{t=1}^T
			\sum_{v \in V \cup \{ s \}}
			\sum_{e \in \oes v \setminus \{ \pi^*(v) \}}
			\Delta(e)q_t(e)
			\log T
		}.
	\end{align*}
	Similarly,
	we have have
	\begin{align*}
		&
		\sum_{t=1}^T
		\frac{n}{\sqrt{t}}
		\sqrt{
		\sum_{v \in V \cup \{ s \}}
		\sum_{e \in \oes v \setminus \{ \pi^*(v) \}}
		q_t(e)
		}
		\\
		&
		\lesssim
		n
		\sqrt{\frac{1}{\Delta^*}}
		\sqrt{
			\sum_{t=1}^T
			\Delta^*
			\sum_{v \in V \cup \{ s \}}
			\sum_{e \in \oes v \setminus \{ \pi^*(v) \}}
			q_t(e)
			\log T
		}
		\\
		&
		\le
		n
		\sqrt{\frac{1}{\Delta^*}}
		\sqrt{
			\sum_{t=1}^T
			\sum_{v \in V \cup \{ s \}}
			\sum_{e \in \oes v \setminus \{ \pi^*(v) \}}
			\Delta(e)
			q_t(e)
			\log T
		}.
	\end{align*}
	Hence,
	from Theorem~\ref{thm:Tsallis-SP},
	Jensen's inequality,
	and the assumption of self-bounding regime in Definition~\ref{def:SBR-SP},
	we have
	\begin{align*}
		\Reg_T(p^*)
		&
		\lesssim
		\sqrt{U
		\cdot
		\E\left[
		\sum_{t=1}^T
		\sum_{v \in V \cup \{ s \}}
		\sum_{e \in \oes v \setminus \{ \pi^*(v) \}}
		\Delta(e)
		q_t(e)
		\right]
		}
		+
		m \log T
		\\
		&
		\le
		\sqrt{U
		\cdot
		\left(
			\Reg_T(p^*)
			+
			C
		\right)
		}
		+
		m \log T
		\\
		&
		\lesssim
		\sqrt{U
		\cdot
		\Reg_T(p^*)
		}
		+
		\sqrt{UC}
		+
		m \log T,
	\end{align*}
	which implies
	$\Reg_T(p^*) \lesssim U + \sqrt{UC} + m\log T$.
\end{proof}

\subsubsection{Analysis for log-barrier case}
In the case of the log-barrier regularizer,
we have the following regret bounds:
\begin{theorem}[Second part of Theorem~\ref{thm:OSP}]
	\label{thm:log-barrier-SP-appendix}
	The proposed algorithm with the log-barrier regularizer \eqref{eq:defpsi-SP-LB-app} achieves:
	\begin{align*}
		\Reg_T
		\lesssim
		\E \left[
			\sum_{v \in V \cup \{ s \}}
			\sum_{e \in \oes v}
			\sqrt{
			\sum_{t=1}^T
			c_t^2
			p_t(v)
			\left(
				p_t(e)
				-
				\frac{q_t(e)}{q_t(v)}
			\right)^2
			\log (T)
			}
		\right]
		+
		m \log (T).
	\end{align*}
\end{theorem}
\begin{corollary}[Second part of Corollary~\ref{cor:bobw-shortest-path}]
	\label{cor:log-barrier-SP-appendix}
	In the adversarial regime,
	the proposed algorithm with the log-barrier regularizer \eqref{eq:defpsi-SP-LB-app} achieves
	\begin{align}
		\label{eq:corLB1}
		\Reg_T
		\lesssim
		\E \left[
			\sqrt{
				mL
				\sum_{t=1}^T
				c_t^2
				\log T
			}
		\right]
		\lesssim
		\sqrt{
			mL
			\log T
			\cdot
			\min_{p^* \in \cP}
			\E\left[
				\sum_{t=1}^T \linner \ell_t, p^* \rinner
			\right]
		}
		\le
		\sqrt{mLT}
		.
	\end{align}
	$\Reg_T(p^*) \lesssim U + \sqrt{UC} + m \log T$,
  where we define
  \begin{align}
    U 
    = 
    \sum_{v \in V \cup \{ s \}} \sum_{e \in \oes v \setminus \{ \pi^*(v) \}} \frac{\log T}{\Delta(e)} .
  \end{align}
\end{corollary}
To show these,
we use the following lemmas:
\begin{lemma}
	\label{lem:stabLB-SP}
	When we use the log-barrier regularizer \eqref{eq:defpsi-SP-LB-app},
	we have
	\begin{align*}
		\stab_t
		\lesssim
		c_t^2
		\sum_{v \in V \cup \{ s \}}
		\sum_{e \in \oes v}
		\eta_t(e)
		p_t(v)
		\left(
			p_t(e)
			-
			\frac{q_t(e)}{q_t(v)}
		\right)^2.
	\end{align*}
\end{lemma}
\begin{proof}
	As we have
	$- \eta_t q_t(e) \widehat{\ell}_t(e) \le \eta_t \frac{q_t(e)}{q_t(e_{-})} \le \eta_t \le 1/2$,
	we can use Lemma~\ref{lem:stabLB} to bound the stability terms:
	\begin{align*}
		\stab_t
		&
		\lesssim
		\sum_{e}
		\eta_t(e)
		(q_t(e))^2 
		(\widehat{\ell}_t(e))^2 
		\\
		&
		=
		c_t^2
		\sum_{v \in V \cup \{ s \}}
		\sum_{e \in \oes v}
		\eta_t(e)
		(q_t(e))^2
		\left(
			\frac{p_t(e)}{q_t(e)}
			-
			\frac{p_t(v)}{q_t(v)}
		\right)^2
		\\
		&
		=
		c_t^2
		\sum_{v \in V \cup \{ s \}}
		\sum_{e \in \oes v}
		\eta_t(e)
		\left(
			p_t(e)
			-
			\frac{q_t(e)p_t(v)}{q_t(v)}
		\right)^2
		\\
		&
		=
		c_t^2
		\sum_{v \in V \cup \{ s \}}
		\sum_{e \in \oes v}
		\eta_t(e)
		p_t(v)
		\left(
			p_t(e)
			-
			\frac{q_t(e)}{q_t(v)}
		\right)^2,
	\end{align*}
	where we used the fact that $q(v) = 0 $ implies $q(e)=0$ for $e \in \oes(v)$ and for any $q \in Q$.
\end{proof}
Define $\rho_t(e)$ by
\begin{align*}
	\rho_t(e) = c_t^2 
	p_t(e_{-})
	\left(
		p_t(e)
		-
		\frac{q_t(e)}{q_t(e_{-})}
	\right)^2.
\end{align*}
Define $q^*$ by \eqref{eq:defqstar-SP} with $\epsilon = m/T$.
Then,
from \eqref{eq:spdecomposition}
and Lemma~\ref{lem:stabLB-SP},
we have
\begin{align*}
	\Reg_T(p^*)
	&
	\le
	m
	+
	\Reg_T(q^*)
	\lesssim
	\E\left[
	\sum_{t=1}^T
		\sum_{e \in E}
		\eta_t(e)
		\rho_t(e)
		+
		\sum_{e \in E}
		\frac{1}{\eta_{T}(e)}
		\log \frac{1}{q^*(e)}
	\right]
	+
	m
	\\
	&
	\le
	\E\left[
		\sum_{e \in E}
		\left(
		\sum_{t=1}^T
		\eta_t(e)
		\rho_t(e)
		+
		\frac{\log(T)}{\eta_{T}(e)}
		\right)
	\right]
	+
	m
	\\
	&
	\lesssim
	\E\left[
		\sum_{e \in E}
		\sqrt{
		\sum_{t=1}^T
		\rho_t(e)
		\log(mT)
		}
	\right]
	+
	m \log(T),
\end{align*}
where the last inequality follows from the setting of $\eta_t(e)$:
\begin{align*}
	\eta_t(e)
	=
	\left(
		4
		+
		\frac{1}{ \log ( T) }
		\sum_{\tau=1}^{t-1}
		\rho_\tau(e)
	\right)^{-1/2}.
\end{align*}
This completes the proof of Theorem~\ref{thm:log-barrier-SP-appendix}.
\begin{proof}[Proof of Corollary~\ref{cor:log-barrier-SP-appendix}]
	We can show \eqref{eq:corLB1}
	by using the Cauchy-Schwarz inequality,
	Jensen's inequality,
	and the fact that $\sum_{e} p_t(e) \le L$.
	The other one can be shown via the following:
	\begin{align*}
		&
		\E \left[
			\sum_{v \in V \cup \{ s \}}
			\sum_{e \in \oes v}
			\sqrt{
			\sum_{t=1}^T
			c_t^2
			p_t(v)
			\left(
				p_t(e)
				-
				\frac{q_t(e)}{q_t(v)}
			\right)^2
			}
		\right]
		\\
		&
		\le
		\sum_{v \in V \cup \{ s \}}
		\sum_{e \in \oes v}
		\sqrt{
		\E\left[
			\sum_{t=1}^T
			p_t(v)
			\left(
				p_t(e)
				-
				\frac{q_t(e)}{q_t(v)}
			\right)^2
		\right]
		}
		\\
		&
		\le
		\sum_{v \in V \cup \{ s \}}
		\sum_{e \in \oes v}
		\sqrt{
		\E\left[
			\sum_{t=1}^T
			q_t(e)
			\left(
				1
				-
				\frac{q_t(e)}{q_t(v)}
			\right)
		\right]
		}
		\\
		&
		\le
		\sum_{v \in V \cup \{ s \}}
		\left(
		\sum_{e \in \oes v \setminus \pi^*(v)}
		\sqrt{
		\E\left[
			\sum_{t=1}^T
			q_t(e)
			\left(
				1
				-
				\frac{q_t(e)}{q_t(v)}
			\right)
		\right]
		}
		+
		\sqrt{
		\sum_{t=1}^T
		\E\left[
			q_t(\pi^*(v))
			\left(
				1
				-
				\frac{q_t(\pi^*(v))}{q_t(v)}
			\right)
		\right]
		}
		\right)
		\\
		&
		\le
		\sum_{v \in V \cup \{ s \}}
		\left(
		\sum_{e \in \oes v \setminus \pi^*(v)}
		\sqrt{
		\E\left[
			\sum_{t=1}^T
			q_t(e)
		\right]
		}
		+
		\sqrt{
		\E\left[
			\sum_{t=1}^T
			\left(
				q_t(v)
				-
				q_t(\pi^*(v))
			\right)
		\right]
		}
		\right)
		\\
		&
		\le
		2
		\sum_{v \in V \cup \{ s \}}
		\sum_{e \in \oes v \setminus \pi^*(v)}
		\sqrt{
		\E\left[
			\sum_{t=1}^T
			q_t(e)
		\right]
		}.
	\end{align*}
	Hence,
	by using the condition of the self-bounding constraint in Definition~\ref{def:SBR-SP},
	we obtain
	\begin{align*}
		\Reg_T(p^*)
		&
		\lesssim
		\sum_{v \in V \cup \{ s \}}
		\sum_{e \in \oes v \setminus \pi^*(v)}
		\sqrt{
		\E\left[
			\sum_{t=1}^T
			q_t(e)
			\log T
		\right]
		}
		+
		m \log T
		\\
		&
		=
		\sum_{v \in V \cup \{ s \}}
		\sum_{e \in \oes v \setminus \pi^*(v)}
		\sqrt{\frac{\log T}{\Delta(e)}}
		\sqrt{
		\E\left[
			\Delta (e)
			\sum_{t=1}^T
			q_t(e)
		\right]
		}
		+
		m \log T
		\\
		&
		\le
		\sqrt{
		\sum_{v \in V \cup \{ s \}}
		\sum_{e \in \oes v \setminus \pi^*(v)}
			\frac{\log T}{\Delta(e)}
			}
		\sqrt{
		\E\left[
		\sum_{v \in V \cup \{ s \}}
		\sum_{e \in \oes v \setminus \pi^*(v)}
			\Delta (e)
			\sum_{t=1}^T
			q_t(e)
		\right]
		}
		+
		m \log T
		\\
		&
		\le
		\sqrt{
		\sum_{v \in V \cup \{ s \}}
		\sum_{e \in \oes v \setminus \pi^*(v)}
			\frac{\log T}{\Delta(e)}
			}
		\sqrt{
			\Reg_T(p^*)
			+
			C
		}
		+
		m \log T
		\\
		&
		\le
		\sqrt{U \Reg_T(p^*)}
		+
		\sqrt{U C}
		+
		m\log T.
	\end{align*}
	This implies that $\Reg_T(p^*) \lesssim U + \sqrt{UC} + m \log T$.
\end{proof}

\subsection{Lower bound for stochastic environments}
As our problem is a special case of the linear bandit problem,
we can apply the lower bound given by \citep[Corollary 2]{lattimore2017end},
which is characterized by the following optimization problem:
\begin{align*}
	\inf_{\alpha \in [0, \infty]^{\cP}} \sum_{ p \in \cP^{-}} \alpha(p) \Delta(p) \quad  \mbox{subject to} \quad
	 \| p \|^{2}_{H(\alpha)^{-1}} \le \frac{\Delta(p)^2}{2} \quad & (\forall p \in P^{-}),
\end{align*}
where $\cP^{-} = \cP \setminus \argmin_{p \in \cP} \linner \ell^*, p \rinner, \Delta(p) = \linner \ell^*, p \rinner - \min_{p^* \in \cP} \linner \ell^*, p^* \rinner $,
and
\begin{align*}
	H(\alpha) = \sum_{p \in \cP} \alpha(p) p p^{\top}.
\end{align*}
However,
deriving an explicit expression for the optimal value of this optimization problem is not straightforward.
In fact, we have not yet identified such an expression.
Instead, in what follows, we present a lower bound by exploiting the specific structure inherent to this problem.

Consider stochastic environments specified by $\ell^*$ such that
$\linner \ell^*, p \rinner \in [3/8, 5/8]$ for all $p \in \cP$.
We suppose that $c_t \in \{0, 1 \}$ follows a Bernoulli distribution of parameter $\linner \ell, p_t \rinner$.
Suppose that $\Delta$ is defined as \eqref{eq:defDelta1}.
We then have
$
\linner \ell^*, p \rinner
-
\min_{p^* \in \cP} \linner \ell^*, p^* \rinner
=
\linner \Delta, p \rinner
$.
Denote
\begin{align}
	\label{eq:defDeltabar}
	\bar{\Delta}(e)
	=
	\min_{ p \in \cP : p(e)=1 }
	\left\{
	\linner
	\Delta ,
	p
	\rinner
	\right\}.
\end{align}
\begin{theorem}
	\label{thm:LBSP}
	Consider an arbitrary consistent algorithm,
	i.e.,
	assume that there exists $\epsilon \in (0, 1)$
	such that $\Reg_T \le M T^{1 - \epsilon}$ holds for any instances,
	where $M > 0$ is a parameter independent of $T$.
	We then have
	\begin{align}
		\label{eq:thmLBSP}
		\liminf_{T \to \infty} \frac{\Reg_T}{\log T}
		\gtrsim
		\epsilon
		\sum_{e \in E: \bar{\Delta}(e) > 0 } 
		\frac{\Delta(e)}{\bar{\Delta}(e)^2}.
	\end{align}
\end{theorem}
\begin{remark}
	If $\Delta$ is given by \eqref{eq:defDelta1},
	we have $\bar{\Delta}(e) = \Delta(e)$ for
	$e \in \oes v$ where $v \in {V}^* := \{ v \in V \mid \dist(s, v) + \dist(v, g) = \dist(s, g) \}$.
	In fact,
	for such an edge $e$,
	we have
	\begin{align*}
		\bar{\Delta}(e)
		=
		\dist(s, e_{-})
		+
		\ell(e)
		+
		\dist(e_{+}, g)
		-
		\dist(s, g)
		=
		\ell^*(e)
		+
		\dist(e_{+}, g)
		-
		\dist(e_{-}, g)
		=
		\Delta(e).
	\end{align*}
	We hence have
	\begin{align*}
		\sum_{e \in E: \bar{\Delta}(e) > 0}
		\frac{\Delta(e)}{\bar{\Delta}(e)^2}
		=
		\sum_{v \in {V} \cup \{ s \}}
		\sum_{e \in \oes v : \bar{\Delta}(e) > 0}
		\frac{\Delta(e)}{\bar{\Delta}(e)^2}
		\ge
		\sum_{v \in {V}^* \cup \{ s \}}
		\sum_{e \in \oes v : \bar{\Delta}(e) > 0}
		\frac{1}{{\Delta}(e)}.
	\end{align*}
\end{remark}
\begin{proof}[Proof of Theorem~\ref{thm:LBSP}]
	Let $p^* \in \argmin_{p \in \cP}\{ \linner \ell^*, p \rinner \}$.
	As we have
	\begin{align*}
		\Reg_T
		=
		\E \left[
			\sum_{t=1}^T
			\linner \ell^*, p_t - p^* \rinner
		\right]
		=
		\E \left[
			\sum_{t=1}^T
			\linner \ell^*, p_t - p^* \rinner
		\right]
		=
		\sum_{e \in E}\Delta(e) 
		\E \left[
			\sum_{t=1}^T
			p_t(e)
		\right],
	\end{align*}
	we have
	\begin{align}
		\label{eq:liminfqSP}
		\liminf_{T \to \infty} \frac{\Reg_T}{\log T}
		=
		\sum_{e \in E}\Delta(e) 
		\liminf_{T \to \infty}
		\frac{1}{\log T}
		\E \left[
			\sum_{t=1}^T
			p_t(e)
		\right].
	\end{align}
	In the following,
	we evaluate the value of
	$
		\liminf_{T \to \infty}
		\frac{1}{\log T}
		\E \left[
			\sum_{t=1}^T
			p_t(e)
		\right]
	$
	for any fixed $\tilde{e} \in E$ such that ${\delta} := \bar{\Delta}(\tilde{e}) > 0$.
	Consider a modified environment given by $\tilde{\ell}$ such that 
	$\tilde{\ell}(\tilde{e}) = \ell^*(\tilde{e}) - 2 \delta$
	and
	$\tilde{\ell}(e) = \ell^*(e) $ for $e \neq \tilde{e}$.
	Then,
	as we have $p^*(\tilde{e}) = 0$,
	it holds for any $p \in \cP$ that
	\begin{align}
		\nonumber
		\linner \tilde{\ell}, p - p^* \rinner 
		&
		=
		\linner {\ell}^*, p - p^* \rinner 
		- 2 \delta (p(\tilde{e}) - p^*(\tilde{e}))
		\\
		\nonumber
		&
		=
		\linner \Delta, p \rinner 
		- 2 \delta (p(\tilde{e}) - p^*(\tilde{e}))
		\\
		&
		\ge \delta p(\tilde{e})
		- 2 \delta (p(\tilde{e}) - p^*(\tilde{e}))
		=
		-
		\delta p(\tilde{e}).
		\label{eq:linnerellqSP}
	\end{align}
	where the inequality follows from the definition of $\delta = \bar{\Delta}(\tilde{e})$ given in \eqref{eq:defDeltabar}.
	Further,
	as there exists $p \in \cP$ such that
	$p(e) = 1$ and $\linner \Delta, p \rinner = \delta$,
	we have
	\begin{align}
		\min_{p \in \cP} \linner \tilde{\ell}, p - p^* \rinner 
		\le \delta - 2 \delta
		=
		- \delta.
		\label{eq:minelltildeSP}
	\end{align}
	Hence,
	from \eqref{eq:linnerellqSP} and \eqref{eq:minelltildeSP},
	the regret for the environment given by $\tilde{\ell}$ satisfies 
	\begin{align}
		\tilde{\Reg}_T 
		&
		=
		\max_{p \in \cP}
		\tilde{\E}
		\left[
			\sum_{t=1}^T
			\linner
			\tilde{\ell},
			p_t
			-
			p
			\rinner
		\right]
		\ge
		\delta \cdot
		\tilde{\E}
		\left[
			\sum_{t=1}^T
			\left(
			1
			-
			p_t(\tilde{e})
			\right)
		\right],
		\label{eq:tildeRegSP}
	\end{align}
	where $\tilde{\E}[ \cdot ]$ represents the expected value when feedback is generated by an environment associated with $\tilde{\ell}$.
	On the other hand,
	the regret for the environment given by ${\ell}^*$ satisfies 
	\begin{align}
		{\Reg}_T 
		=
		\E \left[
			\sum_{t=1}^T
			\linner
			\Delta,
			p_t
			\rinner
		\right]
		\ge
		\delta
		\cdot
		{\E} \left[
			\sum_{t=1}^T p_t(\tilde{e})
		\right]
		.
		\label{eq:RegSP}
	\end{align}
	Let $\TV$ denote the total variation distance between trajectories $((p_t, c_t))_{t=1}^T$ for environments with $\ell^*$ and $\tilde{\ell}$.
	Then,
	as we have
	$
	\frac{1}{T} 
		\sum_{t=1}^T p_t(\tilde{e})
		\in [0, 1]
	$,
	we have
	\begin{align*}
		\left|
		\tilde{\E}
		\left[
		\frac{1}{ T} 
		\sum_{t=1}^T p_t(\tilde{e})
		\right]
		-
		{\E}
		\left[
		\frac{1}{T} 
		\sum_{t=1}^T p_t(\tilde{e})
		\right]
		\right|
		\le
		\TV.
	\end{align*}
	We hence have
	\begin{align*}
		1 - \TV
		&
		\le
		1
		-
		\tilde{\E}
		\left[
		\frac{1}{ T} 
		\sum_{t=1}^T p_t(\tilde{e})
		\right]
		+
		{\E}
		\left[
		\frac{1}{T} 
		\sum_{t=1}^T p_t(\tilde{e})
		\right]
		\\
		&
		=
		\tilde{\E}
		\left[
		\frac{1}{ T} 
		\sum_{t=1}^T \left(
			1 - p_t(\tilde{e})
		\right)
		\right]
		+
		{\E}
		\left[
		\frac{1}{ T} 
		\sum_{t=1}^T p_t(\tilde{e})
		\right]
		\\
		&
		\le \frac{1}{\delta T} 
		\left(
			\tilde{\Reg}_T
			+
			\Reg_T
		\right),
	\end{align*}
	where the last inequality follows from \eqref{eq:tildeRegSP} and \eqref{eq:RegSP}.
	Here,
	from the Bretagnolle-Huber inequality (e.g., [\citet{canonne2022short}, Corollary 4])
	and the chain rule of the KL divergence,
	we have
	\begin{align*}
		1 - \TV
		&
		\ge
		\frac{1}{2}
		\exp\left(
			\sum_{t=1}^T 
			\E \left[ 
			\KL \left( \linner \ell^*, p_t \rinner ||
			\langle \tilde{\ell}, p_t \rangle
			\right)
			\right] 
		\right)
		\\
		&
		=
		\frac{1}{2}
		\exp\left(
			\sum_{t=1}^T 
			\E \left[ 
			p_t(\tilde{e})
			\KL \left( \linner \ell^*, p_t \rinner ||
			\langle {\ell}^*, p_t \rangle - 2\delta
			\right)
			\right] 
		\right)
		\\
		&
		\ge
		\frac{1}{2}
		\exp\left(
			- 5 \delta^2 \cdot
			\E \left[ \sum_{t=1}^T p_t(\tilde{e}) \right] 
		\right).
	\end{align*}
	Combining the above inequalities,
	we obtain
	\begin{align*}
		\E \left[ \sum_{t=1}^T p_t(\tilde{e}) \right] 
		&
		\gtrsim
		\frac{1}{\delta^2}
		\log \frac{1}{2(1 - \TV)}
		\ge
		\frac{1}{\delta^2}
		\log \frac{\delta T }{2 (\Reg_T + \tilde{\Reg}_T)}
		\\
		&
		\ge
		\frac{1}{\delta^2}
		\log \frac{\delta T }{4 M T^{1-\epsilon}}
		=
		\frac{1}{\delta^2}
		\left(
			\epsilon \log T
			+
			\log \left(
				\frac{\delta}{4M}
			\right)
		\right)
	\end{align*}
	Consequently,
	we have
	\begin{align*}
		\liminf_{T \to \infty}
		\frac{1}{\log T}
		\E \left[ \sum_{t=1}^T p_t(\tilde{e}) \right] 
		\gtrsim
		\frac{\epsilon}{\delta^2}
		=
		\frac{\epsilon}{\Delta(\tilde{e})^2}
	\end{align*}
	for any $\tilde{e} \in E$
	such that $\Delta(\tilde{e}) > 0$.
	By combining this with \eqref{eq:liminfqSP},
	we obtain \eqref{eq:thmLBSP}.
\end{proof}

\section{MDPs with known transition}
\label{appendix:known_transition}
In this section, we analyze our algorithm for episodic MDPs with known transitions. Specifically, we prove Theorems~\ref{thm:Tsallis-KT} and~\ref{thm:log-barrier-KT-appendix}, which together directly imply Corollary~\ref{cor:result-known-mdp} in the main body. In addition, we provide our lower bound result in Theorem~\ref{appendix:lower-bound-known}.
\subsection{Algorithm}
The algorithm's construction is almost identical to that of the shortest path case.
The unbiased loss estimator is defined in Lemma~\ref{lem:loss_estimator_KT}.
\begin{proof}[Proof of Lemma~\ref{lem:loss_estimator_KT}]
	For notational simplicity, we omit the conditioning in expectations throughout this proof.
	Fix an arbitrary $s \in S_k$ and $a$.
	We then have
	\begin{align*}
		&
		\E \left[
			c_t \cdot \frac{\ind_t (s,a)}{q_t(s,a)}
		\right]
		=
		\E \left[
			\sum_{i=0}^{L-1}
			\ell(s_i^t, a_i^t)
			\mid
			\ind_t (s, a) = 1
		\right]
		=
		\E \left[
			\sum_{i=0}^{L-1}
			\ell(s_i^t, a_i^t)
			\mid
			(s_k^t, a_k^t) = (s, a)
		\right]
		\\
		&
		=
		\E \left[
			\sum_{i=0}^{k-1}
			\ell(s_i^t, a_i^t)
			\mid
			s_k^t = s
		\right]
		+
		\E \left[
			\sum_{i=k}^{L-1}
			\ell(s_i^t, a_i^t)
			\mid
			s_k^t = s,
			a_k^t = a
		\right]
		\\
		&
		=
		\E \left[
			\sum_{i=0}^{k-1}
			\ell(s_i^t, a_i^t)
			\mid
			s_k^t = s
		\right]
		+
		Q^{\pi_t}(s, a ; \ell_t).
	\end{align*}
	Similarly,
	for any fixed $s \in S_k$,
	we have
	\begin{align*}
		\E \left[
			c_t \cdot \frac{\ind_t (s)}{q_t(s)}
		\right]
		=
		\E \left[
			\sum_{i=0}^{L-1}
			\ell(s_i^t, a_i^t)
			\mid
			s_k^t = s
		\right]
		=
		\E \left[
			\sum_{i=0}^{k-1}
			\ell(s_i^t, a_i^t)
			\mid
			s_k^t = s
		\right]
		+
		V^{\pi_t}(s ; \ell_t).
	\end{align*}
	By combining the above two equalities,
	we obtain
	\begin{align*}
		\E \left[
			\widehat{\ell}_t(s, a)
		\right]
		=
		\E \left[
			c_t \cdot \frac{\ind_t(s,a)}{q_t(s,a)}
		\right]
		-
		\E \left[
			c_t \cdot \frac{\ind_t(s)}{q_t(s)}
		\right]
		=
		Q^{\pi_t}(s, a; \ell_t)
		-
		V^{\pi_t}(s; \ell_t).
	\end{align*}
\end{proof}
In the following,
we denote 
\begin{align*}
	\bar{\ell}_t
	:=
	\E \left[
		\widehat{\ell}_t \mid
		\pi_t, \ell_t
	\right]
	=
	Q^{\pi_t}(s, a; \ell_t)
	-
	V^{\pi_t}(s; \ell_t)
	\in
	[-1,1]^{S \times A}.
\end{align*}
By combining Lemmas~\ref{lem:PDL} and \ref{lem:loss_estimator_KT},
we obtain the following expression of the regret:
\begin{align}
	\nonumber
	\Reg_T(\pi^*)
	&
	=
	\E \left[
		\sum_{t=1}^T
		\left(
		V^{\pi_t}(s_0; \ell_t)
		-
		V^{\pi^*}(s_0; \ell_t)
		\right)
	\right]
	\\
	\nonumber
	&
	=
	\E \left[
		\sum_{t=1}^T
		\left(
		V^{\pi_t}(s_0; \bar{\ell}_t)
		-
		V^{\pi^*}(s_0; \bar{\ell}_t)
		\right)
	\right]
	\\
	\nonumber
	&
	=
	\E \left[
		\sum_{t=1}^T
		\linner
		\bar{\ell}_t,
		q_t - q^*
		\rinner
	\right]
	\\
	&
	=
	\E \left[
		\sum_{t=1}^T
		\linner
		\widehat{\ell}_t,
		q_t - q^*
		\rinner
	\right],
	\label{eq:reg-KT}
\end{align}
where we denote $q^{*} = q^{\pi^*}$.
To upper bound the value of
$
\sum_{t=1}^T
\linner
\widehat{\ell}_t,
q_t - q^*
\rinner
$,
we choose $q_t \in \cQ$ by using the following FTRL approach similarly to the case of online shortest path problem:
\begin{align*}
	q_t \in \argmin_{q \in \cQ}
	\left\{
		\linner
		\sum_{\tau = 1}^{t-1} \widehat{\ell}_{\tau},
		q
		\rinner
		+
		\psi_t(q)
	\right\},
\end{align*}
where 
\begin{align}
  \psi_t(q)
  &
  =
  -
  \frac{2}{\eta_t}
  \sum_{s \neq s_L, a \in A}
  \sqrt{q(s,a)}
  -
  \sum_{s \neq s_L, a \in A}
  \beta
  \ln q(s,a)
  \quad
  \mbox{with}
  \quad
  \eta_t = \frac{1}{\sqrt{t}},
  ~
  \beta = 2,
  \quad
  \mbox{or}
  \label{eq:defpsi-KT-Tsallis-app}
  \\
  \psi_t(q)
  &
  =
  -
  \sum_{s \neq s_L, a \in A}
  \frac{1}{\eta_t(s,a)}
  \ln q(s,a)
  \quad
  \mbox{with}
  \quad
  \eta_t(s,a) 
  = \big(4 + \frac{1}{\ln T} \sum_{\tau = 1}^{t-1} 
  \rho_{\tau}(s,a)
  \big)^{-\frac{1}
  {2}},
  \label{eq:defpsi-KT-LB-app}
\end{align}
where $\rho_{\tau} \in [0,1]$ will be defined later.

\subsection{Regret analysis}
In our regret analysis,
we use the following lemma:
\begin{lemma}[First part of Lemma~\ref{lem:second-moment}]
	\label{lem:ellhatsquare}
	If $\widehat{\ell}_t$ is given by as in Lemma~\ref{lem:loss_estimator_KT},
	the expectation of $\widehat{\ell}(s,a)^2$ taken w.r.t.~the randomness of $p$ satisfies
	\begin{align*}
		\E \left[
			\widehat{\ell}_t(s, a)^2
		\right]
		\le
		\frac{ 1 - \pi_t(a | s) }{q_t(s, a)}.
	\end{align*}
\end{lemma}
\begin{proof}
	For notational simplicity, we omit the subscript $t$ throughout this proof.
	From the definition of $\widehat{\ell}$,
	we have
	\begin{align*}
		\E \left[
			\widehat{\ell}(s, a)^2
		\right]
		&
		\le
		q(s)
		\left(
			\pi(a | s)
			\left(
				\frac{1}{q(s,a)}
				-
				\frac{1}{q(s)}
			\right)^2
			+
			(1 - \pi(a | s))
			\frac{1}{q(s)^2}
		\right)
		\\
		&
		=
		\frac{q(s)}{q(s,a)^2}
		\left(
			\pi(a|s)
			(1 - \pi(a|s))^2
			+
			(1 - \pi(a|s)) \pi(a|s)^2
		\right)
		\\
		&
		=
		\frac{q(s)}{q(s,a)^2}
		\pi(a|s) (1 - \pi(a|s))
		=
		\frac{1 - \pi(a|s)}{q(s,a)}.
	\end{align*}
\end{proof}
We also use $q_0$ and $\tilde{q}^*$,
which are defined as follows in the analysis:
For all $s\in S$ and $a\in A$,
let $q_{s,a} \in \argmax_{q' \in \cQ} q'(s,a)$.
Define $q_0 \in \cQ $ by
$q_0 = \frac{1}{|S| |A|} \sum_{s, a} q_{s, a}$.
For $q^* = q^{\pi^*} \in \cQ$ and $\epsilon \in [0,1]$,
define $\tilde{q}^*$ by
\begin{align}
	\label{eq:defqstar-KT}
	\tilde{q}^*
	=
	(1-\epsilon) q^* +  \epsilon q_0.
\end{align}
Then,
it holds for any $q \in \cQ$,
$s \in S$ and $a \in A$ that
\begin{align}
	\label{eq:boundratio-KT}
	\frac{q(s,a)}{\tilde{q}^*(s,a)}
	\le
	\frac{1}{\epsilon}
	\frac{q(s,a)}{{q}_0(s,a)}
	\le
	\frac{|S||A|}{\epsilon}
	\frac{q(s,a)}{{q}_{s,a}(s,a)}
	\le
	\frac{|S||A|}{\epsilon}.
\end{align}
From \eqref{eq:reg-KT},
by a similar way to \eqref{eq:spdecomposition},
we can show that
\begin{align}
	\Reg_T(\pi^*)
	\le
	\epsilon
	T 
	+
	\sum_{t=1}^T 
		\E \left[
				\underset{=: \stab_t}{
				\underline{
		\linner \widehat{\ell}_t, {q}_t - {q}_{t+1} \rinner
		-
		D_t({q}_{t+1}, {q}_t )
				}
			}
		+
			\underset{=: \pena_t}{
				\underline{
		\left(
			\psi_t(\tilde{q}^*)
			-
			\psi_{t-1}(\tilde{q}^*)
			-
			\psi_t({q}_{t})
			+
			\psi_{t-1}({q}_{t})
		\right)
				}
			}
	\right].
	\label{eq:spdecomposition-KT}
\end{align}

\subsubsection{Analysis for Tsallis-entropy case}
\begin{theorem}
	\label{thm:Tsallis-KT}
	For any deterministic policy $\pi^* \in \Pi$,
	the proposed algorithm with the Tsallis-entropy regularizer \eqref{eq:defpsi-KT-Tsallis-app} achieves:
	\begin{align*}
		&
		\Reg_T(\pi^*)
		\\
		&
		\lesssim
		\sum_{t=1}^T
		\frac{1}{\sqrt{t}}
		\E
		\left[
		\sum_{s \neq s_L}
		\sum_{a \neq \pi^*(s)}
		\sqrt{q_t(s, a)}
		+
		\sqrt{
			L |S||A|
			\sum_{s \neq s_L}
			\sum_{a \neq \pi^*(s)}
			(q_t(s, a) + q^*(s, a))
		}
		\right]
		+
		|S||A|\log T
	\end{align*}
\end{theorem}
In the proof of this theorem,
we can bound the stability term by using the following lemma:
\begin{lemma}
	\label{lem:stabTsallis-KT}
	When we use the Tsallis-entropy regularizer \eqref{eq:defpsi-SP-Tsallis-app},
	stability terms are bounded as
	\begin{align}
		\label{eq:stabTsallis-KT}
		\E [ \stab_t ]
		&
		\lesssim
		\E\left[
		\eta_t
		\sum_{s \neq s_L}
		\left(
		\sum_{a \neq \pi^*(s)}
		\sqrt{q_t(s,a)}
		-
		\sqrt{q_t(s)}
		\right)
		\right]
		\lesssim
		\E\left[
		\eta_t
		\sum_{s \neq s_L}
		\sum_{a \neq \pi^*(s)}
		\sqrt{q_t(s, a)}
		\right].
	\end{align}
\end{lemma}
This can be shown in a similar way as Lemma~\ref{lem:stabTsallis-SP},
by using Lemmas~\ref{lem:stabTsallisLB} and \ref{lem:ellhatsquare}.
Furthermore,
the penalty term can be bounded in the same manner as done by \citet{jin2020simultaneously},
using their Lemma 6 with $\alpha = 0$.
By combining these results,
Theorem~\ref{thm:Tsallis-KT} can be established in the same way as Theorem~\ref{thm:Tsallis-SP}.

\subsubsection{Analysis for log-barrier case}
In the case of the log-barrier regularizer,
we have the following regret bounds:
\begin{theorem}
	\label{thm:log-barrier-KT-appendix}
	For any deterministic policy $\pi^*$,
	the proposed algorithm with the log-barrier regularizer \eqref{eq:defpsi-KT-LB-app} achieves:
	\begin{align}
		\Reg_T
		\lesssim
		\E \left[
			\sum_{s \neq s_L}
			\sum_{a \neq \pi^*(s)}
			\sqrt{
			\sum_{t=1}^T
			c_t^2
			\ind_t(s)
			\left(
				\ind_t(s,a)
				-
				\pi_t(a|s)
			\right)^2
			\log (T)
			}
		\right]
		+
		|S||A| \log (T).
	\end{align}
\end{theorem}
To show these,
we use the following lemmas:
\begin{lemma}
	\label{lem:stabLB-KT}
	When we use the log-barrier regularizer \eqref{eq:defpsi-KT-LB-app},
	we have
	\begin{align*}
		\stab_t
		\lesssim
		c_t^2
		\sum_{s \neq s_L}
		\sum_{a \in A}
		\eta_t(s,a)
		\ind_t(s)
		\left(
			\ind_t(s,a)
			-
			\pi_t(a|s)
		\right)^2.
	\end{align*}
\end{lemma}
\begin{proof}
We have
\begin{align*}
	\stab_t
	&
	=
	\linner 
	\widehat{\ell}_t, q_t - q_{t+1}
	\rinner
	-
	D_t(q_{t+1}, q_t)
	\\
	&
	=
	\sum_{s,a}
	\left(
		\widehat{\ell}_t(s,a) (q_t(s,a) - q_{t+1}(s,a))
		-
		\frac{1}{\eta_t(s,a)} D_{\phi}(q_{t+1}(s,a), q_t(s,a))
	\right),
\end{align*}
where $D_{\phi}$ is the Bregman divergence associated with $\phi(x) = - \ln(x)$.
As we have
\begin{align*}
	- \eta_t(s,a) q_t(s, a) 
	\widehat{\ell}_t(s, a) 
	\le
	\eta_t(s,a) q_t(s, a)
	\frac{c_t}{q_t(s)} 
	\le
	\eta_t(s,a)
	\le
	\frac{1}{2},
\end{align*}
we can apply Lemma~\ref{lem:stabLB} to obtain the following:
\begin{align*}
	\stab_t
	\le
	\sum_{s,a} \eta_t(s,a) q_t(s,a)^2 \widehat{\ell}_t(s,a)^2
	=
	\sum_{s,a} \eta_t(s,a) c_t^2 \ind_t(s) \left( \ind_t(s, a) - \pi_t(a | s) \right)^2.
\end{align*}
\end{proof}
Define $\rho_t(s,a)$ by
\begin{align*}
	\rho_t(s,a) = c_t^2 
	\ind_t(s)
	\left(
		\ind_t(s, a)
		-
		\frac{q_t(s,a)}{q_t(s)}
	\right)^2.
\end{align*}
Then,
Theorem~\ref{thm:log-barrier-KT-appendix} can be established in the same way as Theorem~\ref{thm:log-barrier-SP-appendix}.

Lastly,
results in
Corollary~\ref{cor:result-known-mdp} follows from Theorems~\ref{thm:Tsallis-KT} and \ref{thm:log-barrier-KT-appendix} by an argument similar to that used for Corollaries~\ref{cor:Tsallis-SP} and \ref{cor:log-barrier-SP-appendix}.

\subsection{Lower bound for stochastic MDPs}
Consider stochastic environment in which $c_t$ follows a Bernoulli distribution of parameter $\linner \ell^*, p_t \rinner$,
where we assume that $\ell^*: S \times A \to [0, 1] $ satisfies
$\linner \ell^*, p \rinner \in [3/8, 5/8]$ for any possible trajectories $p$
and $\ell^*(s_L, a) = 0$ for all $a \in A$.
Define $\Delta:S \times A \to [0, 1]$ by optimal $Q$ function values:
\begin{align*}
	Q^*(s_L, a) 
	&
	= 0
	&
	(a \in A),
	\\
	Q^*(s, a) 
	&
	= \ell^*(s, a) + 
	\sum_{s' \in S} P(s'|s,a) \cdot 
	V^*(s'),
	&
	(s \in S \setminus \{ s_L \}, a \in A),
	\\
	V^*(s) 
	&
	=
	\min_{a' \in A} Q^*(s, a')
	&
	(s \in S),
	\\
	\Delta(s, a) 
	&
	= Q^*(s, a) - \min_{a' \in A} Q^*(s, a')
	= Q^*(s, a) - V^*(s)
	&
	(s \in S , a \in A).
\end{align*}
We then have $\min_{a \in A} \Delta(s,a) = 0$ for all $s \in S$ and $\| \Delta \|_1 \le 1/4$.
Also,
we have
\begin{align*}
	\linner \ell^*, q \rinner
	&
	=
	\sum_{k = 0}^{L-1}
	\sum_{s \in S_k}
	\sum_{a \in A}
	\ell^*(s, a) 
	q(s, a)
	\\
	&
	=
	\sum_{k = 0}^{L-1}
	\sum_{s \in S_k}
	\sum_{a \in A}
	\left(
		Q^*(s,a) 
		- \sum_{s' \in S_{k+1}} P(s'|s,a) \cdot V^*(s')
	\right)
	q(s, a)
	\\
	&
	=
	\sum_{k = 0}^{L-1}
	\left(
		\sum_{s \in S_k}
		\sum_{a \in A}
		Q^*(s,a) 
		q(s,a)
		- 
		\sum_{s' \in S_{k+1}} q(s') \cdot V^*(s')
	\right)
	\\
	&
	=
	\sum_{k = 0}^{L-1}
	\left(
		\sum_{s \in S_k}
		\sum_{a \in A}
		\left(
		Q^*(s,a) 
		-
		V^*(s)
		\right)
		q(s,a)
	\right)
	+
	V^*(s_0)
	-
	V^*(s_L)
	\\
	&
	=
	\linner \Delta, q \rinner + \min_{q^* \in \cQ} \linner \ell^*, q^* \rinner
\end{align*}
as $V^*(s_L) = \ell^*(s_L,a) = 0$ for all $a \in A$.
Denote 
\begin{align*}
\cQ^* 
&
= \argmin_{q \in \cQ} \left\{ \linner \ell^*, q \rinner \right\} 
= \argmin_{q \in \cQ} \left\{ \linner 
\Delta, q
\rinner \right\} 
\\
&
= \{ q \in \cQ \mid \linner \Delta, q \rinner = 0  \}
= \{ q \in \cQ \mid \Delta(s,a) > 0 \Longrightarrow q(s,a)=0  \}.
\end{align*}
We also define
\begin{align*}
	S^*
	=
	\{ s \in S \setminus \{ s_L \} \mid \exists q \in \cQ^*, q(s) > 0 \}.
\end{align*}
Then,
for any consistent algorithms,
we have
\begin{align*}
	\liminf_{T \to \infty} \frac{\Reg_T}{\log T}
	\gtrsim
	\sum_{s \in S^*}\sum_{a \in A : \Delta(s,a) > 0} \frac{1}{\Delta(s,a)}.
\end{align*}
\begin{theorem}\label{appendix:lower-bound-known}
	Consider an arbitrary consistent algorithm,
	i.e.,
	assume that there exists $\epsilon \in (0, 1)$
	such that $\Reg_T \le M T^{1 - \epsilon}$ holds for all instances,
	where $M > 0$ is a parameter independent of $T$.
	Then for any MDP with $\ell^* : S \times A \rightarrow [0,1]$  satisfying
$\linner \ell^*, p \rinner \in [3/8, 5/8]$ for any possible trajectories $p$
and $\ell^*(s_L, a) = 0$ for all $a \in A$, we have
	\begin{align}
		\label{eq:thmLBMDP}
		\liminf_{T \to \infty} \frac{\Reg_T}{\log T}
		\gtrsim
		\epsilon
		\sum_{s \in S^*}\sum_{a \in A : \Delta(s,a) > 0} \frac{1}{\Delta(s,a)}.
	\end{align}
\end{theorem}
\begin{proof}
	Let $q^* \in \cQ^*$.
	As we have
	\begin{align*}
		\Reg_T
		=
		\E \left[
			\sum_{t=1}^T
			\linner \ell^*, p_t - q^* \rinner
		\right]
		=
		\E \left[
			\sum_{t=1}^T
			\linner \ell^*, q_t - q^* \rinner
		\right]
		=
		\sum_{s, a}\Delta(s,a) 
		\E \left[
			\sum_{t=1}^T
			q_t(s,a)
		\right],
	\end{align*}
	we have
	\begin{align}
		\label{eq:liminfq}
		\liminf_{T \to \infty} \frac{\Reg_T}{\log T}
		=
		\sum_{s, a}\Delta(s,a) 
		\liminf_{T \to \infty}
		\frac{1}{\log T}
		\E \left[
			\sum_{t=1}^T
			q_t(s,a)
		\right].
	\end{align}
	In the following,
	we evaluate the value of
	$
		\liminf_{T \to \infty}
		\frac{1}{\log T}
		\E \left[
			\sum_{t=1}^T
			q_t(s,a)
		\right]
	$
	for any fixed $\tilde{s} \in S^*$ and $\tilde{a} \in A$ such that ${\delta} := \Delta(\tilde{s}, \tilde{a}) > 0$.
	Let $q^* \in \argmax_{q \in \cQ^*} \left\{ q(\tilde{s}) \right\}$ and denote $\bar{q} = \max_{q \in \cQ^*} \{ q(\tilde{s}) \} = q^*(\tilde{s}) > 0$.
	Then,
	from Lemma~\ref{lem:Deltaqc} below,
	there exists $c \in (0, \delta]$ such that
	\begin{align*}
		\linner \Delta, q \rinner \ge c \max\{ 0, q(\tilde{s}) - \bar{q} \}
		+
		{\delta} q(\tilde{s}, \tilde{a}).
	\end{align*}
	Consider a modified environment given by $\tilde{\ell}$ such that 
	$\tilde{\ell}(\tilde{s}, \tilde{a}) = \ell^*(\tilde{s}, \tilde{a}) - \delta - c/2$
	and
	$\tilde{\ell}(s, a) = \ell^*(s,a) $ for $(s, a) \neq (\tilde{s}, \tilde{a})$.
	Then,
	as we have $q^*(\tilde{s}, \tilde{a}) = 0$,
	it holds for any $q \in \cQ$ that
	\begin{align}
		\linner \tilde{\ell}, q - q^* \rinner 
		&
		=
		\linner {\ell}^*, q - q^* \rinner 
		- \left(\delta + \frac{c}{2} \right) (q(\tilde{s}, \tilde{a}) - q^*(\tilde{s}, \tilde{a}))
		\nonumber
		\\
		&
		=
		\linner \Delta, q \rinner 
		- \left(\delta + \frac{c}{2} \right) q(\tilde{s}, \tilde{a}) 
		\ge
		c \max \{0, q(\tilde{s}) - \bar{q} \}
		- \frac{c}{2} q(\tilde{s}, \tilde{a})
		\nonumber
		\\
		&
		\ge
		c \max \{0, q(\tilde{s},\tilde{a}) - \bar{q} \}
		- \frac{c}{2} q(\tilde{s}, \tilde{a})
		=
		\frac{c}{2} \left( | q(\tilde{s},\tilde{a}) - \bar{q} | - \bar{q} \right)
		\label{eq:linnerellq}
	\end{align}
	Further,
	as there exists $q \in \cQ$ such that
	$\linner \tilde{\ell}, q - q^* \rinner = - \frac{c \bar{q}}{2}$,
	we have
	\begin{align}
		\min_{q \in \cQ} \linner \tilde{\ell}, q \rinner 
		= \linner \tilde{\ell}, q^* \rinner - \frac{c \bar{q}}{2}.
		\label{eq:minelltilde}
	\end{align}
	In fact,
	such an occupancy measure $q$ can be constructed from a corresponding policy $\tilde{\pi}:S \to A$,
	by modifying a policy $\pi^*: S \to A$ corresponding to $q^* \in \cQ^*$ such that $q^*(s) = \bar{q}$.
	We set $\tilde{\pi}(s) = \pi^*(s)$ for all $s \neq \tilde{s}$ and set $\tilde{\pi}(\tilde{s}) = \tilde{a}$.
	An occupancy measure $q$ corresponding to $\tilde{\pi}$ satisfies
	$q(\tilde{s}, \tilde{a}) = \bar{q}$ and
	$\linner \ell^*, q \rinner = \bar{q} \delta$,
	which implies that
	$\linner \tilde{\ell}, q \rinner = \frac{c\bar{q}}{2} $.
	Hence,
	the regret for the environment given by $\tilde{\ell}$ satisfies 
	\begin{align}
		\tilde{\Reg}_T 
		&
		=
		\max_{q \in \cQ}
		\tilde{\E}
		\left[
			\sum_{t=1}^T
			\linner
			\tilde{\ell},
			q_t
			-
			q
			\rinner
		\right]
		&
		\nonumber
		\\
		&
		=
		\tilde{\E}
		\left[
			\sum_{t=1}^T
			\left(
			\linner
			\tilde{\ell},
			q_t
			-
			q^*
			\rinner
			+
			\frac{c \bar{q}}{2}
			\right)
		\right]
		&
		\mbox{(from \eqref{eq:minelltilde})}
		\nonumber
		\\
		&
		\ge
		\frac{c}{2}
		\tilde{\E}
		\left[
			\sum_{t=1}^T
			\left|
			q_t(\tilde{s}, \tilde{a})
			-
			\bar{q}
			\right|
		\right]
		&
		\mbox{(from \eqref{eq:linnerellq})}
		\nonumber
		\\
		&
		\ge
		\frac{c}{2}
		\tilde{\E}
		\left[
			\sum_{t=1}^T
			\max\left\{
				0,
				\bar{q}
				-
				q_t(\tilde{s}, \tilde{a})
			\right\}
		\right],
		\label{eq:tildeRegTLBMDP}
	\end{align}
	where $\tilde{\E}[ \cdot ]$ represents the expected value when feedback is generated by an environment associated with $\tilde{\ell}$.
	On the other hand,
	the regret for the environment given by ${\ell}^*$ satisfies 
	\begin{align}
		{\Reg}_T \ge
		\delta
		\cdot
		{\E} \left[
			\sum_{t=1}^T q_t(\tilde{s}, \tilde{a})
		\right]
		\ge
		\delta
		\cdot
		{\E} \left[
			\sum_{t=1}^T \min\left\{ \bar{q}, q_t(\tilde{s}, \tilde{a}) \right\}
		\right]
		.
		\label{eq:RegTLBMDP}
	\end{align}
	Let $\TV$ denote the total variation distance between trajectories $((q_t, p_t, c_t))_{t=1}^T$ for environments with $\ell^*$ and $\tilde{\ell}$.
	Then,
	as we have
	$
	\frac{1}{\bar{q} T} 
		\sum_{t=1}^T \min\left\{ \bar{q}, q_t(\tilde{s}, \tilde{a}) \right\}
		\in [0, 1]
	$,
	we have
	\begin{align*}
		\left|
		\tilde{\E}
		\left[
		\frac{1}{\bar{q} T} 
		\sum_{t=1}^T \min\left\{ \bar{q}, q_t(\tilde{s}, \tilde{a}) \right\}
		\right]
		-
		{\E}
		\left[
		\frac{1}{\bar{q} T} 
		\sum_{t=1}^T \min\left\{ \bar{q}, q_t(\tilde{s}, \tilde{a}) \right\}
		\right]
		\right|
		\le
		\TV.
	\end{align*}
	We hence have
	\begin{align*}
		1 - \TV
		&
		\le
		1
		-
		\tilde{\E}
		\left[
		\frac{1}{\bar{q} T} 
		\sum_{t=1}^T \min\left\{ \bar{q}, q_t(\tilde{s}, \tilde{a}) \right\}
		\right]
		+
		{\E}
		\left[
		\frac{1}{\bar{q} T} 
		\sum_{t=1}^T \min\left\{ \bar{q}, q_t(\tilde{s}, \tilde{a}) \right\}
		\right]
		\\
		&
		=
		\tilde{\E}
		\left[
		\frac{1}{\bar{q} T} 
		\sum_{t=1}^T \max\left\{ 0, \bar{q} - q_t(\tilde{s}, \tilde{a}) \right\}
		\right]
		+
		{\E}
		\left[
		\frac{1}{\bar{q} T} 
		\sum_{t=1}^T \min\left\{ \bar{q}, q_t(\tilde{s}, \tilde{a}) \right\}
		\right]
		\\
		&
		\le \frac{1}{\bar{q}T} 
		\left(
			\frac{2}{c} \tilde{\Reg}_T
			+
			\frac{1}{\delta} \Reg_T
		\right),
	\end{align*}
	where the last inequality follows from \eqref{eq:tildeRegTLBMDP} and \eqref{eq:RegTLBMDP}.
	Here,
	from the Bretagnolle-Huber inequality (e.g., [\citet{canonne2022short}, Corollary 4])
	and the chain rule of the KL divergence,
	we have
	\begin{align*}
		1 - \TV
		&
		\ge
		\frac{1}{2}
		\exp\left(
			\sum_{t=1}^T 
			\E \left[ 
			\KL \left( \linner \ell^*, p_t \rinner ||
			\langle \tilde{\ell}, p_t \rangle
			\right)
			\right] 
		\right)
		\\
		&
		=
		\frac{1}{2}
		\exp\left(
			\sum_{t=1}^T 
			\E \left[ 
			q_t(\tilde{s}, \tilde{a})
			\KL \left( \linner \ell^*, p_t \rinner ||
			\langle {\ell}^*, p_t \rangle - \delta - \frac{c}{2}
			\right)
			\right] 
		\right)
		\\
		&
		\ge
		\frac{1}{2}
		\exp\left(
			- 5 \left( \delta + \frac{c}{2} \right)^2 \cdot
			\E \left[ \sum_{t=1}^T q_t(\tilde{s}, \tilde{a}) \right] 
		\right).
	\end{align*}
	Combining the above inequalities and $c \in (0, \delta ]$,
	we obtain
	\begin{align*}
		\E \left[ \sum_{t=1}^T q_t(\tilde{s}, \tilde{a}) \right] 
		&
		\gtrsim
		\frac{1}{\delta^2}
		\log \frac{1}{2(1 - \TV)}
		\ge
		\frac{1}{\delta^2}
		\log \frac{\bar{q} {c} T }{4 (\Reg_T + \tilde{\Reg}_T)}
		\\
		&
		\ge
		\frac{1}{\delta^2}
		\log \frac{\bar{q} {c} T }{8 M T^{1-\epsilon}}
		=
		\frac{1}{\delta^2}
		\left(
			\epsilon \log T
			+
			\log \left(
				\frac{\bar{q}c}{8M}
			\right)
		\right)
	\end{align*}
	Consequently,
	we have
	\begin{align*}
		\liminf_{T \to \infty}
		\frac{1}{\log T}
		\E \left[ \sum_{t=1}^T q_t(\tilde{s}, \tilde{a}) \right] 
		\gtrsim
		\frac{\epsilon}{\delta^2}
		=
		\frac{\epsilon}{\Delta(\tilde{s},\tilde{a})^2}
	\end{align*}
	for any $\tilde{s} \in S^*$ and $\tilde{a} \in A$ 
	such that $\Delta(\tilde{s}, \tilde{a}) > 0$.
	By combining this with \eqref{eq:liminfq},
	we obtain \eqref{eq:thmLBMDP}.
\end{proof}

\begin{lemma}
	\label{lem:Deltaqc}
	Suppose that $s \in S^*$ and denote $\bar{q} = \max_{q \in \cQ^*} \left\{ q(s) \right\}$.
	Then,
	there exists $c > 0$ such that
	the following holds for all $q \in \cQ$:
	\begin{align*}
		\linner \Delta, q \rinner \ge
		c \max\{ 0, q(s) - \bar{q} \} 
		+
		\sum_{a \in A} \Delta(s, a) q(s, a).
	\end{align*}
\end{lemma}
\begin{proof}
	Suppose that $s \in S^* \cap S_k$.
	Decompose $\Delta$ as $\Delta = \Delta_{<k} + \Delta_{\ge k}$,
	where
	\begin{align*}
		(\Delta_{<k}(s, a), \Delta_{\ge k}(s, a))
		=
		\left\{
			\begin{array}{ll}
				(\Delta(s, a), 0) & 
				\text{if } s \in \bigcup_{k' < k} S_{k'} \\
				(0, \Delta(s, a)) & 
				\text{if } s \in \bigcup_{k' \ge  k} S_{k'} \\
			\end{array}
		\right. .
	\end{align*}
	We define $q_{<k}$ and $q_{\ge k}$ in the same way.
	Define 
	$f(x) = \inf_{q \in \cQ : q(s) = x} \linner \Delta, q \rinner$.
	We then have
	\begin{align*}
	f(x) 
	= \inf_{q \in \cQ : q(s) = x} \linner \Delta, q \rinner
	= \inf_{q \in \cQ : q(s) = x} \linner \Delta_{<k} + \Delta_{\ge k}, q \rinner
	= \inf_{q \in \cQ : q(s) = x} \linner \Delta_{<k}, q \rinner.
	\end{align*}
	The last equality follows from the fact that,
	for any $q \in \cQ$ 
	(corresponding to $\pi \in \Pi$)
	such that $q(s) = x$,
	there exists $q' \in \cQ$ such that
	$q(s) = x$,
	$\linner \Delta_{<k},  q' \rinner = \linner \Delta_{< k}, q \rinner$,
	and
	$\linner \Delta_{\ge k},  q' \rinner = 0$.
	Such an occupancy measure $q'$ can be constructed by a policy $\pi' \in \Pi$ given as
	$\pi'(s,a) = \pi(s,a)$ for $s \in \bigcup_{k' < k} S_{k'}$ and
	$\pi'(s,a) = \pi^*(s,a)$ for $s \in \bigcup_{k' \ge k} S_{k'}$.

	Define 
	$\underline{x} = \min_{q \in \cQ} \{ q(s) \}$
	and
	$\bar{x} = \max_{q \in \cQ} \{ q(s) \}$.
	We note that
	$\underline{x} \le \bar{q} \le \bar{x}$ and
	$f(x) < +\infty$ if and only if $\underline{x} \le x \le \bar{x}$.
	As $\cQ$ is a polytope,
	$f(x)$ is a piecewise linear function in $x$,
	i.e.,
	there exists a finite sequence of real numbers 
	$x_0=\underline{x} < x_1 < x_2 < \cdots < x_n = \bar{x} \in \mathbb{R}$
	such that $f(x)$ is an affine function in each interval $[x_i, x_{i+1}]$.
	From the definition of $\cQ^*$ and $\bar{q}$,
	we have $f(x) > 0$ for any $x > \bar{q}$.
	Indeed,
	if $q(s) > \bar{q}$ then $q \notin \cQ^*$,
	which means that $\linner \Delta, q \rinner > 0$.
	Hence,
	$c \in \mathbb{R}$
	defined as
	\begin{align*}
		c = \inf \left\{ \frac{f(x_i)}{x_i - \bar{q}} \mid i \in [n], x_i > \bar{q} \right\}
	\end{align*}
	is positive.
	(When $\bar{q} = \bar{x}$,
	i.e.,
	when there is no $x_i > \bar{q}$,
	we let $c$ be an arbitrary positive number.)
	Further,
	as $f(x) \ge 0$ for all $x$ and $f(x)$ is affine in each interval $[x_i, x_{i+1}]$,
	we have
	$f(x) \ge c \max\{0, x - \bar{q} \}$ for all $x \le \bar{x}$.
	From this,
	we have
	\begin{align*}
		\linner \Delta, q \rinner
		=
		\linner \Delta_{<k}, q \rinner
		+
		\linner \Delta_{\ge k}, q \rinner
		\ge
		f(q(s))
		+
		\linner \Delta_{\ge k}, q \rinner
		\ge
		c \max \{0, q(s) - \bar{q} \}
		+
		\sum_{a \in A} \Delta(s,a) q(s,a).
	\end{align*}
\end{proof}

\section{Algorithm for MDPs with unknown transition}
\label{sec:unknown_transition_appendix}

In this section, we present the details for our best-of-both-worlds algorithm for MDPs with unknown transitions. Similar to \cite{jin2021best}, our algorithm proceeds in epochs. In each epoch, we execute FTRL using our novel loss estimator and the current empirical estimates of the transitions. At the end of the epoch, we update these empirical estimates. We refer the reader to Algorithm \ref{alg:unknown-transitions} for full details. In this section, we use $\E_{t}[\cdot]$ to denote the conditional expectation $\E[\cdot|\mathcal{F}_{t-1}]$, where $\mathcal{F}_{t-1}$ is the past filtration.

\subsection{Confidence set of the true transition}
In this section, we use the same confidence sets used by prior works \citet{jin2020learning,jin2021best}.

For each epoch $i$, we define the empirical transition $\bar P_i$ as:
\begin{equation}\label{eq:emp-transition-appendix}
    \bar P_i(s'|s,a)=\frac{m_i(s,a,s')}{m_i(s,a)},\; \forall(s,a,s')\in S_k\times A\times S_{k+1}, k=0,\ldots,L-1.
\end{equation}

For each epoch $i$, we define the confidence width $B_i$ as follows:
\begin{equation}\label{eq:conf-wid-appendix}
    B_i(s,a,s')=\min\left\{2\sqrt{\frac{\bar P_i(s'|s,a)\ln\left(\frac{T|S||A|}{\delta}\right)}{m_i(s,a)}+\frac{14\ln\left(\frac{T|S||A|}{\delta}\right)}{m_i(s,a)}},1\right\},
\end{equation}
where $\delta$ is some confidence parameter.

For each epoch $i$, the confidence set $\mathcal{P}_i$ of the true transition is defined as follows:
\begin{equation}\label{eq:conf-set-appendix}
    \mathcal{P}_i=\left\{\widehat P: \left|\widehat P(s'|s,a)-\widehat P_i(s'|s,a)\right|\leq B_i(s,a,s'),\forall(s,a,s')\in S_k\times A\times S_{k+1}, k<L\right\}.
\end{equation}
As shown in Lemma 2 of \cite{jin2020learning}, the true transition $P$ lies in the confidence $\mathcal{P}_i$ for all epoch $i$ with probability at least $1-4\delta$.

\subsection{Loss estimator and Regularizer}
We begin by presenting our novel loss estimator:
\begin{equation}
    \ell_t^u(s,a) = \frac{c_t \cdot \ind_t(s,a) + \left(1 - \pi_t(a \mid s) - c_t\right) \cdot \ind_t(s)\pi_t(a \mid s)}{u_t(s,a)} - \left(1 - \pi_t(a \mid s)\right), \label{eq:loss-estimator-unknown}
\end{equation}
where $u_t(s,a)$ denotes the upper occupancy measure of $(s,a)$ under policy $\pi_t$, and is defined as
\begin{equation}
    u_t(s,a) = \max_{\widehat{P} \in \mathcal{P}_{i(t)}} q^{\widehat{P}, \pi_t}(s,a), \label{eq:UOB}
\end{equation}
where $i(t)$ denotes the epoch to which round $t$ belongs. Note that $u_t(s,a)$ can be efficiently computed using the \textsc{Comp-UOB} procedure proposed in~\citep{jin2020learning}.

To build intuition for why this estimator enables best-of-both-worlds guarantees, we now consider the corresponding loss estimator in the setting with known transitions.
\begin{equation}\label{eq:loss-estimator-known}
    \ell_t^q(s,a)=\frac{c_t \cdot \ind_t(s,a)+(1-\pi_t(a|s)-c_t)\cdot \ind_t(s)\pi_t(a|s)}{q_t(s,a)}-(1-\pi_t(a|s))
\end{equation}
In the unknown transition case under semi-bandit feedback, \cite{jin2021best} considered the loss estimator $\frac{\ell_t(s,a)\cdot \ind_t(s,a)}{u_t(s,a)}$, which ensures that $\ell_t(s,a) - \E_t\left[\frac{\ell_t(s,a)\cdot \ind_t(s,a)}{u_t(s,a)}\right] \geq 0$ whenever $u_t(s,a) \geq q_t(s,a)$. This inequality plays a key role in establishing their best-of-both-worlds result.

In our setting, the role of $\ell_t(s,a)$ is played by the pseudo-loss $\bar{\ell}_t(s,a) := \E_t[\ell_t^q(s,a)]$. When $u_t(s,a) \geq q_t(s,a)$, we show an analogous inequality: $\bar{\ell}_t(s,a) - \E_t[\ell_t^u(s,a)] \geq 0$, which is similarly crucial for our analysis.

We begin by analyzing the pseudo-loss $\bar{\ell}_t(s,a)$ as follows:
\begin{align*}
    \bar{\ell}_t(s,a)
	&
	:=
	\E_t \left[ 
    {\ell}^q_t(s,a)
	\right]
	\\
	&
	=
	\E_t \left[ 
	\frac{c_t \cdot \ind_t(s,a) - c_t\cdot \ind_t(s)\pi_t(a|s)}{q_t(s,a)} 
	\right]
	+
	\E_t \left[ 
	\frac{(1 - \pi_t(a|s)) \ind_t(s)\pi_t(a|s)}{q_t(s,a)} 
	\right]
	- 
	(1 - \pi_t(a|s))
	\\
	&
	=
	Q^{\pi_t}(s,a)
	-
	V^{\pi_t}(s)
	+
	1 - \pi_t(a|s)
	- 
	(1 - \pi_t(a|s))
	\\
	&
	=
	Q^{\pi_t}(s,a)
	-
	V^{\pi_t}(s).
\end{align*}

By Lemma \ref{lem:PDL}, we have $\langle \bar{\ell}_t, q_t - q^* \rangle = \langle \ell_t, q_t - q^* \rangle$, implying that the pseudo-loss $\bar{\ell}_t$ can indeed play the role of $\ell_t$ in our setting.

Next, we compute the expectation of $\ell_t^u(s,a)$ as follows:
$$
\E_t \left[ \ell_t^u(s,a) \right]
=
\frac{q_t(s)}{u_t(s)} \left( Q^{\pi_t}(s,a) - V^{\pi_t}(s) + 1 - \pi_t(a|s) \right)
-
(1 - \pi_t(a|s)).
$$

To analyze this expression, observe that:
\begin{align*}
    Q^{\pi_t}(s,a) - V^{\pi_t}(s)
&=
Q^{\pi_t}(s,a) - \sum_{a'\in A} \pi_t(a' \mid s) Q^{\pi_t}(s,a')\\
&=
(1 - \pi_t(a|s)) Q^{\pi_t}(s,a) - \sum_{a' \ne a} \pi_t(a' \mid s) Q^{\pi_t}(s,a')\\
&\ge - (1 - \pi_t(a|s)), \tag{as $Q^{\pi_t}(s,a)\in[0,1]$}
\end{align*}
and thus:
$$
Q^{\pi_t}(s,a) - V^{\pi_t}(s) + 1 - \pi_t(a|s) \ge 0.
$$

Now, under the condition that $u_t(s) \ge q_t(s)$, we have $\frac{q_t(s)}{u_t(s)} \le 1$, and therefore:
$$
\begin{aligned}
\E_t \left[ \ell_t^u(s,a) \right]
&=
\frac{q_t(s)}{u_t(s)} \left( Q^{\pi_t}(s,a) - V^{\pi_t}(s) + 1 - \pi_t(a|s) \right)
- (1 - \pi_t(a|s)) \\
&\le
\left( Q^{\pi_t}(s,a) - V^{\pi_t}(s) + 1 - \pi_t(a|s) \right)
- (1 - \pi_t(a|s)) \\
&=
Q^{\pi_t}(s,a) - V^{\pi_t}(s) \\
&= \bar{\ell}_t(s,a).
\end{aligned}
$$
That is, $\ell_t^u(s,a)$ is an optimistic estimator of the pseudo-loss $\bar{\ell}_t(s,a)$, which plays a crucial role in our regret analysis.

In addition to bounding the expectation, we also analyze the second moment of the loss estimator. This analysis is facilitated by the careful introduction of the term $(1 - \pi_t(a \mid s))$ into our loss estimator, which yields the following bound:
$$
\begin{aligned}
	\E_t \left[ \ell_t^u(s,a)^2 \right]
	&\lesssim
	\E_t \left[ 
		\frac{1}{u_t(s,a)^2}
		\left(
			c_t^2 \cdot (\ind_t(s,a) - \ind_t(s)\pi_t(a \mid s))^2
			+
			(1 - \pi_t(a \mid s))^2 \cdot \ind_t(s) \pi_t(a \mid s)^2
		\right)
	\right]\\
	&\quad+
	(1 - \pi_t(a \mid s))^2 \\
	&\lesssim
	(1 - \pi_t(a \mid s))
	\left(
		\frac{q_t(s)}{u_t(s)} \cdot \frac{1}{u_t(s,a)}
		+
		1 - \pi_t(a \mid s)
	\right).
\end{aligned}
$$

Using this bound, we obtain a control on the stability term in the regret analysis:
$$
\E_t \left[ 
	\widehat{q}_t(s,a)^{3/2} \cdot \ell_t^u(s,a)^2
\right]
\lesssim
\widehat{q}_t(s,a)^{1/2} (1 - \pi_t(a \mid s)).
$$

This upper bound plays a crucial role in enabling a regret analysis based on self-bounding terms. In particular, it allows us to bypass the loss-shifting technique employed in \cite{jin2021best}, while still controlling the stability term.

To leverage this upper bound, we use the following regularizer in epoch $i$:
\begin{equation}\label{eq:reg-unknown-appendix}
\phi_t(q) = -\frac{1}{\eta_t} \sum_{s \neq s_L} \sum_{a \in A} \sqrt{q(s,a)}
- \beta \sum_{s \neq s_L} \sum_{a \in A} \ln q(s,a),
\end{equation}
where $\eta_t = \frac{1}{\sqrt{t - t_i + 1}}$ and $\beta = 1024L$. The log-barrier term is included to stabilize the updates, following the approach of \cite{jin2020simultaneously}.




\subsection{Main Result}
\begin{theorem}\label{thm:main_bandit-appendix}
In the bandit feedback setting,  Algorithm \ref{alg:unknown-transitions} with $\delta = \frac{1}{T^3}$ and $\iota=\frac{|S||A|T}{\delta}$ guarantees $\Reg_T(\pi^\star)= \tilde{\mathcal{O}}\rbr{  L|S|\sqrt{|A|T} + |S||A| \sqrt{LT} + L^2|S|^3|A|^2 }$ and simultaneously $\Reg_T(\pi^\star)=\order\rbr{U + \sqrt{UC}+ V }$ under Condition~\eqref{eq:self-bounding-MDP},   
	where $V =  L^2|S|^3|A|^2 \ln^2 \iota$ and $U$ is defined as
	\[
	U =  \sum_{s \neq s_L} \sum_{ a\neq \pi^\star(s)} \sbr{ \frac{L^4|S|\ln \iota +  |S||A|\ln^2 \iota}{\gap(s,a)}} + \sbr{\frac{(L^4|S|^2+L^3|S|^2|A|) \ln \iota +  L|S|^2|A| \ln^2 \iota}{\gapmin}}.
	\]
	\label{prop:bobw_bandit_stoc_prop-appendix}

\end{theorem} 

We defer the proof of the above theorem to Appendix \ref{sec:analysis-bobw-unknown}.

\section{Analysis of BOBW with unknown transitions}\label{sec:analysis-bobw-unknown}
For any time-step $t$, let $i(t)$ denote the epoch the time-step $t$ is part of. Let $\E_{t}[\cdot]:=\E[\cdot|\mathcal{F}_{t-1}]$ be the conditional expectation, where $\mathcal{F}_{t-1}$ is the past filtration. Recall the definition of $\ell_t^u$ and $\ell_t^q$ from Eq. \eqref{eq:loss-estimator-unknown} and Eq. \eqref{eq:loss-estimator-known} respectively. Also recall that $\widehat{\ell}_t=\ell_t^u-B_{i(t)}$. Let $\bar \ell_t$ be a pseudo-loss such that $\bar\ell_t(s,a):=\E_t[\ell_t^q(s,a)]=Q^{\pi_t}(s,a)-V^{\pi_t}(s)$. 


Now we define the conditional expectation of $\widehat{\ell}_t$ as follows:
\begin{equation}\label{eq:cond-exp-loss}
    \tilde{\ell}_t(s,a):=\E_t[\widehat{\ell}_t(s,a)]=\frac{q_t(s)}{u_t(s)}
	\left(
		Q^{\pi_t}(s,a)
		-
		V^{\pi_t}(s)
		+
		1 - \pi_t(a|s)
	\right)
	-
	(1 - \pi_t(a|s))-B_{i(t)}(s,a).
\end{equation}

\begin{definition}
    For any policy $\pi$, the estimated state-action and state value functions associated with $\bar{P}_{i(t)}$ and loss function $\tilde{\ell}_t$ are defined as:
    \begin{align*}
    \widetilde{Q}^{\pi}_t(s,a) &= \tilde{\ell}_t(s,a) + \sum_{s' \in \mathcal{S}_{K(s)+1}} \bar{P}_{i(t)}(s'|s,a) \widetilde{V}^{\pi}_t(s'), \quad \forall (s,a) \in (S \setminus \{s_L\}) \times A, \\
    \widetilde{V}^{\pi}_t(s) &= \sum_{a \in A} \pi(a|s) \widetilde{Q}^{\pi}_t(s,a), \quad \forall s \in S, \\
    \widetilde{Q}^{\pi}_t(s_L, a) &= 0, \quad \forall a \in A.
    \end{align*}

    On the other hand, the true state-action and value functions are defined as:
    \begin{align*}
    Q^{\pi}_t(s,a) &= \ell_t(s,a) + \sum_{s' \in \mathcal{S}_{K(s)+1}} P(s'|s,a)V^{\pi}_t(s'), \quad \forall (s,a) \in (S \setminus \{s_L\}) \times A, \\
    V^{\pi}_t(s) &= \sum_{a \in A} \pi(a|s)Q^{\pi}_t(s,a), \quad \forall s \in S, \\
    Q^{\pi}_t(s_L, a) &= 0, \quad \forall a \in A.
\end{align*}
where $P$ denotes the true transition function.

Moreover, we define pseudo state-action and value functions as follows:
\begin{align*}
    \bar{Q}^{\pi}_t(s,a) &= \bar\ell_t(s,a) + \sum_{s' \in \mathcal{S}_{K(s)+1}} P(s'|s,a)\bar{V}^{\pi}_t(s'), \quad \forall (s,a) \in (S \setminus \{s_L\}) \times A, \\
    \bar{V}^{\pi}_t(s) &= \sum_{a \in A} \pi(a|s)\bar{Q}^{\pi}_t(s,a), \quad \forall s \in S, \\
    \bar{Q}^{\pi}_t(s_L, a) &= 0, \quad \forall a \in A.
\end{align*}

\end{definition}

Let $\mathcal{A}$ be the event that $P\in \mathcal{P}_i$ for all epochs $i\geq 1$. Moreover, we also define $\mathcal{A}_i$ to be the event $P \in \mathcal{P}_i$.Note that the value of $\mathbbm{1}\{\mathcal{A}_i\}$ gets determined based on observations prior to epoch $i$ only. Let $\iota = \frac{T|\mathcal{S}||\mathcal{A}|}{\delta}$ and let $\delta=\frac{1}{T^3} \in (0,1)$.

 We decompose the regret against policy $\pi$ as follows:
\begin{align}
    \mathrm{Reg}(\pi) &= \E \left[ \sum_{t=1}^T V_t^{\pi_t}(s_0) - V_t^{\pi}(s_0) \right] \notag \\
    &= \E \left[ \sum_{t=1}^T \bar{V}_t^{\pi_t}(s_0) - \bar{V}_t^{\pi}(s_0) \right] \tag{due to Lemma \ref{lem:PDL}} \\
    &= \E\Biggsbr{\underbrace{\sum_{t=1}^{T} \bar{V}^{\pi_t}_t(s_0) - \widetilde{V}^{\pi_t}_t(s_0)}_{\textsc{Err}_1 }}  + \E\Biggsbr{\underbrace{\sum_{t=1}^{T} \widetilde{V}^{\pi_t}_t(s_0) - \widetilde{V}^{\pi}_t(s_0) }_{\textsc{EstReg} }}  + \E\Biggsbr{\underbrace{\sum_{t=1}^{T}\widetilde{V}^{\pi}_t(s_0) - \bar{V}^{\pi}_t(s_0) }_{\textsc{Err}_2}}. \label{st:regret-decomp-49}
\end{align}

\noindent Note that, the second term (restated below) is controlled by the FTRL process.
\begin{align}
    \E[\textsc{EstReg}] = \E \left[ \sum_{t=1}^T \left\langle q^{\bar{P}_{i(t)},\pi_t} - q^{\bar{P}_{i(t)},\pi}, \tilde{\ell}_t \right\rangle \right]
    = \E \left[ \sum_{t=1}^T \left\langle q^{\bar{P}_{i(t)},\pi_t} - q^{\bar{P}_{i(t)},\pi}, \widehat{\ell}_t \right\rangle \right].
\end{align}
 
\subsection{Auxiliary lemmas}
We often use the following lemma to handle the small-probability event $\calA^c$ while taking the expectation.
\begin{lemma}[\cite{jin2021best}]\label{lem:exp_high_prob_bound_cond}  Suppose that a random variable $X$ satisfies the following conditions:
	\begin{itemize}
		\item Conditioning on event $\calE$, $ X < Y$ where $Y>0$ is another random variable;
		\item $ X < C $ holds where $C$ is another random variable which ensures $\E\sbr{C|\calE^c } \leq D$ for some fixed $D\in \mathbb{R}_{+}$. 
	\end{itemize}
	Then, we have
	\begin{align*}
	\E\sbr{ X } \leq D \cdot \Pr\sbr{ \calE^c  } + \E\sbr{Y}.   
	\end{align*}
\end{lemma}

We next restate the performance difference lemma.
\begin{lemma}[Performance difference lemma]\label{lem:perf-diff}
Suppose $\bar{\ell}$ is defined by $\bar{\ell}(s, a) = Q^{\pi}(s, a; \ell) - V^{\pi}(s; \ell)$ for some $\pi \in \Pi$ and for all $s \in S$ and $a \in A$. We then have
\begin{align}
    V^{\pi'}(s; \bar{\ell}) &= V^{\pi'}(s; \ell) - V^{\pi}(s; \ell), \quad
    Q^{\pi'}(s, a; \bar{\ell}) = Q^{\pi'}(s, a; \ell) - V^{\pi}(s; \ell),
\end{align}
for any $\pi' \in \Pi$, $s \in S$ and $a \in A$.
\end{lemma}

We immediately get following corollary.
\begin{corollary}\label{cor:barV<=1}
    $\forall (s,a)\in S\times A$, we have $-1\leq \bar{V}^\pi_t(s)\leq 1$ and $-1\leq \bar{Q}^\pi_t(s,a)\leq 1$.
\end{corollary}

We next state the following lemma.
\begin{lemma}\label{lem:induction-probability}
    If event $\mathcal{A}$ holds, then $\sum_{s' \in S_{k(s)+1}} \rbr{ \bar{P}_{i(t)}(s'|s,a)  - P(s'|s,a) } \bar{V}_t^{\pi}(s') - B_{i(t)}(s,a)\leq 0$.
\end{lemma}
\begin{proof}
    When $B_{i(t)}(s,a) = 2$, we have 
	\begin{align*}
    &\sum_{s' \in S_{k(s)+1}} \rbr{ \bar{P}_{i(t)}(s'|s,a)  - P(s'|s,a) }\bar{V}_t^{\pi}(s') - B_{i(t)}(s,a)\\
     &\leq \sum_{s' \in S_{k(s)+1}} \bar{P}_{i(t)}(s'|s,a) \cdot 1+\sum_{s' \in S_{k(s)+1}} P_{i(t)}(s'|s,a) \cdot 1 - 2 = 0 ,
	\end{align*}
	where the inequality follows from the fact $-1\leq \bar{V}_t^{\pi}(s') \leq 1 $. 
	
	On the other hand, when $\sum_{s' \in S_{k(s)+1}}B_{i(t)}(s,a,s') = B_{i(t)}(s,a)$, we have 
	\begin{align*}
	& \sum_{s' \in S_{k(s)+1}} \rbr{ \bar{P}_{i(t)}(s'|s,a)  - P(s'|s,a) } \bar{V}_t^{\pi}(s') - B_{i(t)}(s,a)  \\
	& \leq \sum_{s' \in S_{k(s)+1}}  B_{i(t)}(s,a,s') \cdot 1 - B_{i(t)}(s,a) = 0,
	\end{align*}
	where the second line follows from the definition of event $\calA$. 
	 
\end{proof}
We next state the following proposition
\begin{proposition}\label{prop:pseudo-loss}
    For all $(s,a)\in S\times A$, we have $0\leq Q^{\pi_t}(s,a)
	-
	V^{\pi_t}(s)
	+
	1 - \pi_t(a|s)
	\leq
	2$.
\end{proposition}
\begin{proof}
     As we have
\begin{align}
	Q^{\pi_t}(s,a)
	-
	V^{\pi_t}(s)
	&
	=
	Q^{\pi_t}(s,a)
	-
	\sum_{a'} \pi_t(a'|s) Q^{\pi_t}(s,a')
	\\
	&
	=
	(1 - \pi_t(a|s))Q^{\pi_t}(s,a)
	-
	\sum_{a' \neq a} \pi_t(a'|s) Q^{\pi_t}(s,a')
	\ge
	- (1 - \pi_t(a|s)),
\end{align}
we get
$
	Q^{\pi_t}(s,a)
	-
	V^{\pi_t}(s)
	+
	1 - \pi_t(a|s)
	\ge
	0
$.

As $0\leq Q^{\pi_t}(s,a)\leq 1$, we also have $Q^{\pi_t}(s,a)
	-
	V^{\pi_t}(s)
	+
	1 - \pi_t(a|s)
	\leq Q^{\pi_t}(s,a)+1\leq
	2$.
\end{proof}
We next state the following proposition.
\begin{proposition}\label{prop:tilde<=bar-ell}
    If event $\mathcal{A}$ holds, $\tilde{\ell}_t(s,a)\leq \bar{\ell}_t(s,a)-B_{i(t)}(s,a)$ for all $(s,a)\in S\times A$.
\end{proposition}
\begin{proof}
    The following holds by Proposition \ref{prop:pseudo-loss}, given that the event $\mathcal{A}$ occurs:
    \begin{align*}
        \tilde{\ell}_t(s,a)&=\frac{q_t(s)}{u_t(s)}
	\left(
		Q^{\pi_t}(s,a)
		-
		V^{\pi_t}(s)
		+
		1 - \pi_t(a|s)
	\right)
	-
	(1 - \pi_t(a|s))-B_{i(t)}(s,a)\\
        &\leq \left(
		Q^{\pi_t}(s,a)
		-
		V^{\pi_t}(s)
		+
		1 - \pi_t(a|s)
	\right)
	-
	(1 - \pi_t(a|s))-B_{i(t)}(s,a)\\
    &=\bar{\ell}_t(s,a)-B_{i(t)}(s,a).
    \end{align*}
\end{proof}

We next state the following proposition.
\begin{proposition}\label{prop:-3<=tilde-ell<=1}
    If event $\mathcal{A}$ holds, $-3\leq \tilde\ell_t(s,a)\leq 1$ for all $(s,a)\in S\times A$.
\end{proposition}
\begin{proof}
    Due to Proposition \ref{prop:tilde<=bar-ell} and the total loss of any trajectory is between $0$ and $1$, we have $\tilde \ell_t(s,a)\leq \bar{\ell}_t(s,a)-B_{i(t)}(s,a)\leq  Q^{\pi_t}(s,a)\leq 1$. On the otherhand, due to Proposition \ref{prop:pseudo-loss}:
    \begin{align*}
        \tilde{\ell}_t(s,a)&=\frac{q_t(s)}{u_t(s)}
	\left(
		Q^{\pi_t}(s,a)
		-
		V^{\pi_t}(s)
		+
		1 - \pi_t(a|s)
	\right)
	-
	(1 - \pi_t(a|s))-B_{i(t)}(s,a)\\
        &\geq 
	-(1 - \pi_t(a|s))-B_{i(t)}(s,a)\\
    &\geq -3.
    \end{align*}
\end{proof}

We next state the following lemma.
\begin{lemma}\label{lem:tQ<=bQ-c.1.1}
    If event $\mathcal{A}$ holds, the following holds:
    \[\tilde Q_t^\pi(s,a)\leq \bar{Q}_t^\pi,\forall(s,a)\in S\times A, t\in [T].\]
    Specifically, we have:
    \[ \left\langle q^{\bar P_{i(t),\pi}},\tilde \ell_t\right\rangle= \tilde V_t^\pi(s_0)\leq \bar V_t^\pi(s_0)=\left\langle q^{P,\pi},\bar\ell_t\right\rangle.\]
\end{lemma}
\begin{proof}
We prove this result via a backward induction from layer $L$ to layer $0$. 
	
	\textbf{Base case:} for $s_L$, $\widetilde{Q}_t^{\pi}(s,a) = \bar{Q}_t^\pi(s,a) = 0$ holds always. 
	
	\textbf{Induction step:} Assume that $\widetilde{Q}_t^{\pi}(s,a)\leq Q_t^\pi(s,a)$ holds for all states $s$ with $k(s) > h$. Then, for any state $s$ with $k(s) = h$, we have 
	\begin{align*} 
	\widetilde{Q}_t^{\pi}(s,a) & = \bar{\ell}_t(s,a)+ \sum_{s' \in S_{k(s)+1}} \bar{P}_{i(t)}(s'|s,a) \widetilde{V}_t^{\pi}(s') - B_{i(t)}(s,a)&& (\text{due to Proposition \ref{prop:tilde<=bar-ell}}   ) \\
	& \leq \bar{\ell}_t(s,a) + \sum_{s' \in S_{k(s)+1}} \bar{P}_{i(t)}(s'|s,a) \bar{V}_t^{\pi}(s') - B_{i(t)}(s,a) && (\text{induction  hypothesis}) \\
	& \leq \bar{\ell}_t(s,a) + \sum_{s' \in S_{k(s)+1}} P(s'|s,a) \bar{V}_t^{\pi}(s') \\
	& \quad + \sum_{s' \in S_{k(s)+1}} \rbr{ \bar{P}_{i(t)}(s'|s,a)  - P(s'|s,a) } \bar{V}_t^{\pi}(s') - B_{i(t)}(s,a) \\
	& \leq \bar{\ell}_t(s,a) + \sum_{s' \in S_{k(s)+1}} P(s'|s,a) \bar{V}_t^{\pi}(s') \tag{due to Lemma \ref{lem:induction-probability}} \\
	& \leq  \bar{\ell}_t(s,a) + \sum_{s' \in S_{k(s)+1}} P(s'|s,a) \bar{V}^{\pi}(s') 
	= \bar{Q}_t^\pi(s,a).
	\end{align*}
	This concludes the induction.    
\end{proof}
The following lemma follows directly from Lemma C.1.2 in \cite{jin2021best}.
\begin{lemma}[\cite{jin2020simultaneously}]\label{lem:ut-lb-c.1.2}
    Algorithm \ref{alg:unknown-transitions} ensures $u_t(s)\geq \frac{1}{|S|t}$ for all $t$ and $s$.
\end{lemma}

\begin{lemma}\label{lem:hatl-upper-c.1.3}
    Algorithm \ref{alg:unknown-transitions} ensures the following:
    \[\left|\widehat\ell_t(s,a)\right|\leq 3+\frac{\ind_t(s,a)+\ind_t(s)\pi_t(a|s)}{q_t(s,a)}\cdot |S|t.\]
    We also have:
    \[\E\left[\frac{\ind_t(s,a)+\ind_t(s)\pi_t(a|s)}{q_t(s,a)}\middle|\mathcal{A}_{i(t)}\right]=\E\left[\frac{\ind_t(s,a)+\ind_t(s)\pi_t(a|s)}{q_t(s,a)}\middle|\mathcal{A}_{i(t)}^c\right]=2.\]
\end{lemma}
\begin{proof}
    Due to Lemma \ref{lem:ut-lb-c.1.2}, we have the following:
    \begin{align*}
        \left|\widehat\ell_t(s,a)\right|&\leq 3+ \frac{\ind_t(s,a)+\ind_t(s)\pi_t(a|s)}{u_t(s)\cdot \pi_t(a|s)}\\
        &\leq 3+ \frac{\ind_t(s,a)+\ind_t(s)\pi_t(a|s)}{q_t(s)\cdot \pi_t(a|s)}\cdot |S|t\\
        &= 3+\frac{\ind_t(s,a)+\ind_t(s)\pi_t(a|s)}{q_t(s,a)}\cdot |S|t.
    \end{align*}

    Next, we have:
    \[\E\left[\frac{\ind_t(s,a)+\ind_t(s)\pi_t(a|s)}{q_t(s,a)}\middle|\mathcal{A}_{i(t)}\right]=\E\left[\E_t\left[\frac{\ind_t(s,a)+\ind_t(s)\pi_t(a|s)}{q_t(s,a)}\right]\middle|\mathcal{A}_{i(t)}\right]=\E\left[2\middle|\mathcal{A}_{i(t)}\right]=2\]

    Similarly, we can show that $\E\left[\frac{\ind_t(s,a)+\ind_t(s)\pi_t(a|s)}{q_t(s,a)}\middle|\mathcal{A}_{i(t)}^c\right]=2$.
\end{proof}

\begin{lemma}\label{lem:til-l-up-c.1.4}
    Algorithm \ref{alg:unknown-transitions} ensures the following:
    \[\left|\tilde \ell_t(s,a)\right|\leq 6|S|t,\; \forall(s,a)\in S\times A, t\in [T].\]
\end{lemma}
\begin{proof}
    Due to Eq. \eqref{eq:cond-exp-loss}, we have the following:
   \begin{align*}
       \left|\tilde \ell_t(s,a)\right|&=\left|\frac{q_t(s)}{u_t(s)}
	\left(
		Q^{\pi_t}(s,a)
		-
		V^{\pi_t}(s)
		+
		1 - \pi_t(a|s)
	\right)
	-
	(1 - \pi_t(a|s))-B_{i(t)}(s,a)\right|\\
       &\leq \frac{2q_t(s)}{u_t(s)}+3\tag{due to Proposition \ref{prop:pseudo-loss}}\\
       &\leq 2|S|t+3 \tag{due to Lemma \ref{lem:ut-lb-c.1.2}}\\
       &\leq 6|S|t.
   \end{align*}
\end{proof}
We immediately get the following corollary.
\begin{corollary}\label{cor:til-Q-c.1.5}
     Algorithm \ref{alg:unknown-transitions} ensures the following:
    \[\left|\tilde Q_t^\pi(s,a)\right|\leq 6L|S|t,\; \forall(s,a)\in S\times A, t\in [T].\]
\end{corollary}

Let $\phi_H(q)=-\sum_{s\neq s_L}\sum_{a\in A}\sqrt{q(s,a)}$ and $\phi_L(q)=-\beta\sum_{s\neq s_L}\sum_{a\in A}\ln q(s,a)$. Recall that $\phi_t(q)=\phi_H(q)+\phi_L(q)$ and $\beta=1024L$. Now we prove the following proposition.
\begin{proposition}\label{prop:hessian}
    If event $\mathcal{A}$ holds, then $||\widehat\ell_t||_{(\nabla^2\phi_t(\widehat q_t))^{-1}}\leq \frac{1}{8}$.
\end{proposition}
\begin{proof}\allowdisplaybreaks
    We have the following:
    \begin{align*}
        ||\widehat\ell_t||_{(\nabla^2\phi_t(\widehat q_t))^{-1}}^2&\leq ||\widehat\ell_t||_{(\nabla^2\phi_L(\widehat q_t))^{-1}}^2\tag{as $\nabla^2\phi_L(\widehat q_t)\preceq \nabla^2\phi_t(\widehat q_t)$}\\
        &\leq\frac{1}{\beta}\cdot\sum_{s\neq s_L}\sum_{a\in A}\left(\frac{2(\ind_t(s,a)+\ind_t(s)\pi(s,a))^2}{u_t(s,a)^2}+2(-(1-\pi_t(a|s))+B_i(s,a))^2\right)\cdot \widehat q_t(s,a)^2\tag{as $(a+b)^2\leq 2a^2+2b^2$}\\
        &\leq\frac{1}{\beta}\cdot\sum_{s\neq s_L}\sum_{a\in A}\left(\frac{4(\ind_t(s,a)+\ind_t(s)\pi(s,a))}{u_t(s,a)^2}+8\right)\cdot \widehat q_t(s,a)^2\tag{as $\ind_t(s),\ind_t(s,a),\pi(s,a)\in[0,1]$ and $-(1-\pi_t(a|s))+B_i(s,a)\in [-1,2]$}\\
        &\leq \frac{1}{\beta}\cdot\sum_{s\neq s_L}\sum_{a\in A}\left(4(\ind_t(s,a)+\ind_t(s)\pi(s,a))+8\widehat q_t(s,a)\right)\tag{as $u_t(s,a)\leq \widehat q_t(s,a)$}\\
        &= \frac{1}{\beta}\cdot\sum_{s\neq s_L}\left(8\ind_t(s)+8 \widehat q_t(s)\right)\\
        &= \frac{1}{\beta}\cdot (16L)\\
        &=\frac{1}{64}.
    \end{align*}
\end{proof}

We now state the following lemma, which follows from arguments identical to those in Lemma C.1.8 of \cite{jin2021best}, and provides an upper bound for $\E\sbr{\sum\limits_{t=1}^{T}\sum\limits_{s \neq s_L} \sum\limits_{a \in A} \whatq_t(s,a) \cdot B_{i(t)}(s,a)^2}$.
\begin{lemma} \label{lem:bobw_tasallis_bound_extra} Algorithm \ref{alg:unknown-transitions} ensures the following: 
	\begin{equation}
	\E\sbr{ \sum_{t=1}^{T} \sum_{s \neq s_L} \sum_{a \in A} \whatq_t(s,a) \cdot  B_{i(t)}(s,a)^2   } = \order\rbr{ L^2|S|^3|A|^2 \ln^2 \iota  + |S||A| T \cdot \delta }.
	\end{equation}
\end{lemma}

Finally, we state the following lemma on the learning rates and the number of epochs.
\begin{lemma}[\cite{jin2021best}] \label{lem:bobw_lr_properties} According to the design of the learning rate $\eta_t = \frac{1}{\sqrt{t-t_{i(t)}+1}}$, the following inequalities hold:
	\begin{equation}
	\sum_{t=1}^{T} \eta_t^2  \leq \order\rbr{ |S||A|\log^2 T}, \label{eq:bobw_bandit_lr_1}
	\end{equation}
	\begin{equation}
	\sum_{t=1}^{T} \eta_t  \leq \order\rbr{ \sqrt{ |S||A|T \log T }}.  \label{eq:bobw_bandit_lr_2}
	\end{equation}
    \begin{equation}
        \frac{1}{\eta_t} - \frac{1}{\eta_{t-1}} \leq \eta_t\; \quad\forall t\geq 2
    \end{equation}

Moreover, Algorithm \ref{alg:unknown-transitions} ensures that $N\leq 4|S||A|(\log T+1)$ where $N$ is the number of epochs.
\end{lemma}

\subsection{Technical lemmas to analyze \textsc{EstReg}}
We defined the estimated regret in each epoch $i$ as follows:
\[\textsc{EstReg}_i(\pi)=\E\left[\sum_{t=t_i}^{t_{i+1}-1}\left\langle q^{\bar P_i,\pi_t}-q^{\bar P_i,\pi},\widehat\ell_t\right\rangle\right]=\E\left[\sum_{t=t_i}^{t_{i+1}-1}\left\langle \widehat q_t-q^{\bar P_i,\pi},\widehat\ell_t\right\rangle\right].\]

\begin{lemma}\label{lem:bobw_bandit_tsallis_reg}
 With $\beta = 1024L$, for any epoch $i$, Algorithm \ref{alg:unknown-transitions} ensures 
	\begin{equation}
	\begin{aligned}
	\textsc{EstReg}_i(\pi) 
	& \leq \order\rbr{ \E\sbr{ \sqrt{L|S||A|} \cdot \sum_{t=t_i}^{t_{i+1}-1} \eta_t } +  L|S||A|\log T + \delta \cdot \E\sbr{ L|S|T\rbr{ t_{i+1} - t_i   } }  }
	\end{aligned}
	\label{eq:bobw_bandit_tsallis_reg_worstcase}
	\end{equation}
	for any policy $\pi$, and simultaneously
	\begin{equation}
	\begin{aligned}
\E\sbr{ \textsc{EstReg}_i(\pi) } &  \leq \order\left(\E\sbr{\sqrt{L|S|}  \sum_{t=t_i}^{t_{i+1}-1} \eta_t \cdot \sqrt{\sum_{s  \neq s_L}\sum_{a\neq \pi(s)} \whatq_t(s,a)  }}\right)\\
	&\quad +  \order\left(\E\sbr{\sum_{t=t_i}^{t_{i+1}-1}  \eta_t \cdot \sum_{s \neq s_L}\sum_{a\neq \pi(s)}\sqrt{\whatq_t(s,a)}}\right)\\
        & \quad + \order\left(\sum_{t=t_i}^{t_{i+1}-1} \eta_t \cdot \sum_{s \neq s_L} \sum_{a\in A} \whatq_t(s,a)^{\nicefrac{3}{2}} \cdot B_{i(t)}(s,a)^2\right)\\
	& \quad + \order\rbr{L|S||A|\log T+ \delta \cdot \E\sbr{ L|S|T\rbr{ t_{i+1} - t_i   } }  }\\
	\end{aligned}
	\label{eq:bobw_bandit_tsallis_reg_adaptive}
	\end{equation}
	for any deterministic policy $\pi: S\rightarrow A$.
\end{lemma}
\begin{proof}
    If the event $\mathcal{A}_i$ does not hold, we have the following:
    \begin{align*}
	 \sum_{t=t_i}^{t_{i+1}-1} \inner{ \whatq_t - q, \hatl_t  } &\leq \sum_{t=t_i}^{t_{i+1}-1} \sum_{s \neq s_L} \sum_{a\in A} \rbr{\whatq_t(s,a) + q(s,a)} \cdot  \abr{\hatl_t(s,a)}  \\
	& \leq \sum_{t=t_i}^{t_{i+1}-1} \sum_{s \neq s_L} \sum_{a\in A} \rbr{ \whatq_t(s,a) + q(s,a) }\cdot \rbr{ 3 +  \frac{\ind_t(s,a)+\ind_t(s)\pi_t(a|s)}{u_t(s,a) }\cdot |S|t } \tag{due to Lemma \ref{lem:hatl-upper-c.1.3}} \\
	&\leq |S|T \cdot \sum_{t=t_i}^{t_{i+1}-1} \sum_{s \neq s_L} \sum_{a\in A} \rbr{ \whatq_t(s,a) + q(s,a) } \cdot \frac{\ind_t(s,a)+\ind_t(s)\pi_t(a|s)}{q_t(s,a) }\\
    & \quad +6L \cdot \rbr{ t_{i+1} - t_i   }.
    \end{align*} 

    Due to the second part of Lemma \ref{lem:hatl-upper-c.1.3}, we have the following:
    \begin{align*}
        &\E\sbr{ \sum_{t=t_i}^{t_{i+1}-1}\inner{ \whatq_t - q, \hatl_t  } \middle|\mathcal{A}_i^c}\\
        &\leq \E\sbr{  ( 6L + 4L|S|T) \cdot \rbr{ t_{i+1} - t_i   }   \middle| \mathcal{A}_i^c} \\
         &\leq \order\Bigrbr{  \E\sbr{   L|S|T \cdot \rbr{ t_{i+1} - t_i   } \middle| \mathcal{A}_i^c}    }.
    \end{align*}

Recall that $\phi_H(q)=-\sum_{s\neq s_L}\sum_{a\in A}\sqrt{q(s,a)}$. Also recall that $\widehat Q_t(s,a)=\widehat{\ell}_t(s,a)+\E_{s'\sim \bar P(\cdot|s,a)}[\widehat V_{t}(s')]$ and $\widehat V_t(s)=\E_{a\sim \pi_t(\cdot|s)}[\widehat Q_t(s,a)]$ (with $\widehat V_L(s_L)=0$). Due to Proposition \ref{prop:hessian}, we get the following by using the same argument as \citep[Lemma~5]{jin2020simultaneously}:
\begin{align}
    &\sum_{t=t_i}^{t_{i+1}-1} \inner{ \whatq_t - q, \hatl_t  }\nonumber\\
    &= \order\rbr{L|S||A|\log T} + 
	 \sum_{t=t_i+1}^{t_{i+1}-1}\rbr{ \frac{1}{\eta_t} - \frac{1}{\eta_{t-1}} } \rbr{ \phi_{H}(q) - \phi_H(\whatq_t)  } \nonumber \\
	&\quad + 8 \sum_{t=t_i}^{t_{i+1}-1} \eta_t \sum_{s \neq s_L} \sum_{a\in A} \whatq_t(s,a)^{\nicefrac{3}{2}} \hatl_t(s,a)^2.  \label{eq:reg_under_A}
\end{align}

Now, we condition on the event $\mathcal{A}_i$. Recall the definition of $\ell_t^u$ from Eq. \eqref{eq:loss-estimator-unknown}. Now we have the following:
\begin{align}
    \widehat{q}_t(s,a)^{3/2}\ell_t^u(s,a)^2&\leq \frac{4\widehat{q}_t(s,a)^{3/2}}{u_t(s,a)^2}\left(c_t^2 \cdot(\ind_t(s,a) - \ind_t(s)\pi_t(a|s))^2+(1 - \pi_t(a|s))^2 \ind_t(s)\pi_t(a|s)^2\right)\nonumber\\
    &\quad +2\widehat{q}_t(s,a)^{3/2}(1 - \pi_t(a|s))^2 \tag{as $(a+b)^2\leq 2a^2+2b^2$}\nonumber\\
    &\leq \frac{4\widehat{q}_t(s,a)^{1/2}}{q_t(s,a)}\left((\ind_t(s,a) - \ind_t(s)\pi_t(a|s))^2+(1 - \pi_t(a|s)) \ind_t(s)\pi_t(a|s)\right)\nonumber\\
    &\quad +2\widehat{q}_t(s,a)^{3/2}(1 - \pi_t(a|s))^2, \label{eq1:stability-unknown}
\end{align}
where we get the last inequality as as $u_t(s)\leq q_t(s)$ and $u_t(s,a)\leq \widehat{q}_t(s,a)$.

As $\E_{t}[(\ind_t(s,a) - \ind_t(s)\pi_t(a|s))^2]=q(s)(1-\pi_t(a|s))\pi_t^2(a|s)+q(s)\pi_t(a|s)(1-\pi_t(a|s))^2$ and $\E_{t}[\ind_t(s)]=q_t(s)$, we have the following:
\begin{align}
    &\E_t\left[\frac{4\widehat{q}_t(s,a)^{1/2}}{q_t(s,a)}\left((\ind_t(s,a) - \ind_t(s)\pi_t(a|s))^2+(1 - \pi_t(a|s)) \ind_t(s)\pi_t(a|s)\right)+2\widehat{q}_t(s,a)^{3/2}(1 - \pi_t(a|s))^2\right]\nonumber\\
    &\leq\frac{12\widehat{q}_t(s,a)^{2/2}\cdot q_t(s)\cdot(1-\pi_t(a|s))\cdot \pi_t(a|s)}{q_t(s,a)}+2\widehat{q}_t(s,a)^{3/2}(1 - \pi_t(a|s))^2\nonumber\\
    &\leq 14\widehat{q}_t(s,a)^{1/2}\cdot(1-\pi_t(a|s)).\label{eq2:stability-unknown}
\end{align}
We now proceed to prove Eq. \eqref{eq:bobw_bandit_tsallis_reg_worstcase} and Eq. \eqref{eq:bobw_bandit_tsallis_reg_adaptive}.

\paragraph{Proving Eq. \eqref{eq:bobw_bandit_tsallis_reg_worstcase}}

In this case, we consider the second term inside the minimum in Eq. \eqref{eq:reg_under_A}, and derive a straightforward upper bound to $\phi_{H}(q) - \phi_H(\whatq_t)$ using the Cauchy-Schwarz inequality, yielding $\phi_{H}(q) - \phi_H(\whatq_t) \leq \sum_{s \neq s_L}\sum_{a \in A} \sqrt{\whatq_t(s,a)} \leq \sqrt{L|S||A|}$. This gives
	\begin{align*}
	&\sum_{t=t_i}^{t_{i+1}-1} \inner{ \whatq_t - q, \hatl_t  } \\
	& \leq  \order\rbr{L|S||A|\log T}+ \sqrt{L|S||A|} \cdot \sum_{t=t_i}^{t_{i+1}-1} \eta_t + 8 \sum_{t=t_i}^{t_{i+1}-1} \eta_t \cdot \sum_{s \neq s_L} \sum_{a\in A} \whatq_t(s,a)^{\nicefrac{3}{2}}\hatl_t(s,a)^2  \tag{$\frac{1}{\eta_t} - \frac{1}{\eta_{t-1}} \leq \eta_t$ since $\frac{1}{\eta_t} = \sqrt{t - t_i +1}$}\\
	& \leq \order\rbr{L|S||A|\log T} + \sqrt{L|S||A|} \cdot \sum_{t=t_i}^{t_{i+1}-1} \eta_t +16 \sum_{t=t_i}^{t_{i+1}-1} \eta_t \cdot \sum_{s \neq s_L} \sum_{a\in A} \whatq_t(s,a)^{\nicefrac{3}{2}}\ell_t^u(s,a)^2 \\
	& \quad + 16 \sum_{t=t_i}^{t_{i+1}-1} \eta_t \cdot \sum_{s \neq s_L} \sum_{a\in A} \whatq_t(s,a)^{\nicefrac{3}{2}} \cdot B_{i(t)}(s,a)^2   \\
	& \leq \order\rbr{L|S||A|\log T + \sqrt{L|S||A|} \cdot \sum_{t=t_i}^{t_{i+1}-1} \eta_t +\sum_{t=t_i}^{t_{i+1}-1} \eta_t \cdot \sum_{s \neq s_L} \sum_{a\in A} \whatq_t(s,a)^{\nicefrac{3}{2}}\ell_t^u(s,a)^2}.\tag{as $B_{i(t)}(s,a)^2\leq 4$}
	\end{align*}
	
	Therefore, by Lemma \ref{lem:exp_high_prob_bound_cond}, Eq. \eqref{eq1:stability-unknown}, Eq. \eqref{eq2:stability-unknown}, and tower rule, we have for any policy $\pi$ that, 
	\begin{align*}
	\E\sbr{ \textsc{EstReg}_i(\pi) } & \leq  \mathcal{O}\left(\E\sbr{\sqrt{L|S||A|} \cdot \sum_{t=t_i}^{t_{i+1}-1} \eta_t + \sum_{t=t_i}^{t_{i+1}-1} \eta_t \cdot \sum_{s \neq s_L} \sum_{a\in A} \widehat{q}_t(s,a)^{1/2}\cdot(1-\pi_t(a|s))}\right) \\
	& \quad + \order\rbr{L|S||A|\log T + \delta \cdot \E\sbr{ L|S|T\rbr{ t_{i+1} - t_i   } }  } \\
	& \leq  \order\rbr{ \E\sbr{ \sqrt{L|S||A|} \cdot \sum_{t=t_i}^{t_{i+1}-1} \eta_t } +  L|S||A|\log T + \delta \cdot \E\sbr{ L|S|T\rbr{ t_{i+1} - t_i   } }  },
	\end{align*}
	where the second step follows from $\sum_{s \neq s_L} \sum_{a \in A} \sqrt{\whatq_t(s,a)} \leq \sqrt{L|S||A|}$. This completes the proof of Eq. \eqref{eq:bobw_bandit_tsallis_reg_worstcase}.

\paragraph{Proving Eq.\eqref{eq:bobw_bandit_tsallis_reg_adaptive}}
	
	In this case, note that since $\pi$ is a deterministic policy, it follows that
	\begin{align*}
	  \phi_{H}(q) - \phi_H(\whatq_t)  
	&= \sum_{s\neq s_L} \sqrt{\whatq_t(s)}\rbr{\sum_{a\in A}\sqrt{\pi_t(a|s)} - 1} + \sum_{s\neq s_L}\rbr{\sqrt{\whatq_t(s)} - \sqrt{q(s)}}.
	\end{align*}
	By applying \citep[Lemma 16]{jin2020simultaneously} with $\alpha = 0$ to upper bound the first term, and using \citep[Lemma 19]{jin2020simultaneously} to upper bound the second term, we arrive at
	\begin{align*}
	\phi_{H}(q) - \phi_H(\whatq_t)  
	=
	\sum_{s \neq s_L}\sum_{a\neq \pi(s)}\sqrt{\whatq_t(s,a)}
	+ \sqrt{L|S| \sum_{s  \neq s_L}\sum_{a\neq \pi(s)} \whatq_t(s,a)  }.
	\end{align*}
	
	Therefore, considering the Eq. \eqref{eq:reg_under_A} and using the inequalities $\frac{1}{\eta_t} - \frac{1}{\eta_{t-1}} \leq \eta_t$ from Lemma \ref{lem:bobw_lr_properties} and $(a+b)^2\leq 2a^2+2b^2$, we have
	\begin{equation}
	\begin{aligned}
	\sum_{t=t_i}^{t_{i+1}-1} \inner{ \whatq_t - q, \hatl_t  }  & \leq \sqrt{L|S|}  \sum_{t=t_i}^{t_{i+1}-1} \eta_t \cdot \sqrt{\sum_{s  \neq s_L}\sum_{a\neq \pi(s)} \whatq_t(s,a)  } \\
	&\quad +  \sum_{t=t_i}^{t_{i+1}-1}  \eta_t \cdot \sum_{s \neq s_L}\sum_{a\neq \pi(s)}\sqrt{\whatq_t(s,a)}\\
	& \quad  + 16 \sum_{t=t_i}^{t_{i+1}-1} \eta_t \cdot \sum_{s \neq s_L} \sum_{a\in A} \whatq_t(s,a)^{\nicefrac{3}{2}}\ell_t^u(s,a)^2 \\
        & \quad + 16 \sum_{t=t_i}^{t_{i+1}-1} \eta_t \cdot \sum_{s \neq s_L} \sum_{a\in A} \whatq_t(s,a)^{\nicefrac{3}{2}} \cdot B_{i(t)}(s,a)^2\\
	& \quad + \order\rbr{L|S||A|\log T}. 
	\end{aligned}
	\label{eq:tsallis_adaptive_regret_decomp}
	\end{equation}

Next, observe that for a deterministic policy $\pi$, we have:
\begin{align}
\sum_{s \neq s_L}\sum_{a \in A} \sqrt{ \whatq_t(s,a) }\rbr{ 1 - \pi_t(a|s) } &\leq \sum_{s \neq s_L}\sum_{a \neq \pi(s)} \sqrt{ \whatq_t(s,a) }
+ \sum_{s \neq s_L} \sqrt{ \whatq_t(s) }\rbr{ 1 - \pi_t(\pi(s)|s) } \nonumber\\
&=\sum_{s \neq s_L}\sum_{a \neq \pi(s)} \sqrt{ \whatq_t(s,a) }
+ \sum_{s \neq s_L} \sqrt{ \whatq_t(s) }\rbr{ \sum_{a\neq \pi(s)}\pi_t(a|s) } \nonumber\\
    &\leq 2\sum_{s \neq s_L}\sum_{a \neq \pi(s)} \sqrt{ \whatq_t(s,a) },    \label{eq3:stability-unknown}
\end{align}

Therefore, by Eq. \eqref{eq:tsallis_adaptive_regret_decomp}, Lemma \ref{lem:exp_high_prob_bound_cond}, Eq. \eqref{eq1:stability-unknown}, Eq. \eqref{eq2:stability-unknown}, Eq. \eqref{eq3:stability-unknown} and tower rule, we have for any deterministic policy $\pi:S\rightarrow A$ that, 
	\begin{align*}\allowdisplaybreaks
	\E\sbr{ \textsc{EstReg}_i(\pi) } &  \leq \order\left(\E\sbr{\sqrt{L|S|}  \sum_{t=t_i}^{t_{i+1}-1} \eta_t \cdot \sqrt{\sum_{s  \neq s_L}\sum_{a\neq \pi(s)} \whatq_t(s,a)  }}\right)\\
	&\quad +  \order\left(\E\sbr{\sum_{t=t_i}^{t_{i+1}-1}  \eta_t \cdot \sum_{s \neq s_L}\sum_{a\neq \pi(s)}\sqrt{\whatq_t(s,a)}}\right)\\
        & \quad + \order\left(\sum_{t=t_i}^{t_{i+1}-1} \eta_t \cdot \sum_{s \neq s_L} \sum_{a\in A} \whatq_t(s,a)^{\nicefrac{3}{2}} \cdot B_{i(t)}(s,a)^2\right)\\
	& \quad + \order\rbr{L|S||A|\log T+ \delta \cdot \E\sbr{ L|S|T\rbr{ t_{i+1} - t_i   } }  } \\
	\end{align*}

\end{proof}

\subsection{Adversarial regret guarantee}
Recall from Eq. \eqref{st:regret-decomp-49} that the regret decomposes as:
\begin{align*}
    \mathrm{Reg}(\pi) = \E\Biggsbr{\underbrace{\sum_{t=1}^{T} \bar{V}^{\pi_t}_t(s_0) - \widetilde{V}^{\pi_t}_t(s_0)}_{\textsc{Err}_1 }}  + \E\Biggsbr{\underbrace{\sum_{t=1}^{T} \widetilde{V}^{\pi_t}_t(s_0) - \widetilde{V}^{\pi}_t(s_0) }_{\textsc{EstReg} }}  + \E\Biggsbr{\underbrace{\sum_{t=1}^{T}\widetilde{V}^{\pi}_t(s_0) - \bar{V}^{\pi}_t(s_0) }_{\textsc{Err}_2}}.
\end{align*}
We now analyse each term separately. 

First, we have the following:
\begin{align*}
\textsc{Err}_1 & = \sum_{t=1}^{T} \inner{q_t, \bar\ell_t } - \inner{ \whatq_t,  \tilde \ell_t} \\
& = \sum_{t=1}^{T}\sum_{s \neq s_L} \sum_{a \in A}  \frac{(Q^{\pi_t}(s,a)-V^{\pi_t}(s)+1 - \pi_t(a|s))\whatq_t(s,a)}{u_t(s,a)} \cdot \rbr{u_t(s,a) - q_t(s,a)}\\
&\quad+ \sum_{t=1}^{T} \inner{q_t - \whatq_t, 
\bar\ell_t } +   \sum_{t=1}^{T} \inner{\whatq_t, B_{i(t)} }, 
\end{align*}
where the second equality follows from Eq. \eqref{eq:cond-exp-loss} and $\bar{\ell}_t(s,a)=Q^{\pi_t}(s,a)-V^{\pi_t}(s)$.

Due to the above equality, proposition \ref{prop:pseudo-loss}, Lemma \ref{lem:exp_high_prob_bound_cond}, the fact that $\whatq_t(s,a)\leq u_t(s,a)$ under event $\mathcal{A}$ and analysis of $\textsc{Err}_1$ from Appendix C.2 \cite{jin2021best}, we get the following:
\begin{equation*}
    \E[\textsc{Err}_1]=\tilde{\mathcal{O}}\rbr{ L|S|\sqrt{|A|T } +  L^2|S|^3|A|^2 }
\end{equation*}

Next, we have the following due to Lemma \ref{lem:tQ<=bQ-c.1.1}, Lemma \ref{lem:exp_high_prob_bound_cond}, Corollary \ref{cor:barV<=1}, and Corollary \ref{cor:til-Q-c.1.5}:
\begin{equation*}
    \E[\textsc{Err}_2]=\tilde{\mathcal{O}}\rbr{ L|S|T^2\cdot \delta}
\end{equation*}

According to Eq. \eqref{eq:bobw_bandit_tsallis_reg_worstcase} of Lemma \ref{lem:bobw_bandit_tsallis_reg}, we have  
\begin{align*}
& \textsc{EstReg}(\pi) =  \E\sbr{ \sum_{t=1}^{T} \inner{ \whatq_t - q^{\bar{P}_{i(t)}, \pi }, \hatl_t  } }  = \E\sbr{ \sum_{i=1}^{N} \textsc{EstReg}_i(\pi) } \\
&\leq \order\rbr{ \E\sbr{ \sqrt{L|S||A|} \cdot \sum_{i=1}^N\sum_{t=t_i}^{t_{i+1}-1} \eta_t } +  L|S|^2|A|^2\log^2 T + \delta \cdot L|S|T^2   }\tag{due to Eq. \eqref{eq:bobw_bandit_tsallis_reg_worstcase} and Lemma \ref{lem:bobw_lr_properties}} \\
& \leq \tilde{\mathcal{O}}\rbr{ |S||A| \sqrt{LT} + L|S|^2|A|^2 }. \tag{due to Eq. \eqref{eq:bobw_bandit_lr_2} }
\end{align*}

Finally, by combining the bounds for $\textsc{Err}_1$, $\textsc{Err}_2$, and $\textsc{EstReg}$, we obtain:
\begin{align*}
\text{Reg}(\pi) & = \tilde{\mathcal{O}}\rbr{  L|S|\sqrt{|A|T} + |S||A| \sqrt{LT} + L^2|S|^3|A|^2 }.
\end{align*}
\subsection{Stochastic regret guarantee}
\subsubsection{Self-bounding terms and related lemma}
In this section, we adopt the definition of self-bounding terms and the related lemmas from \cite{jin2021best}.

\begin{definition}[Self-bounding Terms] \label{def:self_bounding_terms}\allowdisplaybreaks
For some mapping $\pi^\star : S\rightarrow A$, define the following:
\begin{equation*}
\begin{aligned}
{\mathbb{Q}_1}(J) & = \sum_{t=1}^{T}\sum_{s\neq s_L} \sum_{a \neq \pi^{\star}(s)} q_t(s,a) \sqrt{\frac{J}{ \max\cbr{m_{i(t)}(s,a)} }}, \\
{\mathbb{Q}_2}(J) & =  \sum_{t=1}^{T}\sum_{s\neq s_L} \sum_{a = \pi^{\star}(s)} \rbr{ q_t(s,a) - q^\star_t(s,a) }  \sqrt{\frac{J}{ \max\cbr{m_{i(t)}(s,a) ,1 } }} ,\\
{\mathbb{Q}_3}(J) & =  \sum_{t=1}^{T}\sum_{s\neq s_L} \sum_{a \neq \pi^{\star}(s)} \sum_{k=0}^{k(s)-1} \sum_{ (u, v, w)\in T_k } q_t(u,v) \sqrt{\frac{P(w|u,v)\cdot J}{ \max\cbr{m_{i(t)}(u,v),1} } } q_t(s,a|w), \\
{\mathbb{Q}_4}(J) & = \sqrt{J\cdot \sum_{t=1}^{T}\sum_{s\neq s_L} \sum_{a \neq \pi^{\star}(s)} q_t(s,a)  } , \\ 
{\mathbb{Q}_5}(J) & = \sum_{s\neq s_L} \sum_{a \neq \pi^{\star}(s)} \sqrt{J \sum_{t=1}^{T} q_t(s,a)  } ,  \\ 
{\mathbb{Q}_6}(J) & = \sum_{t=1}^{T}\sum_{s\neq s_L} \sum_{a = \pi^{\star}(s)} \frac{q_t(s,a) - q^\star_t(s,a) }{ q_t(s,a)} \rbr{ \sum_{k=0}^{k(s)-1} \sum_{ (u, v, w)\in T_k } q_t(u,v) \sqrt{\frac{P(w|u,v)\cdot J}{ \max\cbr{m_{i(t)}(u,v),1} } } q_t(s,a|w)}. \\
\end{aligned}
\end{equation*}
\end{definition}

\begin{lemma} \label{lem:self_bounding_term_1} Suppose Condition~\eqref{eq:self-bounding-MDP} holds. Then we have for any $\alpha \in \mathbb{R}_{+}$,
	\begin{align*}
	\E\sbr{ \mathbb{Q}_1(J)} \leq \alpha \cdot (\Reg_T(\pi^\star)+C)+   \frac{1}{\alpha } \sum_{s\neq s_L}\sum_{a\neq \pi^{\star}(s)} \frac{8J }{\gap(s,a) }.
	\end{align*}
\end{lemma}

\begin{lemma} \label{lem:self_bounding_term_2} Suppose Condition~\eqref{eq:self-bounding-MDP} holds. Then we have for any $\beta \in \mathbb{R}_{+}$, 
	\begin{align*}
	\E\sbr{ \mathbb{Q}_2(J) } \leq \beta \cdot (\Reg_T(\pi^\star)+C)  +   \frac{1}{\beta } \cdot \frac{8|S|LJ }{\gapmin }.
	\end{align*}
\end{lemma}

\begin{lemma} \label{lem:self_bounding_term_3} Suppose Condition~\eqref{eq:self-bounding-MDP} holds. Then we have for any $\alpha,\beta \in \mathbb{R}_{+}$,
	\begin{align*}
	\E\sbr{ \mathbb{Q}_3(J)} \leq \rbr{\alpha + \beta } \cdot (\Reg_T(\pi^\star)+C)  +   \frac{1}{\alpha } \cdot\sum_{s\neq s_L}\sum_{ a\neq \pi^{\star}(s)} \frac{8L^2|S|J }{\gap(s,a) } + \frac{1}{\beta } \cdot \frac{8L^2|S|^2J }{\gapmin }  .
	\end{align*}
\end{lemma}

\begin{lemma} \label{lem:self_bounding_term_4} Suppose Condition~\eqref{eq:self-bounding-MDP} holds. Then we have for any $\beta \in \mathbb{R}_{+}$,
	\begin{align*}
	\E\sbr{ \mathbb{Q}_4(J)} \leq \beta \cdot (\Reg_T(\pi^\star)+C)+   \frac{1}{\beta } \cdot \frac{J }{4\gapmin }.
	\end{align*}
\end{lemma}

\begin{lemma} \label{lem:self_bounding_term_5} Suppose Condition~\eqref{eq:self-bounding-MDP} holds. Then we have for any $\alpha \in \mathbb{R}_{+}$,
	\begin{align*}
	\E\sbr{ \mathbb{Q}_5(J)} \leq \alpha \cdot (\Reg_T(\pi^\star)+C)+    \sum_{s\neq s_L}\sum_{a\neq \pi^{\star}(s)} \frac{J}{4\alpha \gap(s,a)}.
	\end{align*}
\end{lemma}

\begin{lemma} \label{lem:self_bounding_term_6} Suppose Condition~\eqref{eq:self-bounding-MDP} holds. Then we have for any $\beta \in \mathbb{R}_{+}$,
	\begin{align*}
	\E\sbr{ \mathbb{Q}_6(J)} \leq \beta \cdot (\Reg_T(\pi^\star)+C)+   \frac{1}{\beta } \cdot \frac{8L^3|S|^2|A|  \cdot J }{\gapmin }.
	\end{align*}
\end{lemma}

\subsubsection{Proof for the stochastic world}
Similarly to the proof in Appendix C.3 of \cite{jin2021best}, we decompose the sum of $\textsc{Err}_1$ and $\textsc{Err}_2$ into four terms $\textsc{ErrSub}$, $\textsc{ErrOpt}$, $\textsc{OccDiff}$ and $\textsc{Bias}$: 
\begin{align*}
\textsc{Err}_1 + \textsc{Err}_2 &  =  \sum_{t=1}^{T}\sum_{s \neq s_L} \sum_{a \neq \pi^{\star}(s)} q_t(s,a)\bar{E}^{\pi^{\star}}_t(s,a)  &&(\textsc{ErrSub})\\  
& \quad + \sum_{t=1}^{T}\sum_{s \neq s_L} \sum_{a = \pi^{\star}(s) } \rbr{ q_t(s,a) - q^{\star}_{t}(s,a) } \bar{E}^{\pi^{\star}}_t(s,a)  &&(\textsc{ErrOpt}) \\
& \quad + \sum_{t=1}^{T}\sum_{s \neq s_L} \sum_{a \in A} \rbr{ q_t(s,a) - \widehat{q}_t(s,a)} \rbr{\widetilde{Q}^{\pi^{\star}}_t(s,a) -\widetilde{V}^{\pi^{\star}}_t(s)} &&(\textsc{OccDiff})  \\
& \quad + \sum_{t=1}^{T}\sum_{s \neq s_L} \sum_{a\neq \pi^{\star}(s)}  q^{\star}_{t}(s,a) \rbr{\widetilde{V}^{\pi^{\star}}_t(s) - \bar V^{\pi^\star}_t(s) } && (\textsc{Bias}) 
\end{align*}
where $\bar{E}^{\pi}_t$ is defined as 
\begin{align*}
\bar{E}^{\pi}_t(s,a) =  \bar\ell_t(s,a) + \sum_{s' \in S_{k(s)+1}} P(s'|s,a)\widetilde{V}^{\pi}_t(s') -  \widetilde{Q}^{\pi}_t(s,a).
\end{align*}


We now begin by upper bounding $\E\sbr{\textsc{OccDiff}}$. First observe that we have the following:
\begin{align*}
\textsc{OccDiff} & = \sum_{t=1}^{T}\sum_{s \neq s_L} \sum_{a \in A} \rbr{ q_t(s,a) - \whatq_t(s,a)} \rbr{\tilde Q^{\pi^{\star}}_t(s,a) -\tilde V^{\pi^{\star}}_t(s)} \\
& = \sum_{t=1}^{T}\sum_{s \neq s_L} \sum_{a \neq \pi^\star(s) } \rbr{ q_t(s,a) - \whatq_t(s,a)} \rbr{\tilde Q^{\pi^{\star}}_t(s,a) -\tilde V^{\pi^{\star}}_t(s)}  \\
& \leq \sum_{t=1}^{T}\sum_{s \neq s_L} \sum_{a \neq \pi^{\star}(s)} \abr{ q_t(s,a) - \whatq_t(s,a)}\cdot\abr{\tilde Q^{\pi^{\star}}_t(s,a) -\tilde V^{\pi^{\star}}_t(s)},
\end{align*} 

Under event $\mathcal{A}$, we further have $\sum_{t=1}^{T}\sum_{s \neq s_L} \sum_{a \neq \pi^{\star}(s)} \abr{ q_t(s,a) - \whatq_t(s,a)}\cdot\abr{\tilde Q^{\pi^{\star}}_t(s,a) -\tilde V^{\pi^{\star}}_t(s)}\leq 5L \sum_{t=1}^{T}\sum_{s \neq s_L} \sum_{a \neq \pi^{\star}(s)} \abr{ q_t(s,a) - \whatq_t(s,a)}$ as $-3\leq \tilde \ell_t(s,a)\leq 1$ under event $\mathcal{A}$. If event $\mathcal{A}$ doesn't hold, we have $\sum_{t=1}^{T}\sum_{s \neq s_L} \sum_{a \neq \pi^{\star}(s)} \abr{ q_t(s,a) - \whatq_t(s,a)}\cdot\abr{\tilde Q^{\pi^{\star}}_t(s,a) -\tilde V^{\pi^{\star}}_t(s)}\leq 12 L |S|^2|A|T^2$ due to Corollary \ref{cor:til-Q-c.1.5}. Hence due to Lemma \ref{lem:exp_high_prob_bound_cond}, we have the following
\begin{align*}
    \E[\textsc{OccDiff}]&\leq \order\left(L\cdot\E[\sum_{t=1}^{T}\sum_{s \neq s_L} \sum_{a \neq \pi^{\star}(s)} \abr{ q_t(s,a) - \whatq_t(s,a)}]+\delta\cdot L |S|^2|A|T^2\right)\\
    &\leq \order\left(L^3|S|^3|A|^2\ln^2\iota+\delta\cdot L |S|^2|A|T^2+\mathbb{Q}_3(L^2\ln \iota)\right)
\end{align*}
where the last line follows from the definition of $\mathbb{Q}_3$ and Lemma D.3.10 of \cite{jin2021best}.

Next, due to Lemma \ref{lem:exp_high_prob_bound_cond}, Lemma \ref{lem:tQ<=bQ-c.1.1}, and Corollary \ref{cor:til-Q-c.1.5}, we have $\E\sbr{\textsc{Bias}} \leq \delta\cdot\order(L|S|^2AT^2)$. 
The first two terms $\textsc{ErrSub}$ and $\textsc{ErrOpt}$ are bounded differently. First, under event $\calA$, we have
\begin{align*}
\bar{E}^{\pi^{\star}}_t(s,a) & = \bar\ell_t(s,a) - \tilde{\ell}_t(s,a) + \sum_{s' \in S_{k(s)+1}} \rbr{ P(s'|s,a) - \bar{P}_{i(t)}(s'|s,a) } \widetilde{V}^{\pi^{\star}}_t(s') \\
& = (Q^{\pi_t}(s,a)-V^{\pi_t}(s)+1 - \pi_t(a|s)) \rbr{  1 - \frac{q_t(s,a) }{u_t(s,a) }}\\
&\quad + B_{i(t)}(s,a) + \sum_{s' \in S_{k(s)+1}} \rbr{ P(s'|s,a) - \bar{P}_{i(t)}(s'|s,a) } \widetilde{V}^{\pi^{\star}}_t(s')  \\
& \leq \frac{2(u_t(s,a) - q_t(s,a)) }{ q_t(s,a)  } + 4 L \cdot B_{i(t)}(s,a)
\end{align*}
where the last line uses the definition of the event $\calA$ along with the fact that $q_t(s,a) \leq u_t(s,a)$ and $-3\leq \tilde\ell_t(s,a)\leq 1$ under this event.
Next, observe that the range of $\bar{E}^{\pi}_t$ is $\order(L|S|t)$, as established by Proposition \ref{prop:pseudo-loss} and Corollary \ref{cor:til-Q-c.1.5}, which implies that the range of both $\textsc{ErrSub}$ and $\textsc{ErrOpt}$ is $\order(L^2|S|T^2)$.
Thus, it suffices to add a term of order $\order(\delta \cdot L^2|S|T^2)$ to account for the event $\calA^c$.

By the exact same analysis as in Appendix C.3 and Appendix B.2 of \cite{jin2021best} and the fact that $|\tilde V(s)|\leq \order(L)$ under event $\mathcal{A}$, we have the following:
\begin{align*}
\E\sbr{\textsc{ErrSub}} &= \order\rbr{ \E\left[\mathbb{Q}_3(\ln \iota) + \mathbb{Q}_1(L^2|S|\ln \iota)\right] +  L^2|S|^3|A|^2 \ln^2 \iota }\\
\E\sbr{\textsc{ErrOpt}} &= \order\rbr{ \E\left[\mathbb{Q}_6(\ln \iota) + \mathbb{Q}_2(L^2|S|\ln \iota)\right] +  L^2|S|^3|A|^2 \ln^2 \iota}.
\end{align*}

We are now left with bounding $\textsc{EstReg}$ using self-bounding terms.
\paragraph{Term $\textsc{EstReg}$} Based on Eq. \eqref{eq:bobw_bandit_tsallis_reg_adaptive} from Lemma \ref{lem:bobw_bandit_tsallis_reg}, summing over all epochs gives the following upper bound for $\E\sbr{\textsc{EstReg}}$:
\begin{align*}
& \order\rbr{ \E\sbr{ \sqrt{|S|L}  \sum_{i=1}^{N} \sum_{t=t_i}^{t_{i+1}-1} \eta_t \cdot \sqrt{\sum_{s  \neq s_L}\sum_{a\neq \pi^\star(s)} \whatq_t(s,a)  } }  +   \E\sbr{  \sum_{i=1}^{N} \sum_{t=t_i}^{t_{i+1}-1} \eta_t \cdot \sum_{s \neq s_L} \sum_{a\neq \pi^\star(s)} \sqrt{ \whatq_t(s,a) } } } \\
& \quad + \order\left(\sum_{i=1}^N\sum_{t=t_i}^{t_{i+1}-1} \eta_t \cdot \sum_{s \neq s_L} \sum_{a\in A} \whatq_t(s,a)^{\nicefrac{3}{2}} \cdot B_{i(t)}(s,a)^2\right)+ \order\rbr{L|S|^2|A|^2\log^2 T+ \delta \cdot L|S|T^2} \tag{due to Lemma \ref{lem:bobw_lr_properties}}\\
& = \order\rbr{ \E\sbr{ \sqrt{|S|L}  \sum_{t=1}^{T} \eta_t \cdot \sqrt{\sum_{s  \neq s_L}\sum_{a\neq \pi^\star (s)} \whatq_t(s,a)  } } } + \order\rbr{  \E\sbr{  \sum_{t=1}^{T} \eta_t \cdot \sum_{s \neq s_L} \sum_{a\neq \pi^\star(s)} \sqrt{ \whatq_t(s,a) } } } \\
& \quad  + \order\rbr{ L^2|S|^3|A|^2 \ln^2 \iota } \tag{due to Lemma \ref{lem:bobw_tasallis_bound_extra}}.
\end{align*}

By the exact same analysis as in Appendix C.3 of \cite{jin2021best}, we bound the first term as follows:
\begin{align*}
& \E\sbr{ \sqrt{|S|L}  \sum_{t=1}^{T} \eta_t \cdot \sqrt{\sum_{s  \neq s_L}\sum_{a\neq \pi^\star(s)} \whatq_t(s,a)  } }   \\
& = \order\Bigrbr{ \E\sbr{ \mathbb{Q}_4(L|S|^2|A|\log^2 T) + \mathbb{Q}_3\rbr{ \ln \iota  } } +   L^2|S|^3|A|^2 \ln^2 \iota  }.
\end{align*} 

Again by using the exact same analysis as in Appendix C.3 of \cite{jin2021best}, we bound the second term as follows:
\begin{align*}
&  \E\sbr{  \sum_{t=1}^{T} \eta_t \cdot \sum_{s \neq s_L} \sum_{a\neq \pi^\star (s)} \sqrt{ \whatq_t(s,a) } } \\
& = \order\Bigrbr{ \E\sbr{ \mathbb{Q}_5(|S||A|\log^2 T) + \mathbb{Q}_3\rbr{ \ln \iota  } } +   L^2|S|^3|A|^2 \ln^2 \iota}.
\end{align*}

Thus, we obtain the final bound on $\E\sbr{ \textsc{EstReg}}$:
\begin{align*}
\E\sbr{ \textsc{EstReg}} & = \order\Bigrbr{  \E\sbr{ \mathbb{Q}_4\rbr{ L|S|^2|A|\log^2 T} + \mathbb{Q}_5\rbr{|S||A|\log^2 T} + \mathbb{Q}_3\rbr{ \ln \iota  }  }  + L^2|S|^3|A|^2 \ln^2 \iota } 
\end{align*} 

Recall that $\delta=1/T^3$ and $\iota=\frac{|S||A|T}{\delta}$. Finally, by combining the bounds of each term, we finally have 
\begin{align*}
\Reg_T(\pi^\star) & \leq \order \Big( \E\sbr{ \mathbb{Q}_1\rbr{ L^2|S| \ln \iota  } + \mathbb{Q}_3 \rbr{ \ln \iota } } && \rbr{ \text{from } \textsc{ErrSub} } \\
& \quad +\E\sbr{  \mathbb{Q}_2\rbr{ L^2|S| \ln \iota  } +  \mathbb{Q}_6 \rbr{ \ln \iota  } } && \rbr{ \text{from } \textsc{ErrOpt} }  \\
& \quad + \E\sbr{  \mathbb{Q}_3\rbr{ L^2 \ln \iota } }  && \rbr{ \text{from } \textsc{OccDiff} }  \\
& \quad + \E\sbr{ \mathbb{Q}_4\rbr{  L|S|^2|A| \ln^2 \iota  } + \mathbb{Q}_5\rbr{|S||A|\ln^2 \iota} +  \mathbb{Q}_3\rbr{ \ln \iota } } && \rbr{ \text{from } \textsc{EstReg} } \\
& \quad +  L^2|S|^3|A|^2  \ln^2 \iota \Big). 
\end{align*}

Using the self-bounding lemmas (\ref{lem:self_bounding_term_1}-\ref{lem:self_bounding_term_6}) and the exact same analysis as in Appendix C.3 of \cite{jin2021best}, we get $\Reg_T(\pi^\star)$ is bounded by $\order\rbr{U +\sqrt{UC}+ V }$   
	where $V =  L^2|S|^3|A|^2 \ln^2 \iota$ and $U$ is defined as
	\[
	U =  \sum_{s \neq s_L} \sum_{ a\neq \pi^\star(s)} \sbr{ \frac{L^4|S|\ln \iota +  |S||A|\ln^2 \iota}{\gap(s,a)}} + \sbr{\frac{(L^4|S|^2+L^3|S|^2|A|) \ln \iota +  L|S|^2|A| \ln^2 \iota}{\gapmin}}.
	\]
This completes the entire proof.

\end{document}